%% file: _main.tex
\documentclass{article}
\PassOptionsToPackage{numbers}{natbib}
\usepackage{natbib}
\usepackage[final]{neurips_2022}
\usepackage[utf8]{inputenc} % allow utf-8 input
\usepackage[T1]{fontenc}    % use 8-bit T1 fonts
\usepackage{hyperref}       % hyperlinks
\usepackage{url}            % simple URL typesetting
\usepackage{booktabs}       % professional-quality tables
\usepackage{amsfonts}       % blackboard math symbols
\usepackage{nicefrac}       % compact symbols for 1/2, etc.
\usepackage{microtype}      % microtypography
\usepackage{xcolor}         % colors
\usepackage{float}
\usepackage{wrapfig}
\usepackage{algorithm}
\usepackage{algorithmic}
\hypersetup{
    colorlinks=true,
    linkcolor=red,
    citecolor=blue,
    filecolor=magenta,      
    urlcolor=blue,
linktocpage}
\usepackage{comment}
\usepackage{microtype}
\usepackage{graphicx}
\usepackage{subfigure}
\usepackage{booktabs} % for professional tables
\usepackage{threeparttable}

\usepackage{hyperref}

\usepackage{amsmath}
\usepackage{amssymb}
\usepackage{mathtools}
\usepackage{amsthm}
\usepackage{enumitem}
\usepackage[capitalize,noabbrev]{cleveref}
\usepackage{thm-restate}
\theoremstyle{plain}
\newtheorem{theorem}{Theorem}[section]

\newtheorem{lemma}[theorem]{Lemma}

\theoremstyle{definition}
\newtheorem{definition}[theorem]{Definition}
\newtheorem{assumption}[theorem]{Assumption}
\theoremstyle{remark}

\usepackage[textsize=tiny]{todonotes}
\usepackage[noamssymbols,upint]{newtxmath} % various optional arguments
\author{
  Xidong Feng\footnotemark[1]  \\
  University College London\\
  \texttt{xidong.feng.20@ucl.ac.uk} \\
  \And
  Bo Liu\footnotemark[1] \\
  Institute of Automation,\\
  Chinese Academy of Sciences\\
  \texttt{benjaminliu.eecs@gmail.com} \\
  \And
  Jie Ren\\
  University of Edinburgh \\
  \texttt{jieren9806@gmail.com}\\
  \And
  Luo Mai\\
  University of Edinburgh \\
  \texttt{luo.mai@ed.ac.uk}\\
  \And
  Rui Zhu\\
  DeepMind \\
  \texttt{ruizhu@google.com}\\
  \And
  Haifeng Zhang\\
  Institute of Automation, CAS\\
  Nanjing Artificial Intelligence Research of IA\\
  \texttt{haifeng.zhang@ia.ac.cn}\\
  \And
  Jun Wang\\
  University College London \\
  \texttt{jun.wang@cs.ucl.ac.uk}\\
  \And
  Yaodong Yang\footnotemark[2]\\
  Institute for AI, Peking University\\
  Beijing Institute for General AI \\ 
  \texttt{yaodong.yang@pku.edu.cn}\\
}

\title{A Theoretical Understanding of Gradient Bias in Meta-Reinforcement Learning}

\begin{document}

\maketitle

\input{core/0_abstract.tex}
\renewcommand{\thefootnote}{\fnsymbol{footnote}}
\footnotetext[1]{Equal contribution, the order is determined by flipping a coin. See Appendix \ref{apx: author_contrib} for more details.}
\footnotetext[2]{Corresponding author.}
\renewcommand*{\thefootnote}{\arabic{footnote}}
\renewcommand{\thefootnote}{\fnsymbol{footnote}}
% \footnotetext[1]{Equal contribution, \href{mailto:benjaminliu.eecs@gmail.com}{benjaminliu.eecs@gmail.com}, \href{mailto:xidong.feng.20@ucl.ac.uk}{xidong.feng.20@ucl.ac.uk}}
% \footnotetext[2]{Corresponding author, \href{mailto:yaodong.yang@pku.edu.cn}{yaodong.yang@pku.edu.cn}}
% \renewcommand*{\thefootnote}{\arabic{footnote}}

%\input{__table.tex}
\input{core/1_introduction.tex}
\input{core/2_related_work.tex}
\input{core/3_framework.tex}
\input{core/4_analysis.tex}
\input{core/5_fix.tex}

\input{core/6_experiment.tex}

\input{core/7_conclusion.tex}
\input{core/acknowledgement}
\bibliography{_reference}

%%%%%%%%%%%%%%%%%%%%%%%%%%%%%%%%%%%%%%%%%%%%%%%%%%%%%%%%%%%%%%%%%%%%%%%%%%%%%%%
%%%%%%%%%%%%%%%%%%%%%%%%%%%%%%%%%%%%%%%%%%%%%%%%%%%%%%%%%%%%%%%%%%%%%%%%%%%%%%%
% CHECKLIST
%%%%%%%%%%%%%%%%%%%%%%%%%%%%%%%%%%%%%%%%%%%%%%%%%%%%%%%%%%%%%%%%%%%%%%%%%%%%%%%
%%%%%%%%%%%%%%%%%%%%%%%%%%%%%%%%%%%%%%%%%%%%%%%%%%%%%%%%%%%%

\section*{Checklist}

\begin{enumerate}

\item For all authors...
\begin{enumerate}
  \item Do the main claims made in the abstract and introduction accurately reflect the paper's contributions and scope?
    \answerYes{}
  \item Did you describe the limitations of your work?
    \answerYes{See Appendix \ref{dis_epg}}.
  \item Did you discuss any potential negative societal impacts of your work?
    \answerNA{}
  \item Have you read the ethics review guidelines and ensured that your paper conforms to them?
    \answerYes{}
\end{enumerate}

\item If you are including theoretical results...
\begin{enumerate}
  \item Did you state the full set of assumptions of all theoretical results?
    \answerYes{See Sec. \ref{analysis}}
        \item Did you include complete proofs of all theoretical results?
    \answerYes{See Appendix~\ref{lemmas} and Appendix~\ref{theorems}}
\end{enumerate}

\item If you ran experiments...
\begin{enumerate}
  \item Did you include the code, data, and instructions needed to reproduce the main experimental results (either in the supplemental material or as a URL)?
    \answerYes{}
  \item Did you specify all the training details (e.g., data splits, hyperparameters, how they were chosen)?
    \answerYes{}, See Appendix \ref{apx: experiment}
        \item Did you report error bars (e.g., with respect to the random seed after running experiments multiple times)?
    \answerYes{}, See Fig. \ref{fig:sample}, \ref{fig:lola}, \ref{fig:mgrl}.
        \item Did you include the total amount of compute and the type of resources used (e.g., type of GPUs, internal cluster, or cloud provider)?
    \answerYes{See Appendix \ref{apx: experiment}}
\end{enumerate}

\item If you are using existing assets (e.g., code, data, models) or curating/releasing new assets...
\begin{enumerate}
  \item If your work uses existing assets, did you cite the creators?
    \answerYes{}
  \item Did you mention the license of the assets?
    \answerYes{}
  \item Did you include any new assets either in the supplemental material or as a URL?
    \answerNo{}
  \item Did you discuss whether and how consent was obtained from people whose data you're using/curating?
    \answerNA{}
  \item Did you discuss whether the data you are using/curating contains personally identifiable information or offensive content?
    \answerYes{}
\end{enumerate}

\item If you used crowdsourcing or conducted research with human subjects...
\begin{enumerate}
  \item Did you include the full text of instructions given to participants and screenshots, if applicable?
    \answerNA{}
  \item Did you describe any potential participant risks, with links to Institutional Review Board (IRB) approvals, if applicable?
    \answerNA{}
  \item Did you include the estimated hourly wage paid to participants and the total amount spent on participant compensation?
    \answerNA{}
\end{enumerate}

\end{enumerate}

%%%%%%%%%%%%%%%%%%%%%%%%%%%%%%%%%%%%%%%%%%%%%%%%%%%%%%%%%%%%%%%%%%%%%%%%%%%%%%%
%%%%%%%%%%%%%%%%%%%%%%%%%%%%%%%%%%%%%%%%%%%%%%%%%%%%%%%%%%%%%%%%%%%%%%%%%%%%%%%
% APPENDIX
%%%%%%%%%%%%%%%%%%%%%%%%%%%%%%%%%%%%%%%%%%%%%%%%%%%%%%%%%%%%%%%%%%%%%%%%%%%%%%%
%%%%%%%%%%%%%%%%%%%%%%%%%%%%%%%%%%%%%%%%%%%%%%%%%%%%%%%%%%%%%%%%%%%%%%%%%%%%%%%
\newpage
\appendix
\onecolumn

\input{core/8_appendix.tex}

\input{core/9_aux_lemma.tex}

\end{document}

%% file: core/0_abstract.tex
\begin{abstract}

Gradient-based Meta-RL (GMRL) refers to methods that maintain two-level optimisation procedures wherein the outer-loop meta-learner guides the inner-loop gradient-based reinforcement learner to achieve fast adaptations. In this paper, we develop a unified framework that describes variations of GMRL algorithms and points out that existing stochastic meta-gradient estimators adopted by GMRL are actually \textbf{biased}. 
Such meta-gradient bias comes from two sources: 1) the compositional bias incurred by the two-level problem structure, which has an upper bound of $\mathcal{O}\big(K\alpha^{K}\hat{\sigma}_{\text{In}}|\tau|^{-0.5}\big)$ \emph{w.r.t.} inner-loop update step $K$, learning rate $\alpha$, estimate variance $\hat{\sigma}^{2}_{\text{In}}$ and sample size $|\tau|$, and 2) the multi-step Hessian estimation bias $\hat{\Delta}_{H}$ due to the use of autodiff, which has a polynomial impact $\mathcal{O}\big((K-1)(\hat{\Delta}_{H})^{K-1}\big)$ on the meta-gradient bias. We study tabular MDPs empirically and offer quantitative evidence that testifies our theoretical findings on existing stochastic meta-gradient estimators. Furthermore, we conduct experiments on Iterated Prisoner's Dilemma and Atari games to show how other methods such as off-policy learning and low-bias estimator can help fix the gradient bias for GMRL algorithms in general.
\end{abstract}

%% file: core/1_introduction.tex
% !TEX root = OPPO_arxiv.tex
%\begin{flushleft}
\section{Introduction}

Meta Learning, also known as learning to learn, is proposed to equip intelligent agents with meta knowledge for fast adaptations \cite{schmidhuber1987evolutionary}. Meta-Reinforcement Learning (RL) algorithms aim to train RL agents that can adapt to new tasks  with only  few examples \cite{duan2016rl,finn2017model,gupta2018meta}. For example, MAML-RL  \cite{finn2017model}, a typical Meta-RL algorithm, learns the initial parameters of an agent's policy so that the agent can rapidly adapt to new environments with a limited number of policy-gradient updates. Recently, Meta-RL algorithms have been further developed beyond the scope of learning fast adaptations. An important direction is to conduct online meta-gradient learning for adaptively tuning algorithmic   hyper-parameters \cite{xu2018meta,zahavy2020self} or designing intrinsic reward   \cite{zheng2018learning} during training in one single task. Besides,  there are Meta-RL developments that manage to discover new RL algorithms by learning  algorithmic components or other   fundamental concepts in RL, such as the policy gradient objective \cite{oh2020discovering} or  the TD-target \cite{xu2020meta}, which can generalise to solve a distribution of different tasks. 

In general, gradient-based Meta-RL (GMRL) tasks can be formulated by a two-level optimisation procedure. This procedure optimises the parameters of an outer-loop meta-learner, whose objective is dependent on  a $K$-step policy update process (e.g., stochastic gradient descent) conducted by an inner-loop learner. Formally, this procedure can be written as:
\begin{equation}\label{meta_objective}
\begin{gathered}
    \max_{\boldsymbol{\phi}} J^{K}(\boldsymbol{\phi}):=
    J^{\text{Out}}
    (\boldsymbol{\phi}, \boldsymbol{\theta}^{K}),\\
    \; \ \ \ \ \mbox{s.t.}\; \boldsymbol{\theta}^{i+1}=
    \boldsymbol{\theta}^{i} + 
    \alpha
    \nabla_{\boldsymbol{\theta}^{i}} J^{\text{In}}(\boldsymbol{\phi},\boldsymbol{\theta}^{i}),i\in\{0,1\ldots K-1\},
\end{gathered}
\end{equation}
where $\boldsymbol{\theta}$ are inner-loop policy parameters, $\boldsymbol{\phi}$ are meta parameters, $\alpha$ is the learning rate, $J^{\text{In}}$ and $J^{\text{Out}}$ are value functions for the inner and the outer-loop learner. % of a RL policy.
In solving Eq. (\ref{meta_objective}), 
estimating the meta-gradient of $\nabla_{\boldsymbol{\phi}}J^{K}(\boldsymbol{\phi})$ from a two-level optimisation process is non-trivial; how to conduct \emph{accurate} estimation on  $\nabla_{\boldsymbol{\phi}}J^{K}(\boldsymbol{\phi})$ has been a critical yet challenging research problem \cite{liu2019taming,rothfuss2018promp, tang2021unifying}. 

%Though several forms of Meta-RL algorithms exist,
%In this paper, we argue that there are two main issues in a broad range of GMRL methods: compositional bias and multi-step Hessian estimation bias. Specifically, the first compositional bias comes from the bi-level optimisation structure of gradient-based Meta-RL problem --- the outer meta-gradient estimation needs to differentiate through inner-loop stochastic optimisation subroutine. 

In this paper, we point out that the  meta-gradient estimators adopted by many recent GMRL methods \cite{oh2020discovering, xu2020meta} are in fact biased. We conclude that such bias comes from two sources: (1) the \textbf{ compositional bias}  and (2) the \textbf{multi-step Hessian  bias}.  The compositional bias origins from the discrepancy between the sampled policy gradient $\nabla_{\boldsymbol{\theta}} \hat{J}^{\text{In}}$ and expected policy gradient $\nabla_{\boldsymbol{\theta}} J^{\text{In}}$, and  the multi-step Hessian estimation bias occurs due  to the biased Hessian estimation $\nabla_{\boldsymbol{\theta}}^{2} \hat{J}^{\text{In}}$ resulting from the employment of automatic differentiation in modern GMRL implementations.  % of in the multi-step inner-loop setting for general GMRL algorithms.

%Even though there are a few prior studies that investigate the bias in Meta-RL, these studies are limited to a specific case: MAML-RL~\cite{rothfuss2018promp, tang2021unifying}, and they cannot be applied to other Meta-RL algorithms, such as [xxxx]. 
%More importantly, they do not provide methods to mitigate bias. [xxxx] Without effective ways of handling these bias, the bias can become the Achilles' heel of Meta-RL algorithms, especially in large-scale RL problems where meta-learners have to deal with complex optimisation problems, and many-step inner loops are adopted.

There are very few prior work that investigates the bias in GMRL, most of them limited to the Hessian estimation bias in the MAML-RL setting~\cite{rothfuss2018promp, tang2021unifying}. Our paper investigates the above two biases in a broader setting and applies on generic meta-gradient estimators.  For the compositional bias term,  %Several paper discussed this problem in the scope of optimisation \cite{wang2017stochastic, chen2021solving}. 
we offer the first theoretical analysis on its quantity, based on which we then investigate current mitigation solutions.
% first work in Meta-RL that offers both theoretical analysis and remedies to compositional bias.
For the multi-step Hessian bias, we provide rigorous analysis in GMRL settings particularly for those GMRL tasks where  complex inner-loop optimisations are involved.

For the rest of the paper,  
we first introduce a \emph{unified} Meta-RL framework that can describe  variations of existing GMRL algorithms. Building on this framework, % to compute the worst-case impact of bias in estimating meta-gradients. 
we offer two theoretical results that \textbf{1)} the compositional bias has an upper bound of $\mathcal{O}\big(K\alpha^{K}\hat{\sigma}_{\text{In}}|\tau|^{-0.5}\big)$ with respect to the inner-loop update step $K$, the learning rate $\alpha$, the estimate variance $\hat{\sigma}^{2}_{\text{In}}$ and the sample size $|\tau|$, and \textbf{2)} the multi-step Hessian  bias  $\hat{\Delta}_{H}$ has a polynomial impact of  $\mathcal{O}\big((K-1)(\hat{\Delta}_{H})^{K-1}\big)$. 
 To validate our theoretical insights on these two biases, we conduct a comprehensive list ablation studies.  Experiments over tabular MDP with MAML-RL \cite{finn2017model} and LIRPG \cite{zheng2018learning} demonstrate how quantitatively these two biases influence the estimation accuracy, which consolidates our theories.  

Furthermore, our theoretical results help understand to what extent existing methods can mitigate theses two bias empirically. For the compositional bias, we show that off-policy learning methods can reduce the inner-loop policy gradient variance and the resulting compositional bias. For the multi-step Hessian bias, we study how the  low-variance curvature technique based on \citet{rothfuss2018promp} can  help correct the Hessian bias for general GMRL problems. We  test these solutions on environments including   iterated Prisoner's Dilemma with off-policy corrected LOLA-DiCE \cite{foerster2017learning}, and eight Atari games based on MGRL \cite{xu2018meta} with the multi-step Hessian correction technique. Experimental results confirm that those bias-correction methods can substantially decrease the meta-gradients bias and improve the overall performance on rewards. 

%% file: core/2_related_work.tex
\section{Related Work}
\label{related}

A pioneering work that studies meta-gradient estimation is \citet{al2017continuous} who discussed the biased estimation problem of MAML and proposed the E-MAML sample based formulation to fix the meta-gradient bias. Following work includes \citet{rothfuss2018promp,liu2019taming,tang2021unifying} that tried to fix the meta-gradient estimation error by reducing the estimation variance so as to improve the performance. Moreover, \citet{foerster2018dice,farquhar2019loaded,mao2019baseline} discussed the higher-order gradient estimation in RL. Recently,  \citet{bonnet2021one,vuorio2021no} proposed algorithms to balance the  bias-variance trade-off for meta-gradient estimates. Within theoretical context of GMRL, most theoretical analysis focuses on convergence to stationary points. \citet{fallah2020convergence} established convergence guarantees of gradient-based meta-learning algorithms for supervised learning with one-step inner-loop update. \citet{ji2020multi} extended the analysis to multi-step inner loop updates. \citet{fallah2021convergence} established convergence for the E-MAML formulation. Our work is different from all the above prior work from three folds: (1) we study the additional bias term (the compositional bias) (2) we consider different formulation (expected update while sample-based update in \cite{vuorio2021no}, refer to Appendix \ref{dis_epg} for more discussions) (3) we study a broader scope of applications that include different GMRL algorithm instantiations and settings.

In the following part, we review existing GMRL algorithms and categorise them into four research topics. We offer one typical example for each topic in Table \ref{tab:unifying framework} and further discussed their relationship and how they can be unified into one framework in Sec.~\ref{framework}. We also provide a more self-contained explanation for each topic in Appendix \ref{more_topics} for readers that are not familiar with GMRL.
 %Finally, we present  prior works concerning meta-gradient estimation in GMRL. % in the rest of the section. \\ 

\textbf{Few-shot RL.} The idea of few-shot RL is to enable RL agent with fast learning ability. Specifically, the RL agent is only allowed to interact with the environment for a few trajectory to conduct task-specific adaptation. Conducting few-shot RL has two approaches: gradient based and context based. Context based Meta-RL involves works \cite{duan2016rl,fu2020towards, rakelly2019efficient, wang2016learning}, which uses neural network to embed the information from few-shot interactions so as to obtain task-relevant context. Gradient based few-shot RL \cite{al2017continuous,liu2019taming, rothfuss2018promp}  focus on meta-learning the model's initial parameters through meta-gradient descent.

\textbf{Opponent Shaping.} 
\citet{foerster2017learning,letcher2018stable,kim2020policy} explicitly models the learning process of opponents in multi-agent learning problems, which can be thought of as GMRL problems since  meta-gradient estimation in these works involve differentiation over opponent's policy updates.    % and incorporate it into self-learning process. 
By modelling opponents' learning process, the multi-agent learning process can reach better social welfare \cite{foerster2017learning, letcher2018stable} or the ego agent can adapt to a new  peer agent \cite{kim2020policy}.  %This research topic extends the setting of meta-gradient single agent RL to multi-agent scenarios. \\

\textbf{Single-lifetime Meta-gradient RL.} This line of research focuses on learning online adaptation over algorithmic hyper-parameters to enhance the performance of an RL agent in one single task, such as  discount factor   \cite{xu2018meta,zahavy2020self}, intrinsic reward generator  \cite{zheng2018learning}, auxiliary loss  \cite{veeriah2019discovery}, reward shaping mechanism  \cite{hu2020learning}, and value correction \cite{zhou2020online}. The main feature of online meta-gradient RL is that they are under single-lifetime setting \cite{xu2020meta}, meaning that the algorithm  only iterates through the whole RL learning procedure in one task rather a distribution of tasks. %This is consistent with traditional RL algorithms.

\textbf{Multi-lifetime  Meta-gradient RL.}
%There are lots of previous attempts to meta-learn some fundamental concepts in RL.
Compared to the previous single-lifetime settings, 
%In this research field, most works are under multi-lifetime setting and the final meta model can generalize to different environments.
``multi-lifetime" refers to  settings where agents learn to adapt on  a distribution of tasks or environments. The meta-learning target includes  policy gradient or TD learning objectives  \cite{bechtle2021meta,kirsch2019improving,oh2020discovering}, intrinsic reward  \cite{zheng2020can}, target value function 
\cite{xu2020meta}, options in hierarchical RL \cite{veeriah2021discovery}, and recently the design of curriculum in multi-agent learning \cite{feng2021neural}.

%  in contrast to ``single-lifetime" in online meta-gradient RL.
%Difference
%Broader meta-gradient RL algorithm One-step to mutli-step inner-loop
% Analysis on more estimation error term
% Different Setting

%In this paper, we focus on a broader range of settings: (1) One-step to mutli-step inner-loop (2) Broader meta-gradient RL algorithm (3) Analysis on more estimation error term.
%Recently we find a concurrent work \cite{tang2021biased}, which analyzed the discrepancy between estimation in practical algorithm and the real unbiased meta-gradient in the MAML-RL setting involves similar discussion as our Compositional bias in proof of proposition E.3. However, we focus on a broader range of settings: (1) One-step to mutli-step inner-loop (2) Broader meta-gradient RL algorithm (3) Analysis on more estimation error term, while \citet{tang2021biased} mainly focus on the problem in MAML-RL.
%while \citet{tang2021biased} mainly focus on the problem in MAML-RL

%% file: core/3_framework.tex
% !TEX root = arxiv.tex

%\begin{flushleft}

\section{A Unified Framework for Meta-gradient Estimation}
\label{framework}
In this section, we derive a general formulation for meta-gradient estimation in GMRL. This formulation enables us to conduct general analysis about meta-gradient estimation and we will show how algorithms in the four research topics mentioned in Sec.~\ref{related} can be described through it.

We propose that a general GMRL objective with $K$-step inner-loop policy gradient update can be written as the following objective 
\begin{equation}\label{GMRL_objective}
\vspace{-0.1pt}
\max_{\boldsymbol{\phi}} 
J^{K}(\boldsymbol{\phi}):=
J^{\text{Out}}(\boldsymbol{\phi}, \boldsymbol{\theta}^{K}), 
\boldsymbol{\theta}^{K}=
\boldsymbol{\theta}^{0} 
+ \alpha\sum_{i=0}^{K-1}
\nabla_{\boldsymbol{\theta}^{i}} J^{\text{In}}(\boldsymbol{\phi},\boldsymbol{\theta}^{i}).
\vspace{-0.1pt}
\end{equation}
\noindent{We} denote the meta-parameters as $\boldsymbol{\phi}$, and the pre- and post-adapt inner parameters as $\boldsymbol{\theta}$ and $\boldsymbol{\theta}^{K}$, respectively. The meta objective as $J^{\text{Out}}$, the inner loop objective as $J^{\text{In}}$. The general objective for GMRL is to maximise the meta objective $J^{\text{Out}}(\boldsymbol{\phi}, \boldsymbol{\theta}^{K})$, where the inner-loop post-adapt parameters are obtained by taking $K$ policy gradient steps. Note that all inner-loop updates refer to an expected policy gradient (EPG) and all bias term we discuss hereafter is the bias \emph{w.r.t} the \textbf{exact} meta-gradient in this $K$-step EPG inner-loop setting. We discuss the \textbf{truncated} $K$-step EPG setting in Appendix~\ref{dis_epg}. The form of exact meta-gradient $\nabla_{\boldsymbol{\phi}} J^{K}(\boldsymbol{\phi})$ is given by the following proposition via the chain rule.
\begin{restatable}[$K$-step Meta-Gradient]{proposition}{exactmetagradient}\label{exact_meta_gradient}
The exact meta-gradient to the objective in Eq. \eqref{GMRL_objective}
can be written as:
\begin{equation}\label{GMRL_exact_gradient}
\begin{aligned}
&\nabla_{\boldsymbol{\phi}} J^{K}(\boldsymbol{\phi})=
\nabla_{\boldsymbol{\phi}} 
J^{\text{Out}}(\boldsymbol{\phi}, \boldsymbol{\theta}^{K}) 
+ \alpha 
\nabla_{\boldsymbol{\phi}}\boldsymbol{\theta}^{K} \nabla_{\boldsymbol{\theta}^{K}} 
J^{\text {Out}}(\boldsymbol{\phi}, \boldsymbol{\theta}^{K}),\\
&\nabla_{\boldsymbol{\phi}}\boldsymbol{\theta}^{K} = 
\sum_{i=0}^{K-1}  
\nabla_{\boldsymbol{\phi}} \nabla_{\boldsymbol{\theta}^{i}} J^{\text{In}}(\boldsymbol{\phi}, \boldsymbol{\theta}^{i}) \prod_{j=i+1}^{K-1}
\left(I+\alpha \nabla^{2}_{\boldsymbol{\theta}^{j}}J^{\text{In}}(\boldsymbol{\phi}, \boldsymbol{\theta}^{j})\right).
\vspace{-10pt}
\end{aligned}
\end{equation}
\end{restatable} 
The derivation of the above proposition is in Appendix~\ref{GMRL_exact_gradient_pf}.

For each topic mentioned in Sec.~\ref{related}, we pick one GMRL algorithm example to illustrate how they can be fit into this framework.  To describe meta-gradient estimation in RL, we start with basic  notations. Consider a discrete-time finite horizon Markov Decision Process (MDP) defined by $\langle\mathcal{S}, \mathcal{A}, p, r, \gamma, H\rangle$. At each time step $t$, the RL agent observes a state $\boldsymbol{s}_{t}\in \mathcal{S}$, takes an action $\boldsymbol{a}_{t}\in\mathcal{A}$ based on the policy $\pi_{\boldsymbol{\theta}} (\boldsymbol{a}_{t}|\boldsymbol{s}_{t})$ parametrised with $\boldsymbol{\theta} \in \mathbb{R}^{d}$, transits to the next state $\boldsymbol{s}_{t+1}\in \mathcal{S}$ according to the transition function $p(\boldsymbol{s}_{t+1}|\boldsymbol{s}_{t}, \boldsymbol{a}_{t})$ and receives the reward $r(\boldsymbol{s}_{t},\boldsymbol{a}_{t})$ . We define the return $\mathcal{R}(\boldsymbol{\tau})=\sum_{t=0}^{H} \gamma^{t} r(\boldsymbol{s}_{t}, \boldsymbol{a}_{t})$ as the discounted sum of rewards along a trajectory $\boldsymbol{\tau}:=(\boldsymbol{s}_{0}, \boldsymbol{a}_{0}, \ldots, \boldsymbol{s}_{H-1}, \boldsymbol{a}_{H-1}, \boldsymbol{s}_{H})$ sampled by agent policy. The objective for the RL agent is to maximise the expected discounted sum of rewards $V(\boldsymbol{\theta})=\mathbb{E}_{\boldsymbol{\tau} \sim p(\boldsymbol{\tau} ; \boldsymbol{\theta})}[\mathcal{R}(\boldsymbol{\tau})]$. Then the RL agent updates parameter $\boldsymbol{\theta}$ using policy gradient given by $\nabla_{\boldsymbol{\theta}} V(\boldsymbol{\theta})=\mathbb{E}_{\boldsymbol{\tau} \sim p(\boldsymbol{\tau} ; \boldsymbol{\theta})}\left[\nabla_{\boldsymbol{\theta}} \log \pi_{\boldsymbol{\theta}}(\boldsymbol{\tau})\mathcal{R}(\boldsymbol{\tau})\right]$ where $\nabla_{\boldsymbol{\theta}} \log \pi_{\boldsymbol{\theta}}(\boldsymbol{\tau})=\sum_{t=0}^{H} \nabla_{\boldsymbol{\theta}} \log \pi_{\boldsymbol{\theta}}(\boldsymbol{a}_{t} | \boldsymbol{s}_{t})$. For two-agent RL problems, we can  extend the MDP to two-agent MDP (or, Stochastic games \cite{shapley1953stochastic}) defined by $\langle\mathcal{S}_{1},\mathcal{S}_{2}, \mathcal{A}_{1},\mathcal{A}_{2}, P, r_{1},r_{2}, \gamma, H\rangle$, the learning objective for agent $i$ is to maximise its value function of $V_{i}(\boldsymbol{\theta})=\mathbb{E}_{\boldsymbol{\tau}_{i} \sim p(\boldsymbol{\tau}_{i} ; \boldsymbol{\theta}_{i})}\left[\mathcal{R}_{i}(\boldsymbol{\tau}_{i})\right]$.

\input{__table.tex}

\textbf{MAML.} \citet{finn2017model} optimized over meta initial parameters to maximize the return of one-step adapted policy: $\boldsymbol{\theta}^{1} = \boldsymbol{\theta} + \alpha \nabla_{\boldsymbol{\theta}}V(\boldsymbol{\theta})$. In MAML-RL, $J^{\text{Out}}(\boldsymbol{\phi}, \boldsymbol{\theta}^{1})$ degenerates to $V(\boldsymbol{\theta}^{1})$ and $\boldsymbol{\phi}$ and $\boldsymbol{\theta}$ represent the same initial parameters. The meta-gradient can be derived in the form of Eq.~\eqref{GMRL_exact_gradient}: $\nabla_{\boldsymbol{\theta}}\boldsymbol{\theta}^{1} \nabla_{\boldsymbol{\theta}^{1}}V(\boldsymbol{\theta}^{1})$, where $\nabla_{\boldsymbol{\theta}}\boldsymbol{\theta}^{1}= I + \alpha\nabla_{\boldsymbol{\theta}}^{2}V(\boldsymbol{\theta}).$
% \begin{equation}\label{MAML_formulation}
% \nabla J(\boldsymbol{\theta}) = \nabla_{\boldsymbol{\theta}}\boldsymbol{\theta}^{1} \nabla_{\boldsymbol{\theta}}V(\boldsymbol{\theta}^{1}), \nabla_{\boldsymbol{\theta}}\boldsymbol{\theta}^{1}= I + \alpha\nabla_{\boldsymbol{\theta}}^{2}V(\boldsymbol{\theta}).
% \end{equation}

\textbf{LOLA.} \citet{foerster2017learning} proposed a new learning objective by including an additional term  that accounts for the impact of the ego policy to the anticipated opponent's gradient update. For LOLA-agent with parameters $\boldsymbol{\phi}$, it will optimise its return over one-step-lookahead opponent parameters $\boldsymbol{\theta}^{1}$. The meta-gradient can be shown as: $\nabla_{\boldsymbol{\phi}} V_{1}(\boldsymbol{\phi}, \boldsymbol{\theta}^{1}) +\nabla_{\boldsymbol{\phi}}\boldsymbol{\theta}^{1} \nabla_{\boldsymbol{\theta}^{1}} V_{1}(\boldsymbol{\phi}, \boldsymbol{\theta}^{1})$, where $\nabla_{\boldsymbol{\phi}}\boldsymbol{\theta}^{1}= \alpha\nabla_{\boldsymbol{\phi}}\nabla_{\boldsymbol{\theta}} V_{2}(\boldsymbol{\phi}, \boldsymbol{\theta})$
% \begin{equation}\label{LOLA_formulation}
% \begin{aligned}
% &\nabla J(\boldsymbol{\phi}) = \nabla_{\boldsymbol{\phi}} V_{1}(\boldsymbol{\phi}, \boldsymbol{\theta}^{1}) +\nabla_{\boldsymbol{\phi}}\boldsymbol{\theta}^{1} \nabla_{\boldsymbol{\theta}} V_{1}(\boldsymbol{\phi}, \boldsymbol{\theta}^{1}),\\
% &\nabla_{\boldsymbol{\phi}}\boldsymbol{\theta}^{1}= \alpha\nabla_{\boldsymbol{\phi}}\nabla_{\boldsymbol{\theta}} V_{2}(\boldsymbol{\phi}, \boldsymbol{\theta}).
% \end{aligned}
% \end{equation}
% \par

\textbf{MGRL.} \citet{xu2018meta} proposed to tune the discount factor $\gamma$ and bootstrapping parameter $\lambda$ in an online manner. The main feature of MGRL is to conduct inner-loop RL policy $\boldsymbol{\theta}$ update and outer-loop meta parameters $\boldsymbol{\phi}=(\gamma, \lambda)$ alternately. In MGRL, $J^{\text{Out}}(\boldsymbol{\phi}, \boldsymbol{\theta}^{1})$ degenerates to $V(\boldsymbol{\theta}^{1})$. The meta-gradient takes the form: $\nabla_{\boldsymbol{\phi}}\boldsymbol{\theta}^{1} \nabla_{\boldsymbol{\theta}^{1}} V(\boldsymbol{\theta}^{1}), \nabla_{\boldsymbol{\phi}}\boldsymbol{\theta}^{1}= \alpha\nabla_{\boldsymbol{\phi}}\nabla_{\boldsymbol{\theta}}V(\boldsymbol{\phi},\boldsymbol{\theta})$.
% \begin{equation}\label{MGRL_formulation}
% \nabla J(\boldsymbol{\phi}) = \nabla_{\boldsymbol{\phi}}\boldsymbol{\theta}^{1} \nabla_{\boldsymbol{\theta}} V(\boldsymbol{\theta}^{1}), \nabla_{\boldsymbol{\phi}}\boldsymbol{\theta}^{1}= \alpha\nabla_{\boldsymbol{\phi}}\nabla_{\boldsymbol{\theta}}V(\boldsymbol{\phi},\boldsymbol{\theta}).
% \end{equation}
% \par

\textbf{LPG.} \citet{oh2020discovering} aimed to learn a  neural network based RL algorithm, by which a RL agent can be properly trained. In LPG, $\boldsymbol{\theta}$ represents the RL agent policy parameters and $\boldsymbol{\phi}$ is the meta-parameter of neural LSTM RL algorithm, $J^{\text{In}}(\boldsymbol{\phi}, \boldsymbol{\theta})$ denotes $f(\boldsymbol{\phi}, \boldsymbol{\theta})$, which is the output of meta-network $\boldsymbol{\phi}$ for conducting inner-loop neural policy gradients. The meta-gradient can be shown as: $\nabla_{\boldsymbol{\phi}}\boldsymbol{\theta}^{1} \nabla_{\boldsymbol{\theta}^{1}} V(\boldsymbol{\theta}^{1}), \nabla_{\boldsymbol{\phi}}\boldsymbol{\theta}^{1}= \alpha \nabla_{\boldsymbol{\phi} }\nabla_{\boldsymbol{\theta}}f(\boldsymbol{\phi}, \boldsymbol{\theta})$.
% \begin{equation}\label{LPG_formulation}
% \nabla J(\boldsymbol{\phi}) = \nabla_{\boldsymbol{\phi}}\boldsymbol{\theta}^{1} \nabla_{\boldsymbol{\theta}} V(\boldsymbol{\theta}^{1}), \nabla_{\boldsymbol{\phi}}\boldsymbol{\theta}^{1}= \alpha \nabla_{\boldsymbol{\phi} }\nabla_{\boldsymbol{\theta}}f(\boldsymbol{\phi}, \boldsymbol{\theta}).
% \end{equation}

\textbf{Remark.} The analytical form of exact meta-gradient given in Eq.~\eqref{GMRL_exact_gradient} involves computation of first-order gradient $\nabla_{\boldsymbol{\theta}} J^{\text {In }}$, $\nabla_{\boldsymbol{\phi}} J^{\text {Out }}$  and$\nabla_{\boldsymbol{\theta}} J^{\text {Out }}$, Jacobian $\nabla_{\boldsymbol{\phi}}\nabla_{\boldsymbol{\theta}} J^{\text {In}}$  and Hessian $\nabla^{2}_{\boldsymbol{\theta}} J^{\text {In}}$. In practice, these four quantities can be estimated by random samples from inner-loop update step $0$ to $K-1$, which denoted by ${\boldsymbol{\tau}}^{0:K-1}_0,{\boldsymbol{\tau}}^{0:K-1}_1,{\boldsymbol{\tau}}^{0:K-1}_2,{\boldsymbol{\tau}}_3$. As a result, the estimated gradient $\nabla_{\boldsymbol{\phi}}\hat{J}^{K}(\boldsymbol{\phi})$ can be derived as:
\begin{equation}\label{GMRL_estimate_gradient}
\begin{aligned}
&\nabla_{\boldsymbol{\phi}} \hat{J}^{K}(\boldsymbol{\phi})=
\nabla_{\boldsymbol{\phi}} 
\hat{J}^{\text {Out }}(\boldsymbol{\phi}, \hat{\boldsymbol{\theta}}^{K}, \boldsymbol{\tau}_3) 
+ \alpha \nabla_{\boldsymbol{\phi}}\hat{\boldsymbol{\theta}}^{K} \nabla_{\hat{\boldsymbol{\theta}}^{K}} 
\hat{J}^{\text {Out}}(\boldsymbol{\phi},\hat{\boldsymbol{\theta}}^{K},\boldsymbol{\tau}_3), \\
&\nabla_{\boldsymbol{\phi}}\hat{\boldsymbol{\theta}}^{K} =  \sum_{i=0}^{K-1}  \nabla_{\boldsymbol{\phi}} \nabla_{\hat{\boldsymbol{\theta}}^{i}} J^{\text{In}}(\boldsymbol{\phi}, \hat{\boldsymbol{\theta}}^{i},\boldsymbol{\tau}^{i}_1) \prod_{j=t+1}^{K-1}
\left(I+\alpha \nabla^{2}_{\hat{\boldsymbol{\theta}}^{j}}J^{\text{In}}(\boldsymbol{\phi}, \hat{\boldsymbol{\theta}}^{j},\boldsymbol{\tau}^{j}_2)\right). 
\end{aligned}
\end{equation}
The post-adapt inner parameter estimate takes the form $\hat{\boldsymbol{\theta}}^{K}=\boldsymbol{\theta}^{0} + \alpha\sum_{i=0}^{K-1}\nabla_{\hat{\boldsymbol{\theta}}^{i}} \hat{J}^{\text{In}}(\boldsymbol{\phi},\hat{\boldsymbol{\theta}}^{i},\boldsymbol{\tau}^{i}_{0})$.

%% file: __table.tex
\begin{table*}[!t]
\vspace{-0.25cm}
 \centering
 \caption{Four typical gradient-based Meta-RL (GMRL) algorithms.} %$\theta^{1}$ denotes inner loop updated parameter $\theta$. }
%\small
 %\begin{threeparttable}
  \begin{tabular}{|c|c|c|c|}
   \hline
    Category & Algorithms & Meta parameter $\phi$ & Inner parameter $\theta$ 
    \\\hline\hline
    Few-shot RL &
    MAML \cite{finn2017model}  &
    Initial Parameter& 
    Initial Parameter
    \\ \hline
    Opponent Shaping &
    LOLA \cite{foerster2017learning}  &     
    Ego-agent Policy& 
    Other-agent Policy 
    \\ \hline
    Single-lifetime MGRL &
    MGRL \cite{xu2018meta} & 
    Discount Factor & 
    RL Agent Policy
   \\ \hline 
    Multi-lifetime MGRL &
    LPG \cite{oh2020discovering} & 
    LSTM Network  & 
    RL Agent Policy 
   \\ \hline 
  \end{tabular}\label{tab:unifying framework}
 %\end{threeparttable}
\end{table*}
%   \begin{tablenotes}
%  % \item $\theta^{1}$ denotes inner loop updated parameter $\theta$.
%  \end{tablenotes}
% \begin{table*}[!t]
% \vspace{-0.25cm}
%  \centering
%  \caption{Four typical gradient-based Meta-RL (GMRL) algorithms.} %$\theta^{1}$ denotes inner loop updated parameter $\theta$. }
% %\small
%  \vspace{0.1cm}
%  %\begin{threeparttable}
%   \begin{tabular}{|c|c|c|c|c|}
%   \hline
%     Category & Algorithms & Meta parameter $\phi$ & Inner parameter $\theta$ & Formulation 
%     \\\hline\hline
%     Few-shot RL &
%     MAML \cite{finn2017model}  &
%     Initial Parameter& 
%     Initial Parameter& 
%     \eqref{MAML_formulation}
%     \\ \hline
%     Opponent Shaping &
%     LOLA \cite{foerster2017learning}  &     
%     Ego-agent Policy& 
%     Other-agent Policy & 
%     \eqref{LOLA_formulation}
%     \\ \hline
%     Single-lifetime MGRL &
%     MGRL \cite{xu2018meta} & 
%     Discount Factor & 
%     RL Agent Policy& 
%     \eqref{MGRL_formulation}
%   \\ \hline 
%     Multi-lifetime MGRL &
%     LPG \cite{oh2020discovering} & 
%     LSTM Network  & 
%     RL Agent Policy & 
%     \eqref{LPG_formulation}
%   \\ \hline 
%   \end{tabular}\label{tab:unifying framework}
%     \vspace{-0.5cm}
%  %\end{threeparttable}
% % \end{table*}

%% file: core/4_analysis.tex
\section{Theoretical Analysis of Meta-gradient Estimators}\label{analysis}
In this section, we systematically discuss and theoretically analyse the bias and variance terms for meta-gradient estimations in the current GMRL literature. We highlight two important sources of biases in meta-gradient estimations: the compositional bias and the multi-step Hessian bias. 
%We will use several Meta-RL algorithms presented in Sec. \ref{framework} as examples but the analysis can generalise to all gradient-based Meta-RL algorithms. 
Our analysis builds on the following three assumptions:

\begin{assumption}[Lipschitz continuity]\label{assumption_1}
The outer-loop objective function $J^{\text {Out }}$ satisfies that
$J^{\text {Out }}(\cdot, \boldsymbol{\theta}) $ and $J^{\text {Out }}(\boldsymbol{\phi}, \cdot) $ are Lipschitz continuous with constants $m_{\boldsymbol{\theta}}$ and $m_{\boldsymbol{\phi}}$ respectively, $m_1=\max_{\boldsymbol{\theta}} m_{\boldsymbol{\theta}}$, $m_2=\max_{\boldsymbol{\phi}} m_{\boldsymbol{\phi}}$. $\nabla_{\boldsymbol{\phi}} J^{\text {Out }}(\cdot, \boldsymbol{\theta}) $ and $\nabla_{\boldsymbol{\theta}} J^{\text {Out }}(\boldsymbol{\phi}, \cdot) $ are Lipschitz continuous with constants $\mu_{\boldsymbol{\theta}}$ and $\mu_{\boldsymbol{\phi}}$ respectively,  $\mu_1=\max _{\boldsymbol{\theta}} \mu_{\boldsymbol{\theta}}$, $\mu_2=\max_{\boldsymbol{\phi}} \mu_{\boldsymbol{\phi}}$. The inner-loop objective function $J^{\text {In }}$ satisties that $\nabla_{\boldsymbol{\theta}} J^{\text {In}}(\cdot, \boldsymbol{\theta})$and $\nabla_{\boldsymbol{\theta}} J^{\text {In}}(\boldsymbol{\phi}, \cdot)$ are Lipschitz continuous with constants $c_{\boldsymbol{\theta}}$ and $c_{\boldsymbol{\phi}}$ respectively, $c_1=\max _{\boldsymbol{\theta}} c_{\boldsymbol{\theta}}$, $c_2=\max _{\boldsymbol{\phi}} c_{\boldsymbol{\phi}}$. $\nabla_{\boldsymbol{\theta}}^{2} J^{\text {In}}(\boldsymbol{\phi}, \cdot)$ is Lipschitz continuous with constants $\rho_{\boldsymbol{\phi}}$, $\rho_2=\max _{\boldsymbol{\phi}} m_{\boldsymbol{\phi}}$.
\end{assumption}
\begin{assumption}[Bias of estimators]\label{assumption_2} Outer-loop stochastic gradient estimator $\nabla_{\boldsymbol{\phi}} \hat{J}^{\text {Out }}(\boldsymbol{\phi}, \boldsymbol{\theta}, \boldsymbol{\tau})$ and $\nabla_{\boldsymbol{\theta}} \hat{J}^{\text {Out }}(\boldsymbol{\phi}, \boldsymbol{\theta}, \boldsymbol{\tau})$ are unbiased estimator of $\nabla_{\boldsymbol{\phi}} J^{\text {Out }}(\boldsymbol{\phi}, \boldsymbol{\theta})$ and $\nabla_{\boldsymbol{\theta}} J^{\text {Out }}(\boldsymbol{\phi}, \boldsymbol{\theta})$. Inner-loop stochastic gradient estimator $\nabla_{\boldsymbol{\theta}} \hat{J}^{\text {In}}(\boldsymbol{\phi}, \boldsymbol{\theta}, \boldsymbol{\tau})$ is unbiased estimator of $\nabla_{\boldsymbol{\theta}} J^{\text {In}}(\boldsymbol{\phi}, \boldsymbol{\theta})$.
\end{assumption}
\begin{assumption}[Variance of estimators]\label{assumption_3}The outer-loop stochastic gradient estimator $\nabla_{\boldsymbol{\phi}} \hat{J}^{\text {Out }}(\boldsymbol{\phi}, \boldsymbol{\theta}, \boldsymbol{\tau})$ and $\nabla_{\boldsymbol{\theta}} \hat{J}^{\text {Out }}(\boldsymbol{\phi}, \boldsymbol{\theta}, \boldsymbol{\tau})$ has bounded variance, i.e.,  $\mathbb{E}_{\tau} [\|\nabla_{\boldsymbol{\phi}} \hat{J}^{\text {Out}}(\boldsymbol{\phi}, \cdot, \boldsymbol{\tau})-\nabla_{\boldsymbol{\phi}} J^{\text {Out }}(\boldsymbol{\phi}, \cdot)\|^{2}] \leq (\sigma_{1})^{2}$, and $\mathbb{E}_{\tau} [\|\nabla_{\boldsymbol{\theta}} \hat{J}^{\text {Out }}(\boldsymbol{\phi}, \cdot, \boldsymbol{\tau})-\nabla_{\boldsymbol{\theta}} J^{\text {Out }}(\boldsymbol{\phi}, \cdot)\|^{2}] \leq (\sigma_{2})^{2}$.

%Outer loop stochastic gradient estimator $\nabla_{\boldsymbol{\theta}} \hat{J}^{\text {Out }}\left(\boldsymbol{\phi}, \boldsymbol{\theta}, \tau\right)$ has bounded variance, $\mathbb{E}_{\tau} [\|\nabla_{\boldsymbol{\theta}} \hat{J}^{\text {Out }}\left(\boldsymbol{\phi}, \cdot, \tau\right)-\nabla_{\boldsymbol{\theta}} J^{\text {Out }}\left(\boldsymbol{\phi}, \cdot\right)\|^{2}] \leq (\sigma_{2})^{2}$.
\end{assumption} 
The above three assumptions are all common ones adopted by existing work \cite{fallah2020convergence,fallah2021convergence,ji2020multi}. We futher discussed the limitation of assumptions, which are presented in Appendix~\ref{assumption_limitations}.
\vspace{-10pt}
%also unbiasedness of first-order gradient estimators are made  in \ref{assumption_2} because this paper mainly focus on the situation where only bias of high-order gradient estimators exists.

%\begin{figure}
%    \centering
%    \includegraphics[width=0.8\linewidth]{figs/Lemma4.4_fig.png}
%    \caption{Ablation study on inner learning rate and estimator.}
%    \label{fig:Lemma_4.4}
%\end{figure}

\subsection{The Compositional Bias}
\label{compositional bias}

Recall the $K$-step inner-loop meta-gradient estimate in Eq.~\eqref{GMRL_estimate_gradient}, existing GMRL methods usually get unbiased outer-loop gradient estimator {\small$\nabla_{\boldsymbol{\phi}} \hat{J}^{\text {Out }}(\boldsymbol{\phi}, \hat{\boldsymbol{\theta}}^{K}, \boldsymbol{\tau}_3)$} and {\small$\nabla_{\hat{\boldsymbol{\theta}}^{K}} \hat{J}^{\text {Out }}(\boldsymbol{\phi}, \hat{\boldsymbol{\theta}}^{K}, \boldsymbol{\tau}_3)$}, then the algorithm can get unbiased meta-gradient estimation by plugging in unbiased inner-loop gradient estimation {\small${\small\hat{\boldsymbol{\theta}}^{K}}$}, where {\small$\mathbb{E}[\hat{\boldsymbol{\theta}}^{K}]=\boldsymbol{\theta}^{K}$}. However, this is not true because of the compositional optimisation structure.

Consider a non-linear compositional scalar objective {\small${\small f(\boldsymbol{\theta}^{K})}$}, the gradient estimation bias comes from the fact that 
\[ f(\boldsymbol{\theta}^{K})= f(\mathbb{E}[\hat{\boldsymbol{\theta}}^{K}]) \neq \mathbb{E}[ f(\hat{\boldsymbol{\theta}}^{K})].\]
If one substitutes  the non-linear function {\small$f(\boldsymbol{\theta}^{K})$} with {\small$\nabla_{\boldsymbol{\theta}^{K}} J^{\text {Out }}(\boldsymbol{\phi}, \boldsymbol{\theta}^{K})$ and $\nabla_{\boldsymbol{\phi}} J^{\text {Out }}(\boldsymbol{\phi}, \boldsymbol{\theta}^{K})$}, then a typical meta-gradient estimation in GMRL introduces compositional bias:
\begin{equation}
\begin{aligned}
	&\mathbb{E}[\nabla_{\hat{\boldsymbol{\theta}}^{K}} \hat{J}^{\text {Out}}(\boldsymbol{\phi}, \hat{\boldsymbol{\theta}}^{K},\tau_3)]
	= \mathbb{E}[\nabla_{\hat{\boldsymbol{\theta}}^{K}} J^{\text {Out}}(\boldsymbol{\phi}, \hat{\boldsymbol{\theta}}^{K})] \neq
	\nabla_{\boldsymbol{\theta}^{K}} J^{\text {Out }}(\boldsymbol{\phi}, \boldsymbol{\theta}^{K}),  \\  
	&\mathbb{E}[\nabla_{\boldsymbol{\phi}}\hat{J}^{\text {Out }}(\boldsymbol{\phi}, \hat{\boldsymbol{\theta}}^{K},\tau_3)]=
	\mathbb{E}[\nabla_{\boldsymbol{\phi}} J^{\text {Out}}(\boldsymbol{\phi}, \hat{\boldsymbol{\theta}}^{K})]
	\neq\nabla_{\boldsymbol{\phi}} J^{\text {Out }}(\boldsymbol{\phi}, \boldsymbol{\theta}^{K}),
\end{aligned}
\end{equation}
which leads to meta-gradient estimation bias. The following lemma characterises compositional bias. 
\vspace{-10pt}
\begin{restatable}[Compositional Bias]{lemma}{compositionalbias}
\label{multi_step_Adaption_Error}
Suppose that Assumption \ref{assumption_1} and \ref{assumption_2} hold, let $\hat{\Delta}_{C}= \mathbb{E} [\| f(\hat{\boldsymbol{\theta}}^{K})  -  f(\boldsymbol{\theta}^{K})\| ]$ be the compositional bias and $C_0$ the Lipschitz constant of $ f(\cdot)$, $|\boldsymbol{\tau}|$ denote number of trajectories used to estimate inner-loop gradient in each inner-loop update step, $\alpha$ the learning rate, then we have,
\begin{equation}
\hat{\Delta}_{C} \leq C_0 \mathbb{E} [\|\hat{\boldsymbol{\theta}}^{K}  - \boldsymbol{\theta}^{K}\| ]\leq C_0 \left((1+\alpha c_2)^{K} -1\right) \frac{\hat{\sigma}_{\text{In}}}{c_2\sqrt{|\boldsymbol{\tau}|}},
\end{equation}
where $\hat{\sigma}_{\text{In}} = \max_{i} \sqrt{\mathbb{V}[\nabla_{\boldsymbol{\theta}^{i}}\hat{J}^{\text{In}}(\boldsymbol{\phi}, \boldsymbol{\theta}^{i}, \boldsymbol{\tau}_{0}^{i})]}$, $i\in \{0,..,K-1\}$.
\end{restatable}
\begin{proof}
See Appendix~\ref{multi_step_Adaption_Error_pf} for a detailed proof.
\vspace{-10pt}
\end{proof}
Lemma \ref{multi_step_Adaption_Error}  indicates that the compositional bias comes from the inner-loop policy gradient estimate, concerning the learning rate $\alpha$, the sample size $|\boldsymbol{\tau}|$ and the variance of policy gradient estimator $\hat{\sigma}_{\text{In}}$. This is a fundamental issue in many existing GMRL algorithms \cite{foerster2017learning,xu2018meta} since applying stochastic policy gradient update can introduce estimation errors, possibly due to large sampling variance, therefore  $\hat{\boldsymbol{\theta}}^{K}\neq\boldsymbol{\theta}^{K}$. It also  implies that the bias issue becomes more serious under the multi-step formulation since each policy gradient step introduces estimation error, resulting in  composite biases.
\subsection{The Multi-step Hessian  Bias}
\label{hessian bias}
% and Algorithm \ref{alg:main}
Recall the analytical form of the exact meta-gradient in Eq. \eqref{GMRL_exact_gradient}, estimating $\nabla_{\boldsymbol{\phi}}\boldsymbol{\theta}^{K}$ involves computing Hessian $\nabla^{2}_{\boldsymbol{\theta}^{j}}J^{\text{In}}(\boldsymbol{\phi}, \boldsymbol{\theta}^{j})$. In Eq. \eqref{GMRL_estimate_gradient},  the Hessian term is estimated by $\nabla^{2}_{\hat{\boldsymbol{\theta}}^{j}}J^{\text{In}}(\boldsymbol{\phi}, \hat{\boldsymbol{\theta}}^{j},\tau^{j}_2)$ where $j \in \{1,\ldots,K-1\}$.  Hessian estimation is a non-trivial problem in GMRL, biased Hessian estimation issue has been brought up in various MAML-RL papers\cite{liu2019taming,rothfuss2018promp,tang2021unifying}, we offer a brief summary in Appendix \ref{biased_hessian} due to page limit. % tries to address the Hessian estimation in MAML-RL setting. \citet{rothfuss2018promp} theoretically and empricially validates that original MAML implementation can result in biased Hessian and meta-gradient estimation. 
Beyond MAML-RL, many recent GMRL work suffers from  the same bias due to applying direct automatic differentiation. For example, existing work such as Eq. (3) in \citet{oh2020discovering}, Eq. (4) in \citet{zheng2020can}, Eq. (12), (13), (14) in \citet{xu2018meta} and Eq. (5), (6), (7) in \citet{xu2020meta} suffers from this issue.  Interestingly, most of them are \textbf{coincidentally} unbiased if they only conduct only one-step policy gradient update in the inner-loop. For $K$-step GMRL when $K=1$, $\nabla_{\boldsymbol{\phi}} \boldsymbol{\theta}^{1}$ in meta-gradient writes as:
\begin{equation}\label{GMRL_one_step}
\nabla_{\boldsymbol{\phi}} \boldsymbol{\theta}^{1}= \nabla_{\boldsymbol{\phi}} \nabla_{\boldsymbol{\theta}^{1}} J^{\text{In}}(\boldsymbol{\phi}, \boldsymbol{\theta}^{1}) .
\end{equation}
We can see from Eq.~\eqref{GMRL_one_step} that it would not involve Hessian $\nabla^{2}J^{\text {In }}$ computation if $\boldsymbol{\phi}\neq \boldsymbol{\theta}$. To further illustrate, in one-step MGRL, we can show that the estimation of $\nabla_{\boldsymbol{\phi}} \boldsymbol{\theta}^{1}$ are unbiased because it takes derivatives \emph{w.r.t} meta-parameters $\boldsymbol{\phi}=(\gamma, \lambda)$ which don't have gradient dependency on the trajectory distribution. However, when it takes more than one-step inner-loop policy gradient updates, the meta-gradient estimation will get  the hessian estimation bias. As a result, 
for the reason that when $K>1$ in $K$-step GMRL, the $\nabla_{\boldsymbol{\phi}} \boldsymbol{\theta}^{K}$ in meta-gradient takes the form:
\begin{equation}\label{GMRL_K_step}
\nabla_{\boldsymbol{\phi}}\boldsymbol{\theta}^{K} = \sum_{i=0}^{K-1}  \nabla_{ \boldsymbol{\phi}} \nabla_{\boldsymbol{\theta}^{i}} J^{\text{In}}\left(\boldsymbol{\phi}, \boldsymbol{\theta}^{i}\right) \prod_{j=t+1}^{K-1}\left(I+\alpha \nabla^{2}_{\boldsymbol{\theta}^{j}}J^{\text{In}}\left(\boldsymbol{\phi}, \boldsymbol{\theta}^{j}\right)\right).
\end{equation} 
This is the reason why we name it by \textbf{multi-step} Hessian bias. 
%\begin{equation}\label{GMRL_exact_gradient}
%\begin{aligned}
%&\nabla J^{K}(\boldsymbol{\phi})=\nabla_{\boldsymbol{\phi}} J^{\text {Out }}(\boldsymbol{\phi}, \boldsymbol{\theta}^{K}) + \alpha \nabla_{\boldsymbol{\phi}}\boldsymbol{\theta}^{K} \nabla_{\boldsymbol{\theta}} J^{\text {Out }}(\boldsymbol{\phi}, \boldsymbol{\theta}^{K})\\
%&\nabla_{\boldsymbol{\phi}}\boldsymbol{\theta}^{K} = \sum_{i=0}^{K-1}  \nabla \boldsymbol{\phi} \nabla\boldsymbol{\theta} J^{\text{In}}\left(\boldsymbol{\phi}, \boldsymbol{\theta}^{i}\right) \prod_{j=t+1}^{K-1}\left(I+\alpha \nabla^{2}_{\boldsymbol{\theta}}J^{\text{In}}\left(\boldsymbol{\phi}, \boldsymbol{\theta}^{j}\right)\right)
%\end{aligned}
%\end{equation}
%Recall Eq. \eqref{MAML_formulation}, where $\nabla_{\boldsymbol{\theta}} \boldsymbol{\theta}^{1}= I + \alpha\nabla_{\boldsymbol{\theta}}^{2}V(\boldsymbol{\theta})$. 
%Implementation of MAML-RL construct Hessian estimate $\nabla_{\boldsymbol{\theta}}^{2}\hat{V}(\boldsymbol{\theta})$ by automatic differentation. The corresponding estimation is biased:
%Calculating meta-gradient directly with automatic differentiation.
%\begin{equation}
%\begin{aligned}
%&\nabla_{\boldsymbol{\theta}}^{2}V(\boldsymbol{\theta}) -\mathbb{E}[\nabla_{\boldsymbol{\theta}}^{2}\hat{V}(\boldsymbol{\theta})]\\
%=&\mathbb{E}_{\tau \sim p(\tau ; \boldsymbol{\theta})}[
%\nabla_{\boldsymbol{\theta}} \log \pi_{\boldsymbol{\theta}}( \tau) \nabla_{\boldsymbol{\theta}} \log \pi_{\boldsymbol{\theta}}( \tau)^{\top}R(\tau)]
%\end{aligned}
%\label{equ:direct}
%\end{equation}
%\yaodong{this need exact citation! which formula in which paper}. 

\subsection{Theoretical Bias-Variance Analysis}\label{theoretical_analysis}

Based on Lemma \ref{multi_step_Adaption_Error} and the discussion in Sec. \ref{hessian bias}, we can  derive the upper bound on the bias and variance of  the meta-gradient with $K$-step inner-loop updates.
\begin{restatable}[Upper bound for the bias and the variance]{theorem}{biasvariance}
%\begin{theorem}[Upper bound for the bias and the variance]
\label{theorem_1}
Suppose that Assumption \ref{assumption_1} and \ref{assumption_2} and \ref{assumption_3} hold. Let $J_{\boldsymbol{\phi}, \boldsymbol{\theta}}$ denote $\nabla_{\boldsymbol{\phi}} \nabla_{\boldsymbol{\theta}} J^{\text {In}}$, $H_{\boldsymbol{\theta}, \boldsymbol{\theta}}$ denote $ \nabla^{2}_{\boldsymbol{\theta}} J^{\text {In}}$, $\hat{\Delta}^{K}= \|\mathbb{E}[\nabla_{\boldsymbol{\phi}} \hat{J}^{K}(\boldsymbol{\phi})] - \nabla_{\boldsymbol{\phi}} J^{K}(\boldsymbol{\phi})\|$ be the meta-gradient estimation bias, set $B=1+\alpha c_2$.
Then the bound of bias hold: %\yaodong{remind of what are these constants}
%\begin{equation}
%\hat{\Delta}^{K}  
%\leq 
%\mathcal{O}\Bigg(
%\hat{\Delta}_{H}^{K-1} 
%\left(\mathbb{E} [\|\hat{\boldsymbol{\theta}}^{K}  - %\boldsymbol{\theta}^{K}\|]+
%\hat{\Delta}_{J} +(K-1)
%\right) \Bigg)
%\end{equation}
\begin{equation}
\hat{\Delta}^{K}  
\leq 
\mathcal{O}\bigg(
(B + \hat{\Delta}_{H})^{K-1}
\left(\mathbb{E} [\|\hat{\boldsymbol{\theta}}^{K}  - \boldsymbol{\theta}^{K}\|]+
\hat{\Delta}_{J} +(K-1)
\right) \bigg).
\end{equation} 
%\begin{equation}
%\begin{aligned}
%    &V_{1}^{K} = 
%    \left((1+\alpha c_2)^{2} + \alpha^{2} (\hat{\Delta}_{H})^{2}\right)^{K-1} -1 \\
%    &V_{2}^{K} = \left((1+\alpha c_2)^{2} + \alpha^{2}(\hat{\Delta}_{H})^{2} + \alpha^{2}(\hat{\sigma}_{H})^{2}\right)^{K-1} -1\\
%\end{aligned}
%\end{equation}
Let $(\hat{\sigma}^{K})^{2}= \mathbb{V} \left[\nabla_{\boldsymbol{\phi}}\hat{J}^{K}(\boldsymbol{\phi})\right]$ be the meta-gradient estimation variance, set $V_1=(1+\alpha c_2)^{2}$, $V_2=2\alpha^{2} (m^{2}_{1} +3 \sigma_2^{2})$ the estimation variance is given by
% \begin{equation}
%     (\hat{\sigma}^{K})^{2}
% \leq \mathcal{O}\Bigg(
%     (V_1+\hat{\Delta}^{2}_{H})^{K-1}  
%     \left(\mathbb{E} [\|\hat{\boldsymbol{\theta}}^{K}  - \boldsymbol{\theta}^{K}\|^{2}] +
%     (K-1)\right)
%     +
%     \left(
%     V_2+(V_1+\hat{\Delta}^{2}_{H} + \hat{\sigma}^{2}_{H})^{K-1}-
%     (V_1+\hat{\Delta}^{2}_{H})^{K-1}\right)
%     (\hat{\Delta}^{2}_{J} + \hat{\sigma}^{2}_{J})
%     \Bigg).
% \end{equation}
\begin{equation}
\begin{aligned}
    &(\hat{\sigma}^{K})^{2}
\leq \mathcal{O}\Bigg(
    (V_1+\hat{\Delta}^{2}_{H})^{K-1}  
    \left(\mathbb{E} [\|\hat{\boldsymbol{\theta}}^{K}  - \boldsymbol{\theta}^{K}\|^{2}] +
    (K-1)\right)
    \\
    &  \qquad \qquad  +
    \left(
    V_2+(V_1+\hat{\Delta}^{2}_{H} + \hat{\sigma}^{2}_{H})^{K-1}-
    (V_1+\hat{\Delta}^{2}_{H})^{K-1}\right)
    (\hat{\Delta}^{2}_{J} + \hat{\sigma}^{2}_{J})
    \Bigg).
\end{aligned}
\end{equation}
{where $\hat{\Delta}_{J}=\max_{\boldsymbol{\phi}\times\boldsymbol{\theta}} \|\mathbb{E}[\hat{J}_{\boldsymbol{\phi},\boldsymbol{\theta}}]-J_{\boldsymbol{\phi}, \boldsymbol{\theta}}\|$. $\hat{\Delta}_{H}=\max_{\boldsymbol{\theta}}\|\mathbb{E}[\hat{H}_{\boldsymbol{\theta}, \boldsymbol{\theta}}]  -  H_{\boldsymbol{\theta}, \boldsymbol{\theta}}\|$. $(\hat{\sigma}_{J})^{2}= \frac{\max_{\boldsymbol{\phi}\times\boldsymbol{\theta}}\mathbb{V} [\hat{J}_{\boldsymbol{\phi}, \boldsymbol{\theta}}]}{|\tau|}$. $(\hat{\sigma}_{H})^{2}= \frac{\max_{\boldsymbol{\theta}}\mathbb{V} [\hat{H}_{\boldsymbol{\theta}, \boldsymbol{\theta}}]}{|\tau|}$.}
\vspace{-5pt}
\label{theorem_bias_variance}
\end{restatable}
%%Set
%\begin{equation}
%B^{\text{K}} = (1+\alpha c_2 + \alpha \hat{\Delta}_{H})^{K-1} -1
%\end{equation}
\begin{proof}
See Appendix \ref{theorem_1_pf} for a detailed proof.
\vspace{-5pt}
\end{proof}
Theorem \ref{theorem_1} shows that the upper bound of bias and variance consists of two parts: the first term indicates compositional bias $\mathbb{E} [\|\hat{\boldsymbol{\theta}}^{K}  - \boldsymbol{\theta}^{K}\|]$ in Lemma \ref{multi_step_Adaption_Error}, while the second term refers to the second-order  estimation bias ($\hat{\Delta}_{J},\hat{\Delta}_{H}$) and the variance ($\hat{\sigma}_{J},\hat{\sigma}_{H}$). Several observations can be made based on Theorem \ref{theorem_1}: 1) the compositional bias exists in both upper bound of bias and variance; 2) the multi-step Hessian  bias has polynomial impact on the upper bound of bias and variance; 3) most importantly, many existing GMRL algorithms suffer from the  compositional bias; moreover, the Hessian  bias can significantly increase meta-gradient  bias in the multi-step inner-loop setting. 

%% file: core/5_fix.tex
\subsection{Understanding Existing Mitigations for Meta-gradient Biases}\label{fix}
Based on the above theoretical analysis, we now try to explore and understand how existing methods can handle these two estimation biases.  
\par
\textbf{Off-policy Learning.} From the Lemma~\ref{multi_step_Adaption_Error} and the discussion in Sec.~\ref{compositional bias}, we know that the compositional bias is caused by the estimation error between $\hat{\theta}_{t}^{K}$ and $\theta_{t}^{K}$. In this case, a simple idea is that one can leverage off-policy learning technique \cite{sutton2018reinforcement} to handle the compositional bias problem by reusing samples $\tau^{i}_{0:t-1,0}$ together with $\tau^{i}_{t,0}$ to approximate $ \theta_{t}^{K}$, we want $\hat{\theta}_{t}^{K}$ to stay close to $\theta_{t}^{K}$. The intuition behind this correction is to enlarge the sample size $|\boldsymbol{\tau}|$ so that the compositional bias can be minimised according to Lemma~\ref{multi_step_Adaption_Error}. %Therefore, in Algorithm (\ref{alg:main}), at the GMRL iteration step $t$, inner-loop parameter update with off-policy correction takes the form
%\begin{equation}
%    \hat{\theta}_{t}^{K}=\hat{\theta}_{t}^{0} + \alpha
%    \sum_{i=0}^{K-1}\nabla_{\theta} \hat{J}^{\text{In}}\Big(\phi_{t},\hat{\theta}_{t}^{i},\big\{\tau^{i}_{t,0},\tau^{i}_{0:t-1,0}\big\}\Big)%.
%\end{equation}
Specifically, we can apply importance sampling technique to correct the compositional bias, i.e.,  %  of the policy gradient estimation. 
$
    \mathbb{E}_{\boldsymbol{\tau} \sim p(\boldsymbol{\tau} ; \boldsymbol{\theta})}
    \left[
    \sum_{t=0}^{H-1} 
    \frac{\pi_{\theta}(\boldsymbol{a}_{t}|\boldsymbol{s}_{t})}
    {\mu(\boldsymbol{a}_{t}|\boldsymbol{s}_{t}) } 
    \mathcal{R}_{\boldsymbol{\phi}}(\boldsymbol{\tau})
    \right].
$
In practice, we also need to manually control the  level of off-policyness to prevent  the variance introduced by the importance sampling process from increasing the estimation variance conversely. 
\textbf{Multi-step Hessian Estimator.} From the theoretical analysis in Sec.~\ref{theoretical_analysis}, we can know that the Hessian estimation bias can significantly increase the meta-gradient estimation bias in the multi-step inner-loop setting. Many low-bias hessian estimator have been proposed in the scope of MAML-RL. However, the effect of them have never been verified in the general GMRL context. Here we have a second look at the Low Variance Curvature (LVC) method \cite{rothfuss2018promp},  one can replace the original log-likelihood with the LVC operator in the policy gradient step. As such, the Hessian estimator $\nabla^{2}_{\theta}J^{\text{In}}_{\text{LVC}}(\boldsymbol{\phi},\boldsymbol{\theta})=$ takes the form:
%\begin{equation}\label{LVC_fix}
$
    \nabla_{\theta}
    \mathbb{E}_{\boldsymbol{\tau} \sim p(\boldsymbol{\tau} ; \boldsymbol{\theta})}
    \left[
    \sum_{t=0}^{H-1} 
    \frac{\nabla_{\boldsymbol{\theta}}\pi_{\boldsymbol{\theta}}(\boldsymbol{a}_{t}|\boldsymbol{s}_{t})}
    {\perp\pi_{\boldsymbol{\theta}}(\boldsymbol{a}_{t}|\boldsymbol{s}_{t})} 
     \mathcal{R}_{\boldsymbol{\phi}}(\boldsymbol{\tau})
    \right]
$
%\end{equation}
where $\perp$ is the stop-gradient operation which detaches the gradient dependency from the computation graph. As shown in \cite{rothfuss2018promp}, the LVC operator will ensure an unbiased first-order policy gradient and low-biased low-variance second-order policy gradient. Essentially, this operator only corrects the meta-gradient update and leaves the inner-loop gradient estimation formula untouched.

%% file: core/6_experiment.tex
\vspace{-10pt}
\section{Experiments}\label{experiment}
\vspace{-5pt}
In this section we conduct empirical evaluation of our proposed bias analysis and the proposed methods to mitigate the biases, and our experiments cover all 3 GMRL fields listed in Table~\ref{tab:unifying framework}. In particular, we conduct a tabular MDP experiment to show the existence of two biases discussed in Sec.~\ref{analysis} using MAML \cite{finn2017model} and LIRPG \cite{zheng2018learning}. In order to show how the proposed methods can mitigate the compositional bias, we consider a Iterated Prisoner Dilemma (IPD) problem and use LOLA \cite{foerster2017learning} with off-policy corrections in Sec.~\ref{fix}; Similarly, we conduct evaluations on Atari games using MGRL \cite{xu2018meta} with LVC corrections in Sec.~\ref{fix} to mitigate the Hessian estimation bias.  We open source our code at \url{https://github.com/Benjamin-eecs/Theoretical-GMRL}.

\subsection{Investigate the correlation and bias of meta-gradient estimators}
%\vspace{-5pt}
\label{exp_tabular}
Firstly, we conduct experiments to study different meta-gradient estimators using a tabular example using MAML and LIRPG. To align with existing works in the literature, we adopt the settings of random MDPs in \cite{tang2021unifying} with the focus on meta-gradient estimation. We refer readers to Appendix \ref{apx:tabular_setting} for more experimental details.
%\ruizhu{Maybe we can talk about the details in Appendix?} The dimension is 10 for state space and 5 for action space, so we have the reward matrix $R \in \mathbb{R}^{10\times5}$. The policy is a logits matrix $\theta^{0}\in \mathbb{R}^{10\times5}$. The final policy is obtained by adopting softmax activation on this logits matrix. 
% Different from the experiments on gradient/Hessian estimation in \cite{tang2021unifying}, we mainly focus on the meta-gradient estimation. Refer to Appendix ?? for a specific illustration of experimental setting.
%We utilise component-wise correlation and the variance as the metrics to evaluate the performance of estimation. The correlation is calculated between oracle meta-gradient and estimated meta-gradient.

\textbf{MAML-RL.} Recall that the meta-gradient in MAML is $\nabla_{\boldsymbol{\theta}^{0}}J^{\text{Out}}(\boldsymbol{\theta}^{K})$. To control the effect of gradient estimations, we use the three estimators to estimate following three terms: (I) inner-loop policy gradient $\nabla_{\boldsymbol{\theta}}J^{\text{In}}(\boldsymbol{\theta})$; (II) Jacobian/Hessian $\nabla_{\boldsymbol{\theta}}^{2}J^{\text{In}}(\boldsymbol{\theta})$; (III) outer-loop policy gradient $\nabla_{\boldsymbol{\theta}^{K}}J^{\text{Out}}(\boldsymbol{\theta}^{K})$. Refer to Appendix \ref{apx:implement} for the implementation of decomposing meta-gradient estimation with different estimators.

To show how the estimation is biased, we use a stochastic estimator denoted as S, and exact analytic calculator denoted as E, for all three derivative terms. Thus, we can have 7 valid permutations in the experiment to validate the estimation, where the rest EEE estimator is the exact gradient.

\begin{figure}[t]
    \centering
    \includegraphics[width=\linewidth]{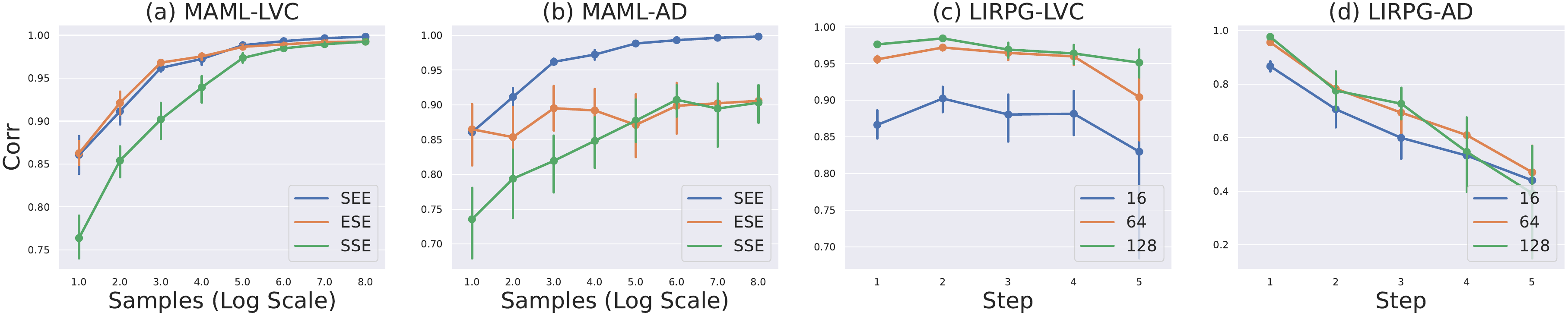}
    \vspace{-0.4cm}
    \caption{(a, b) Ablation study on sample size and estimators in MAML-RL. ``S" is for stochastic estimation while ``E" is for exact solution. AD refers to automatic differentiation. (c,d) Ablation study on sample size, steps and estimators in LIRPG.}
    \label{fig:sample}
    %\vspace{-10pt}
\end{figure}

\textbf{Firstly, we conduct our ablation studies by comparing the correlation between of meta-gradient with the exact one.} The correlation metric, which is determined by bias and variance, can show how the final estimation quality is influenced by these two bias terms. As the quality of stochastic estimators vary from many factors, we conduct this ablation study under extensive combinations of estimation algorithms (including DiCE \cite{foerster2018dice}, Loaded-DiCE \cite{farquhar2019loaded}, LVC \cite{rothfuss2018promp}, and pure automatic differentiation in original MAML \cite{finn2017model}), learning rates, sample sizes, .etc. Due to the page limit, the results illustrated in Fig.~\ref{fig:sample}(a,b) only include ablation study on sample size using LVC/automatic differentiation and estimation using exact gradients for (III), the outer-loop policy gradient. For the rest ablation study and more evaluation metrics (variance of estimation), refer to Appendix \ref{apx:tabular_result}.

We start our evaluation by increasing the sample size. For simplicity, we only conduct one inner step here. In Fig.~\ref{fig:sample}(a), we leverage the  LVC \cite{rothfuss2018promp} estimator. We can see that a correct inner-loop policy estimation and/or a Hessian estimation can significantly improve the estimation quality (as $\mathrm{SEE} \approx \mathrm{ESE} > \mathrm{SSE}$ in low sample size case). By increasing the sample size, we can see the gap between them is shrinking, which verifies the finding in Lemma~\ref{multi_step_Adaption_Error}. In this case the compositional bias correction shares the same importance with the Hessian bais correction ($\mathrm{SEE} \approx \mathrm{ESE}$). We also compare it with the original yet biased gradient estimator of MAML in Fig.~\ref{fig:sample}(b), in which $\mathrm{SEE} > \mathrm{ESE}$ in all sample size settings. In fact, since the gradient estimation is biased, only SEE achieves near $1.0$ correlation provided sufficient samples, again confirms the importance of Hessian bias corrections in Sec. \ref{hessian bias}.

\begin{figure*}[tbhp]
    \centering
    \includegraphics[width=\linewidth]{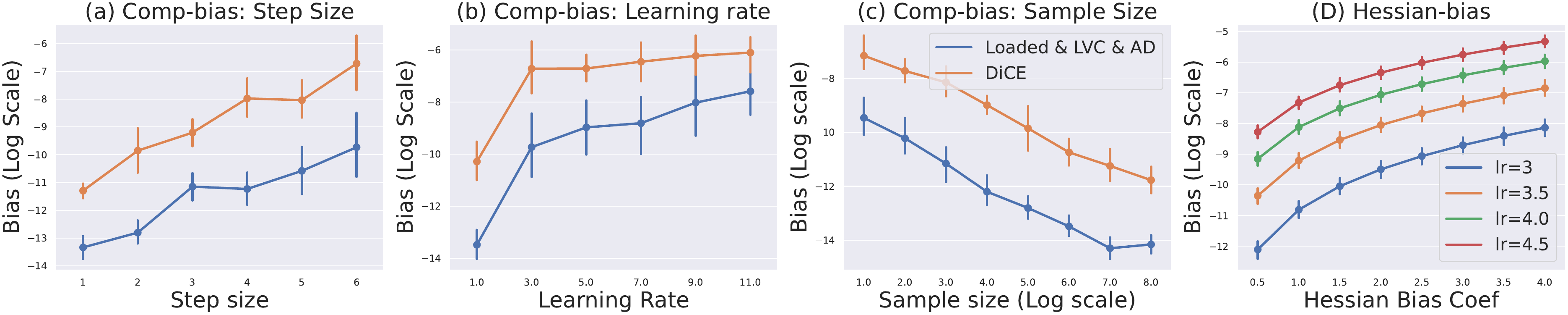}
    \caption{(a, b, c) Ablation study of meta-gradient bias due to the compositional bias in different estimators, step sizes, learning rates. Loaded-DiCE, LVC and AD achieve exactly the same compositional bias because they have the same first-order gradient,  (d) Ablation study of meta-gradient bias due to the Hessian bias in different learning rates and Hessian bias coefficients.}
    \label{fig:tabular_meta_bias}
\end{figure*}
\textbf{Beside the correlation result above, we also add additional experimental results in Fig. \ref{fig:tabular_meta_bias} over the pure meta-gradient bias term introduced by compositional bias and Hessian bias.} It can also be regarded as an empirical verification of our Lemma 4.4 and Theorem 4.5. In Fig. \ref{fig:tabular_meta_bias} (a, b, c), we mainly study How (a) the inner-loop step size, (b) learning rate and (c) sample size influence final meta-gradient bias. It successfully validates our Lemma 4.4 ($\mathcal{O}\big(K\alpha^{K}\hat{\sigma}_{\text{In}}|\tau|^{-0.5}\big)$) about the exponential impact from the inner-loop step $K$ (approximately linear relationship between log-scale bias and step size in (a)), the polynomial impact from the learning rate $\alpha$ (approximately Concave downward relationship between log-scale bias and learning rate in (b)) and the polynomial impact from the sample size $\alpha$ (approximately negative linear relationship between log-scale bias and log-scale sample size in (c)). In the second MAML-Hessian experiment, we conduct experiments to verify the polynomial impact $\mathcal{O}\big((K-1)(\hat{\Delta}_{H})^{K-1}\big)$ on the meta-gradient bias introduced by the multi-step Hessian estimation bias $\hat{\Delta}_{H}$ (The Concave downward relationship in (d)). In our implementation, we manually add the Hessian bias error into the estimation and control the quantity of it by multiplying different coefficients.
\textbf{LIRPG.} In this setting, we follow the algorithm of intrinsic reward generator presented in \cite{zheng2018learning}. In tabular MDP, we have an additional meta intrinsic reward matrix $\phi$. Starting from $\boldsymbol{\theta}^{0}$, the inner-loop process takes policy gradient based on the new reward matrix $R_{\text{new}} = R + \boldsymbol{\phi}$: $\boldsymbol{\theta}^{i+1}=\boldsymbol{\theta}^{i} + \alpha \nabla_{\boldsymbol{\theta}^{i}} J^{\text{In}}(\boldsymbol{\theta}^{i}, \phi), i \in \{0,1...K-1\}$. The meta-gradient estimation of the intrinsic reward matrix $\nabla_{\boldsymbol{\phi}}J^{\text{Out}}(\boldsymbol{\theta}^{K})$ is needed in this case. Note that in the outer loss we use the original reward matrix $R$ so the outer loss is $J^{\text{Out}}(\boldsymbol{\theta}^{K})$ rather than $J^{\text{Out}}( \boldsymbol{\phi},\boldsymbol{\theta}^{K})$. Compared with MAML-RL,  the object of meta-update (intrinsic matrix) and the object of inner-update (policy parameters) are different, which help us identify the problem mentioned in Sec. \ref{hessian bias}.\par

In this case, we choose the LVC and AD estimator. We conduct ablation study on inner-step and sample size shown in Fig.~\ref{fig:sample}(c,d). With more sample size and less step size, the correlation increases for both estimator. Two important features are: (1) With 1-step inner-loop setting, both estimator performs similarly in the correlation. (2) With multi-step inner-loop setting, LVC based estimator can still reach relatively high correlation while MAML-biased estimator directly reaches low correlation after 5-step inner-loop. The phenomenon shown here corresponds exactly to the Hessian estimation issue we discuss in Sec. \ref{hessian bias} and the bias issue will be more severe with multi-step inner-loop setting.
\begin{figure}[t]
    \centering
    \includegraphics[width=\linewidth]{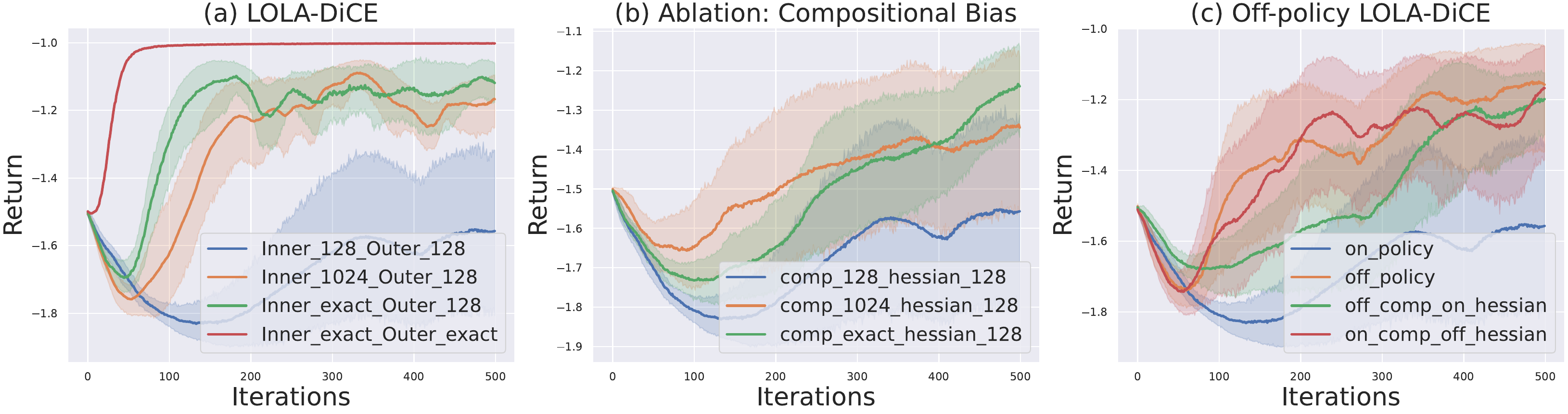}
    \caption{Experiment result of LOLA-DiCE over 10 seeds. The Inner$\_A\_$Outer$\_B$ legend means we use $A$ samples to estimate inner-loop gradient while $B$ samples to estimate outer-loop gradient. The 'exact' means we use analytical solution of policy gradient instead of estimation. 
    %(a) Ablation study on inner/outer estimation (b) Ablation study on Hessian estimation variance. (c) Ablation study on compositional bias (d) Off-policy LOLA-DiCE.}
    }
    \label{fig:lola}
    \vspace{-10pt}
\end{figure}
%\vspace{-5pt}
\subsection{Compositional bias/off-policy learning in LOLA}
\label{exp_lola}
%\vspace{-5pt}
In this subsection, we conduct three experiments on Iterated Prisoner Dilemma (IPD) with the LOLA algorithm to show: (1) The effect brought by different inner/outer estimators. (2) The effect brought compisitional bias (Sec. \ref{compositional bias}) (3) How off-policy correction (Sec. \ref{fix}) can help the LOLA algorithm. 
%We conduct our experiment by adapting code from the official codebase\footnote{\url{https://github.com/alexis-jacq/LOLA_DiCE}}. 
%To evaluate the performance reliably, we conduct all the experiments for 10 random seeds and report the average reward-iteration result. %the correlation (lower side) between estimated meta-gradient and the oracle expected meta-gradient.  
Please refer to Appendix \ref{apx_lola_dice}, \ref{apx_lola_dice_add} for more experimental setting and results.

%To mainly analyze the problem brought by inner-loop update, we recalculate the meta-gradient by resetting the outer-loop gradient in analytic solution when visualising the correlation.
%, \boldsymbol{\phi}
\textbf{Ablation on LOLA-DiCE inner/outer estimation.} We report the result of conducting ablation study for different inner/outer-loop estimation of LOLA-DiCE in the Fig.~\ref{fig:lola}(a). Here the inner-loop estimation refers to $\nabla_{\boldsymbol{\theta}} J^{\mathrm{In}}(\boldsymbol{\phi},\boldsymbol{\theta})$ while outer-loop estimation refers to $\nabla_{\boldsymbol{\theta}^{1}}J^{\mathrm{Out}}(\boldsymbol{\phi}, \boldsymbol{\theta}^{1})$ and $\nabla_{\boldsymbol{\phi}}J^{\mathrm{Out}}(\boldsymbol{\phi}, \boldsymbol{\theta}^{1})$. The return shown in Fig.~\ref{fig:lola}(a) reveals us two findings: 1) The inner-loop gradient estimation plays an important role for making LOLA   work---the default batch size 128 fails while the batch size 1024 succeeds. 2) The outer-loop gradient estimation is also  crucial to the performance of LOLA-Exact. 
Furthermore,  we continue conducting ablation studies on the inner-loop gradient update.
%(3) Higher correlation does not guarantee higher return. The bonus brought by setting inner-loop as exact solution have a really large improvement over correlation (from 0.7 to 1.0) but have limited improvement on return. We believe it is because the outer-loop gradient estimation becomes the main issue when inner-loop estimation is really well. The outer-loop first-order policy gradient estimation is a classical topic in RL and in this paper we mainly focus on the particular problem in Meta-RL brought by inner-loop gradient estimation. Thus, we focus on point (1) and further conduct ablation study on inner-loop gradient update to understand it.

\textbf{Ablation on compositional bias.} Since the unbiased DiCE estimator is used in LOLA-DiCE algorithm, there is no Hessian estimation bias in the LOLA algorithm. Thus, we mainly discuss the problem of compositional bias brought in Fig.~\ref{fig:lola}(b). We also apply the implementation   in Sec. \ref{exp_tabular} to decompose meta-gradient estimation with different estimators. 
%The naming convention of the legend is similar with the previous one. 
%For instance, comp$\_128\_$hessian$\_1024$ means that we use 128 batch size for estimator I to get $\theta_c^{\prime}$ and 1024 batch size for estimator II to get $\theta_h^{\prime}$. 
Fig.~\ref{fig:lola}(b) show us the ablation study over compositional bias, which reveals that: compositional bias may decrease the performance and by adding more samples or using analytical solution, the performance can start to improve. %Refer to Appendix \ref{apx_lola_dice_add} for more ablation study on Hessian variance.
%An interesting thing in Fig. \ref{fig:lola}(c) is that we find out the gradient correlation of these three settings are comparable. An possible explanation is that the main problem here is the hessian variance that is why the performance gain by lowering hessian variance is larger that lowering compositional bias. Though by correcting compositional bias LOLA can have better estimation with performance gain, the gain is not obvious in the aspect of gradient correlation because the hessian variance is still large.

\textbf{Off-policy DiCE and ablation study.} We use the off-policy learning to conduct inner-loop update and keep the outer-loop gradient same as before. By combing DiCE and off-policy learning, we have off-policy DiCE $J^{\mathrm{OFF-DICE}}$:
$
\label{LOLA_off_policy_corrction}\mathbb{E}_{\boldsymbol{\tau}}
    \left[
\sum_{t=0}^{H-1}\left(\prod_{t^{\prime}=0}^{t} \frac{\pi_{\phi}(\boldsymbol{a}_{t^{\prime}}^{1} \mid \boldsymbol{s}_{t^{\prime}}^{1}) \pi_{\theta}(\boldsymbol{a}_{t^{\prime}}^{2} \mid \boldsymbol{s}_{t^{\prime}}^{2})}{\mu_{1}(\boldsymbol{a}_{t^{\prime}}^{1} \mid \boldsymbol{s}_{t^{\prime}}^{1}) \mu_{2}(\boldsymbol{a}_{t^{\prime}}^{2} \mid \boldsymbol{s}_{t^{\prime}}^{2})}\right) R_{t}
\right],
$
where $\phi, \theta$ refer to the current policy, $\mu_{1}, \mu_{2}$ refer the behaviour policy for agent 1 and agent 2, respectively. H is the trajectory length and $R$ refers to the reward for agent. Note that the off-policy DiCE here can not only lower the compositional bias by lowering the first-order policy gradient error, but also helps lower the Hessian variance theoretically. By the decomposition trick, we conduct experiments by traversing over all learning settings for  (off-off/off-on/on-off/on-on), which are shown at Fig.~\ref{fig:lola}(d). Comparisons between different settings verify that off-policy DiCE can increase performance by either lowering the  compositional bias, or the hessian variance, or both. 
\begin{figure}[t]
    \centering
    \includegraphics[width=\linewidth]{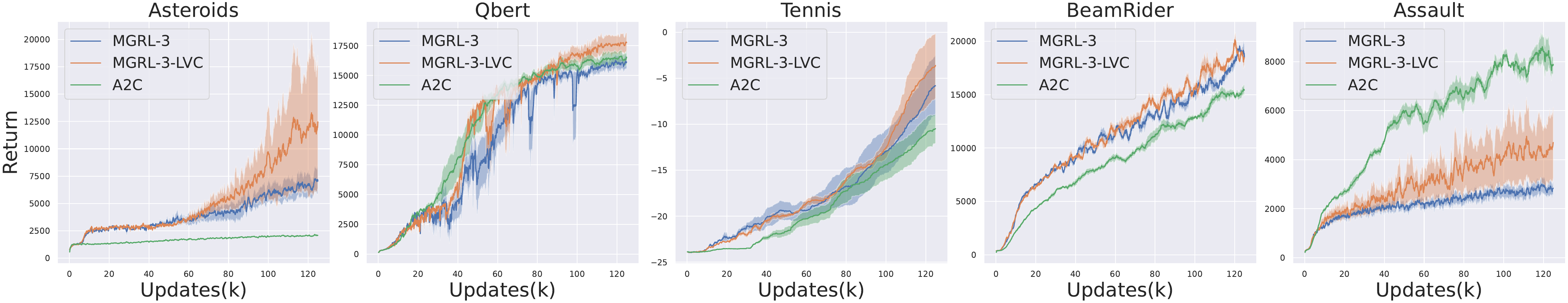}
    \caption{Experimental results on Atari game over 5 random seeds.}
    \label{fig:mgrl}
\end{figure}
%\vspace{-5pt}
\subsection{Multi-step Hessian correction on MGRL}
%\vspace{-5pt}
\label{exp_mgrl}
Finally, we conduct experiment over MGRL \cite{xu2018meta}. See Appendix \ref{apx_mgrl_setting} and \ref{apx: mgrl_dis} for experimental settings. When applying the LVC estimator in MGRL, we get the new inner-loop update equation:
%\begin{equation}
$
    % \vspace{-0.5cm}
    \mathbb{E}_{\boldsymbol{\tau} \sim p(\boldsymbol{\tau} ; \boldsymbol{\theta})}
    \left[
\sum_{t=0}^{H-1}
    \frac{\nabla_{\boldsymbol{\theta}}\pi_{\boldsymbol{\theta}}(\boldsymbol{a}_{t}|\boldsymbol{s}_{t})}
    {\perp\pi_{\boldsymbol{\theta}}(\boldsymbol{a}_{t}|\boldsymbol{s}_{t})} 
\left(g_{\boldsymbol{\phi}}(\boldsymbol{\tau})-v_{\boldsymbol{\theta}}(\boldsymbol{s}_{t})\right) 
    \right],
%\end{equation}
$
where $\boldsymbol{\phi}=(\gamma, \lambda)$ refer to meta-paramters and $\theta$ refers to the RL policy, $g_{\boldsymbol{\phi}}(\tau)$ denotes $\lambda$-return, $v_{\boldsymbol{\theta}}(s_{t})$ denotes value prediction.
% \mathcal{R}(\tau, \phi_{t})
%\nabla_{\theta}\frac{\pi_{\theta}(A|S)}{\perp \pi_{\theta}(A|S)}\
%\begin{equation}
%(g_{\eta}(\tau)-v_{\theta}(S)) \nabla_{\theta}\frac{\pi_{\theta}(A|S)}{\perp \pi_{\theta}(A|S)}+b(g_{\eta}(\tau)-v_{\theta}(S)) \frac{\partial v_{\theta}(S)}{\partial \theta}+c \frac{\partial H(\pi_{\theta}(\cdot \mid S))}{\partial \theta}
%\end{equation}
We conduct experiment on eight environments of Atari games. We follow previous work \cite{bonnet2021one} to use the "discard" strategy in which we conduct multiple virtual inner-loop updates for meta gradient estimation. This strategy is designed to keep the inner learning update unchanged. %We train it for 40M environment steps, which corresponds to 125k update times shown in the figure.
 
In Fig. (\ref{fig:mgrl}),  we show five environments comparing three variants of algorithm: 1) Baseline Advantage Actor-critic(A2C) algorithm \cite{mnih2016asynchronous}; 2) 3-step MGRL + A2C; 4) 3-step MGRL + A2C + LVC correction. The "3-step"  means we take 3 inner-loop RL virtual updates for calculating meta-gradient. Refer to Appendix \ref{apx: mgrl_all} for experimental results on all eight environments. Compared with 3-step MGRL, the MGRL with LVC correction can substantially improves the performance, which validates the effectiveness of the multi-step Hessian correction in Sec. \ref{fix} for handling meta-gradient estimation bias and bring in better hyperparameter-tuning in RL. Note that the fact that  A2C algorithms can achieve better results compared with 3-step MGRL is consistent with the results in \cite{xu2018meta}.

%% file: core/7_conclusion.tex
\section{Conclusion}

In this paper, we introduce a unified framework for studying generic meta-gradient estimations in gradient-based Meta-RL. Based on this framework, we offer two theoretical insights that \textbf{1)} the compositional bias has an upper bound of $\mathcal{O}\big(K\alpha^{K}\hat{\sigma}_{\text{In}}|\tau|^{-0.5}\big)$ with respect to the inner-loop update step $K$, the learning rate $\alpha$, the estimate variance $\hat{\sigma}^{2}_{\text{In}}$ and the sample size $|\tau|$, and \textbf{2)} the multi-step Hessian bias $\hat{\Delta}_{H}$ has a polynomial impact of $\mathcal{O}\big((K-1)(\hat{\Delta}_{H})^{K-1}\big)$. To validate our theoretical discoveries, we conduct a comprehensive list of ablation studies. Empirical results over tabular MDP, LOLA-DiCE and MGRL validate our theories and the effectiveness of  correction methods. We believe our work can inspire more future work  on unbiased meta-gradient estimations in GMRL.

%% file: core/acknowledgement.tex
\section*{Acknowledgements}
We would like to thank Yunhao Tang for insightful discussion and help for tabular experiments.
    

%% file: core/8_appendix.tex
% !TEX root = Oppo_arxiv.tex
\vspace{3ex}
\begin{center}
{\Large\textbf{Supplementary Material}}
\end{center}
\vspace{2ex}

The supplementary material is organized as follows. \Cref{more_topics} offers more algorithm illustration for the 4 topics we discuss in Sec. \ref{framework}. In \cref{dis_epg} we discuss our choice for EPG formulation and the truncated setting in GMRL. In \Cref{biased_hessian} we briefly summarise biased Hessian estimation issue in MAML-RL mentioned in Section~\ref{hessian bias}, In \Cref{assumption_limitations} we illustrate how realistic are Assumption \ref{assumption_1}-\ref{assumption_3} of Section~\ref{analysis}. \Cref{propositions,lemmas,theorems} contain the proofs for the results presented in the paper. In \Cref{supporting_lemmas} we  provide statements and proofs for some auxiliary lemmas which are instrumental for the main results. For convenience of the reader, before each proof we also restate the corresponding theorem. Finally, in \Cref{apx: experiment} we present additional experiments results.
% In \Cref{theoretical_work} we provide a more extended discussion of prior work on theoretical GMRL.
\section{More topics on GMRL}\label{more_topics}
\subsection{Few-shot Reinforcement Learning}
One important research field in Meta Reinforcement Learning is few-shot Reinforcement Learning. The main objective of this research field is to enable Reinforcement Learning agent with fast adaptation ability. Instead of thousands of interactions in traditional Reinforcement Learning algorithms, agent in few-shot setting is only allowed to interact with the new environment for a few trajectories. One of the most classical gradient based algorithms in this field is \textbf{Model Agnostic Meta Learning (MAML-RL)}. \cite{finn2017model} aims at learning neural network's initial parameters for fast adaptation on new environments. It assumes distribution $\rho(\mathcal{T})$ over RL environment $\mathcal{T}$ and tries to optimise $\boldsymbol{\theta}$ which leads to high-performing updated policy $\boldsymbol{\theta}^{\prime}$. The objective equation for one-step MAML-RL can be shown as follows:
\begin{equation}
\begin{aligned}
&J(\boldsymbol{\theta})=\mathbb{E}_{\mathcal{T} \sim \rho(\mathcal{T})}\left[\mathbb{E}_{\boldsymbol{\boldsymbol{\tau}}^{\prime} \sim P_{\mathcal{T}}\left(\boldsymbol{\boldsymbol{\tau}}^{\prime} \mid \boldsymbol{\theta}^{\prime}\right)}\left[R\left(\boldsymbol{\boldsymbol{\tau}}^{\prime}\right)\right]\right] \quad \\
&\text { with } \quad \boldsymbol{\theta}^{\prime}=\boldsymbol{\theta}+\alpha \nabla_{\boldsymbol{\theta}} \mathbb{E}_{\boldsymbol{\boldsymbol{\tau}} \sim P_{\mathcal{T}}(\boldsymbol{\boldsymbol{\tau}} \mid \boldsymbol{\theta})}[R(\boldsymbol{\boldsymbol{\tau}})]
\end{aligned}
\end{equation}
where in practice we use the limited trajectories sampled from the new environment to estimate $\nabla_{\boldsymbol{\theta}} \mathbb{E}_{\boldsymbol{\boldsymbol{\tau}} \sim P_{\mathcal{T}}(\boldsymbol{\boldsymbol{\tau}} \mid \boldsymbol{\theta})}[R(\boldsymbol{\boldsymbol{\tau}})]$. During training, by estimating meta policy gradient $\nabla_{\boldsymbol{\theta}}J(\boldsymbol{\theta})$, MAML can conduct meta update on the initial policy parameters. 

In the scope of Eq. (\ref{GMRL_objective}), MAML-RL optimizes over meta initial parameters to maximize the return of one-step adapted policy: $\boldsymbol{\theta}^{\prime} = \boldsymbol{\theta} + \alpha \nabla_{\boldsymbol{\theta}}J^{\text{In}}(\boldsymbol{\theta})$. In MAML-RL, $J^{\text{Out}}(\boldsymbol{\phi}, \boldsymbol{\theta}^\prime)$ degenerates to $J^{\text{Out}}(\boldsymbol{\theta}^\prime)$ and $\boldsymbol{\phi}$ and $\boldsymbol{\theta}$ represent the same initial parameters. The meta-gradient can be derived with the following equation:
\begin{equation}
\nabla_{\boldsymbol{\theta}} J(\boldsymbol{\theta}) = \nabla_{\boldsymbol{\theta}}\boldsymbol{\theta}^{\prime} \nabla_{\boldsymbol{\theta}^{\prime}} J^{\text {Out }}\left(\boldsymbol{\theta}^{\prime}\right), \nabla_{\boldsymbol{\theta}}\boldsymbol{\theta}^{\prime}= I + \alpha\nabla_{\boldsymbol{\theta}}^{2}J^{\text{In}}(\boldsymbol{\theta})
\end{equation}
\par

\subsection{Meta-gradient in Opponent Shaping}
\label{apx:opponent_shaping}
Opponent shaping \cite{foerster2018learning, kim2020policy,letcher2018stable} is a powerful tool in multi-agent learning process for different purposes. For instance, \citet{foerster2018learning} and \citet{letcher2018stable} have shown that putting other-players learning dynamic into self-learning process can bring in cooperation behaviors, which may help to reach better social welfare compared with purely independent learning. Meta-gradient estimation is needed when ego-agent takes derivatives of other-agent policy gradient step. \textbf{Learning with Opponent-Learning Awareness (LOLA)}
\cite{foerster2018learning} proposed a new learning objective by including an additional term accounting for the impact of ego policy to the anticipated opponent gradient update. Specifically, in the two-player setting, with agent 1 policy $\boldsymbol{\phi}$ and agent 2 policy $\boldsymbol{\theta}$, the traditional independent learning (IL) and 1-step LOLA algorithm can result in different updates for agent 1:
\begin{equation}
    \begin{aligned}
    &\boldsymbol{\phi}^{\prime}_{\text{IL}}  =\boldsymbol{\phi} + \beta \nabla_{\boldsymbol{\phi}}J^{\text{Out}}(\boldsymbol{\phi}, \boldsymbol{\theta})\\
    &\boldsymbol{\phi}^{\prime}_{\text{LOLA}}=\boldsymbol{\phi} + \beta \nabla_{\boldsymbol{\phi}}J^{\text{Out}}(\boldsymbol{\phi}, \boldsymbol{\theta}^{\prime})\\ &\text{where } \boldsymbol{\theta}^{\prime}=\boldsymbol{\theta}+\alpha\nabla_{\boldsymbol{\theta}} J^{\text{In}}(\boldsymbol{\phi}, \boldsymbol{\theta})
    \end{aligned}
\end{equation}
Where $\beta$ refers to the outer learning rate and $J^{\text{In/Out}}$ refers to the value function for agent 2 and agent 1 respectively. For meta-agent 1 with parameters $\boldsymbol{\phi}$, it will optimise its return over one-step-lookahead opponent parameters $\boldsymbol{\theta}^{\prime}$. Thus the meta-gradient of meta-agent corresponds exactly to Eq.~\eqref{GMRL_objective} with  $\nabla_{\boldsymbol{\phi}}{\boldsymbol{\theta}^{\prime}}=\alpha\nabla_{\boldsymbol{\phi}}\nabla_{\boldsymbol{\theta}}J^{\text{In}}(\boldsymbol{\phi}, \boldsymbol{\theta})$. Note that this one-step-lookahead is a just virtual update considered in the optimisation of agent 1. Agent 2 can also choose this LOLA update by conducting one-step-lookahead over agent 1.
\par

\subsection{Single-lifetime Meta-gradient RL}\label{framework_mgrl}

In this setting, the main objective is to self-tune the meta parameters ($\gamma$ in \cite{xu2018meta}) or meta models (intrinsic model in \cite{zheng2018learning}) along with the underlying normal RL updates. It is called online because it only involves one single RL life-time. This research field is also related with online hyperparameter optimisation in supervised learning such as \cite{baydin2017online,franceschi2017forward}. \citet{xu2018meta} proposed meta-gradient reinforcement learning (MGRL) to tune the discount factor $\gamma$ and bootstrapping parameter $\lambda$ in an online manner. It tries to differentiate through one RL inner update to optimize the meta-parameters and maximise one-step policy return.
\begin{equation}\label{equ:ac}
\begin{aligned}
\max _{\boldsymbol{\eta}} V^{\pi_{\boldsymbol{\theta}^{\prime}}}, &\text{where } \boldsymbol{\theta}^{\prime} = \boldsymbol{\theta} + \alpha \nabla_{\boldsymbol{\theta}} J(\boldsymbol{\tau}, \boldsymbol{\theta}, \boldsymbol{\eta}), \text{and}\\
\nabla_{\boldsymbol{\theta}} J(\boldsymbol{\tau}, \boldsymbol{\theta}, \boldsymbol{\eta})&=\left(g_{\boldsymbol{\eta}}(\boldsymbol{\tau})-v_{\boldsymbol{\theta}}(S)\right) \nabla_{\boldsymbol{\theta}} \log \pi_{\boldsymbol{\theta}}(A \mid S)\\
&+\left(g_{\boldsymbol{\eta}}(\boldsymbol{\tau})-v_{\boldsymbol{\theta}}(S)\right) \nabla_{\boldsymbol{\theta}} v_{\boldsymbol{\theta}}(S)\\
&+ \nabla_{\boldsymbol{\theta}} H\left(\pi_{\boldsymbol{\theta}}(\cdot \mid S)\right)
\end{aligned}
\end{equation}
where $\eta$ refers to $(\gamma, \lambda)$, $\boldsymbol{\tau}$ refers to trajectories, $g_{\boldsymbol{\eta}}$, $v_{\boldsymbol{\theta}}$, H represent GAE estimation, value function and entropy respectively. Eq. \eqref{equ:ac} combines actor loss, critic loss and entropy loss, which are commonly used in typical Actor-Critic \cite{mnih2016asynchronous} algorithms.
Specifically, the meta parameters $(\gamma, \lambda)$ corresponds to $\boldsymbol{\phi}$ in Eq. \eqref{GMRL_objective} . After the policy parameters $\boldsymbol{\theta}$ take one policy gradient update to become $\boldsymbol{\theta}^{\prime}$($\boldsymbol{\theta}^{\prime} = \boldsymbol{\theta} + \alpha \nabla_{\boldsymbol{\theta}} J^{\text{In}}(\boldsymbol{\theta}, \boldsymbol{\phi})$), we can calculate the meta-gradient by backpropogating from $J^{\text{Out}}$ to meta parameters. In MGRL, $J^{\text{Out}}(\boldsymbol{\phi}, \boldsymbol{\theta}^\prime)$ degenerates to $J^{\text{Out}}(\boldsymbol{\theta}^\prime)$. The meta-gradient can be shown as:
\begin{equation}
\nabla_{\boldsymbol{\phi}} J(\boldsymbol{\phi}) = \nabla_{\boldsymbol{\phi}}\boldsymbol{\theta}^{\prime} \nabla_{\boldsymbol{\theta}^{\prime}} J^{\text {Out }}\left(\boldsymbol{\theta}^{\prime}\right), \nabla_{\boldsymbol{\phi}}\boldsymbol{\theta}^{\prime}= \alpha\nabla_{\boldsymbol{\phi}}\nabla_{\boldsymbol{\theta}}J^{\text{In}}(\boldsymbol{\theta}, \boldsymbol{\phi})
\end{equation}
Here for simplicity we omit the critic and entropy loss. Usually work in this research field only conduct one-step inner-loop update before taking meta update. Some recent works such as \cite{veeriah2019discovery,bonnet2021one} have also shown that multi-step online meta-gradient can achieve better performance.

\subsection{Multi-lifetime Meta-gradient RL}\label{framework_inverse_design}

Existing work like \cite{oh2020discovering,zheng2020can,xu2020meta,feng2021neural} are trying to learn some fundamental/generalizable meta module across different environments such as a neural RL algorithm in \cite{oh2020discovering}(LPG). An important feature of multi-lifetime Meta-gradient RL is that it inherently needs multi-step inner-loop to account for the effect of fundamental meta module over the RL process. The objective of LPG is to learn a neural network based RL algorithm, by which a RL agent can be properly trained. The mathmatical formulation can be shown as follows:
\begin{gather}
J(\boldsymbol{\phi})=\mathbb{E}_{\mathcal{T} \sim \rho(\mathcal{T})}\left[\mathbb{E}_{\boldsymbol{\tau}^{K} \sim P_{\mathcal{T}}\left(\boldsymbol{\tau}^{K} \mid \theta^{K}\right)}\left[R\left(\boldsymbol{\tau}^{K}\right)\right]\right], \text{with}\\ \theta^{i}=\theta^{i-1}+\alpha \nabla_{\theta^{i-1}} \mathbb{E}_{\boldsymbol{\tau} \sim P_{\mathcal{T}}(\boldsymbol{\tau} \mid \theta^{i-1})}[ f_{\boldsymbol{\phi}}(\boldsymbol{\tau})]
\end{gather}
%\begin{gather}
%J(\boldsymbol{\phi})=\mathbb{E}_{\mathcal{T} \sim \rho(\mathcal{T})}\left[\mathbb{E}_{\boldsymbol{\tau}^{k} \sim P_{\mathcal{T}}\left(\boldsymbol{\tau}^{k} \mid \boldsymbol{\theta}^{k}\right)}\left[R\left(\boldsymbol{\tau}^{k}\right)\right]\right], \text{with}\\ \boldsymbol{\theta}^{i}=\boldsymbol{\theta}^{i-1}+\alpha \nabla_{\boldsymbol{\theta}^{i-1}} \mathbb{E}_{\boldsymbol{\boldsymbol{\tau}} \sim P_{\mathcal{T}}(\boldsymbol{\boldsymbol{\tau}} \mid \boldsymbol{\theta}^{k-1})}[ f_{\boldsymbol{\phi}}(\boldsymbol\boldsymbol{\tau})]
%\end{gather}
where $f_{\boldsymbol{\phi}}(\boldsymbol{\tau})$ is the output of meta-network $\boldsymbol{\phi}$ for conducting inner-loop neural policy gradient and $k$ can be large to show the long-range impact brought by neural RL algorithm. We omit the kl inner loss used in \cite{oh2020discovering} for simplicity. In the scope of Eq. \eqref{meta_objective}, $J^{\text{In/Out}}$ refers to the value function, $\boldsymbol{\theta}$ represents the RL agent policy parameters and $\boldsymbol{\phi}$ is the meta-parameter of neural RL algorithm. Most of works are under a multi-task/environment (or a distribution over environment) and multi-lifetime setting. \cite{xu2020meta} is a special case in these work because it is also under the online setting. We believe the main reason is that the training iterations/sample complexity in \cite{xu2020meta} is real large (1e9) and makes it become a special case of 'multi-lifetime' setting.

%Recently, one work \cite{vuorio2021no} argues that: (1) the general %unbiased meta-gradient for MAML-RL (\cite{finn2017model}) and Online %meta-gradient (\cite{xu2018meta,zheng2018learning}) should be the K-sample inner-loop meta-gradient shown in E-MAML \citep{al2017continuous} rather than the expected policy gradient inner-loop meta-gradient used in many recent work (\cite{rothfuss2018promp}, \cite{liu2019taming}, \cite{tang2021unifying}). (2) The gradient estimator in online meta-gradient utilise truncated optimization and the unbiased meta-gradient should be the one in untruncated setting. We need to clarify that in our paper, unlike \cite{vuorio2021no}, we still focuses on the previous work (MAML-RL/MGRL and LOLA) objectives with EPG inner-loop setting and use its meta-gradient as our target gradient. All bias term we discuss is the bias w.r.t. the expected meta-gradient in this EPG inner-loop and truncated setting (it is truncated because usually we set the step as a fixed value). The reasons of choosing this meta-gradient as the target gradient rather than that in \cite{vuorio2021no} are specifically discussed in Appendix \ref{dis_epg}. 

\section{Discussion of expected policy gradient (EPG) formulation and truncated setting}
\label{dis_epg}
We discuss 4 research topics in Section \ref{framework}: few-shot RL(MAML-RL), opponent shaping(LOLA-DiCE), online meta gradient RL(MGRL) and meta gradient based inverse design(LPG). And we need to discuss how this multi-step EPG inner-loop formulation differs in these topics. Though they all need meta policy gradient estimation, the differences between setting and final objective require us to discuss them separately.

\textbf{Different setting}: MAML-RL and most inverse design algorithms are under multi-lifetime setting which can renew an environment and restart the RL training from the very beginning. Work in online meta gradient RL/LOLA only happen in a single lifetime RL process. There only exists one RL training process.

\textbf{Different objectives}: For MAML-RL, the main objective is to maximise the return of few-step adapted policy. Thus the objective corresponds exactly to few-step inner-loop formulation. However, for topics beyond few-shot RL, in most case they need to measure the influence of meta module over RL final (after thousands of steps) performance.

There are two important issues in this EPG formulation. The first one is that it assumes an expected policy gradient inner-loop update. And the second one is because we only consider few-step inner-loop update so they are under a truncated estimation setting which might bring in bias. Recently, one work \cite{vuorio2021no} argues that: (1) the general unbiased meta gradient for MAML-RL (\cite{finn2017model}) and Online Meta Gradient (\cite{xu2018meta},\cite{zheng2018learning}) should be the K-sample inner-loop meta gradient shown in E-MAML \cite{al2017continuous} rather than the expected policy gradient inner-loop meta gradient used in many recent work \cite{liu2019taming, rothfuss2018promp, tang2021unifying}. (2) The gradient estimator in online meta gradient utilise truncated optimization and the unbiased meta gradient should be the one in untruncated setting. 

Overall we agree that: (1) The K-sample inner-loop meta gradient estimator is unbiased for MAML-RL problem when sampled policy gradient are used. (2) To learn an schedule (rather than a global meta module) of meta-parameter/meta-module for MGRL or to learn some fundamental concepts in inverse-design, the gradient estimator in untruncated setting is unbiased. However, we argue that (1) For MAML-related problem, the variance of sampling correction term in K-sample inner-loop meta gradient estimator is large because it needs to sum up all $k$ terms and that is why \cite{vuorio2021no} proposes to use one coefficient to control. The EPG can achieve lower variance estimation and perform better empirically \cite{rothfuss2018promp} (2) For meta gradient based inverse design with multi-lifetime, the few-step meta gradient estimation under truncated setting is biased. 

However, in online meta gradient setting (MGRL) or online opponent modelling (LOLA) with single-lifetime, things are completely different thus a direct transform of K-sample inner-loop formulation from MAML to MGRL might not be that straightforward. There exists a large gap between the implementation of online meta gradient algorithm and the final objective (meta-module/hyperparameters schedule) we may wish. First, it's an online setting so the multiple lifetime setting where the algorithm can restart from the very beginning and reiterate the whole process is banned here. This makes the estimation of unbiased meta gradient impossible because the algorithm cannot access to the future dynamic for gradient estimation. The experiments with multi-lifetime training in \cite{vuorio2021no} is in fact out of the scope of online meta gradient setting and are more like meta gradient based inverse design. Second, in implementation of MGRL they only maintain one running $\gamma$ or intrinsic model rather than multiple meta modules as a real schedule needs. Also, recently there exist one work \cite{bonnet2021one} discussing multi-step MGRL and use one fixed meta parameters rather than a schedule for multi-step inner-loop, which may show a different understanding about untruncated gradient. In all, we believe that what online meta algorithm/opponent shaping like MGRL or LOLA optimizes and what the best they can achieve in such online setting are still open questions and remain to be further explored. It is really hard to simply formulate the unbiased meta gradient since the gap between implementation and objective is still not clear. 

Thus, in our paper, we still focuses on the previous work (MAML/MGRL and LOLA) objectives with EPG inner-loop setting and use its meta gradient as our target gradient. All bias term we discuss is the bias w.r.t. the expected meta gradient in this EPG inner-loop and truncated setting. That is our work's limitation and we leave more things for future work: (1)The gap between EPG inner-loop meta gradient and K-sample inner-loop meta gradient in MAML-RL related problem. (2) The gap between truncated EPG inner-loop meta gradient and what the best gradient estimation we can get in online meta gradient/opponent shaping. (3) The gap between truncated EPG inner-loop meta gradient and the untruncated gradient in meta gradient based inverse design.

% \section{Algorithm in Section~\ref{framework}}\label{algorithms}

%\input{__algo.tex}

\section{Brief summary on biased Hessian estimation in MAML-RL}
\label{biased_hessian}

We will briefly introduce the reasons of biased Hessian estimation with automatic differentiation in one-step MAML-RL. Firstly, we can derive the analytic form of $\boldsymbol{\theta}^{1}$ and $\nabla_{\boldsymbol{\theta}^{0}}\boldsymbol{\theta}^{1}$ 
\begin{equation}
    \boldsymbol{\theta}^{1} = \boldsymbol{\theta}^{0} + \alpha \mathbb{E}_{\boldsymbol{\tau} \sim p(\boldsymbol{\tau} ; \boldsymbol{\theta}^{0})}[\nabla_{\boldsymbol{\theta}^{0}} \log \pi(\boldsymbol{\tau} ) \mathcal{R}(\boldsymbol{\tau})]
\end{equation}
\begin{equation}
\nabla_{\boldsymbol{\theta}^{0}} \boldsymbol{\theta}^{1}=I+\mathbb{E}_{\boldsymbol{\tau} \sim p(\boldsymbol{\tau} ; \boldsymbol{\theta}^{0})}\left[\mathcal{R}(\boldsymbol{\tau})\left(\nabla_{\boldsymbol{\theta}^{0}}^{2} \log \pi_{\boldsymbol{\theta}^{0}}(\boldsymbol{\tau})+\nabla_{\boldsymbol{\theta}^{0}} \log \pi_{\boldsymbol{\theta}^{0}}(\boldsymbol{\tau}) \nabla_{\boldsymbol{\theta}^{0}} \log \pi_{\boldsymbol{\theta}^{0}}(\boldsymbol{\tau})^{\top}\right)\right]
\end{equation}
Typically we need to use trajectory samples $\boldsymbol{\tau}_{n}$ to estimate the policy gradient, we can get the adapted policy estimate.
\begin{equation}
    \hat{\boldsymbol{\theta}}^{1} = \boldsymbol{\theta}^{0} + \alpha \frac{1}{N} \sum_{\boldsymbol{\tau}_{n}} \sum_{t=0}^{H-1} \nabla_{\boldsymbol{\theta}}\log \pi_{\boldsymbol{\theta}}(\boldsymbol{a}^{n}_{t} \mid \boldsymbol{s}^{n}_{t})\left(\sum_{t^{\prime}=0}^{H} \gamma^{t} r\left(\boldsymbol{s}^{n}_{t^{\prime}}, \boldsymbol{a}^{n}_{t^{\prime}}\right)\right)
\end{equation}
Finally, implementation of MAML-RL derives the gradient estimate by automatic differentation. The corresponding estimation is biased:
\begin{equation}
\begin{aligned}
\mathbb{E}[\nabla_{\boldsymbol{\theta}^{0}} \hat{\boldsymbol{\theta}}^{1}]&=I+\alpha\mathbb{E}_{\boldsymbol{\tau} \sim p(\boldsymbol{\tau} ; \boldsymbol{\theta}^{0})}\left[\frac{1}{N} \sum_{\boldsymbol{\tau}_{n}} \sum_{t=0}^{H-1} \nabla_{\boldsymbol{\theta}^{0}}^{2}\log \pi_{\boldsymbol{\theta}^{0}}(\boldsymbol{a}^{n}_{t} \mid \boldsymbol{s}^{n}_{t})\left(\sum_{t^{\prime}=0}^{H} \gamma^{t} r(\boldsymbol{s}^{n}_{t^{\prime}}, \boldsymbol{a}^{n}_{t^{\prime}})\right)\right]\\
&=I+\alpha\mathbb{E}_{\boldsymbol{\tau} \sim p(\boldsymbol{\tau} ; \boldsymbol{\theta}^{0})}\left[\mathcal{R}(\boldsymbol{\tau})\nabla_{\boldsymbol{\theta}}^{2} \log \pi_{\boldsymbol{\theta}^{0}}(\boldsymbol{\tau})\right]\not=\nabla_{\boldsymbol{\theta}^{0}}\boldsymbol{\theta}^{1}
\end{aligned}
\end{equation}
The main reason of biased Hessian estimation is that automatic differentiation tools only consider the dependency of $\boldsymbol{\theta}$ in $\nabla_{\boldsymbol{\theta}}\log\pi_{\boldsymbol{\theta}}$ while ignoring the dependency in expectation $\mathbb{E}_{\boldsymbol{\tau} \sim p(\boldsymbol{\tau} ; \boldsymbol{\theta}^{0})}$. In practice, the $\mathbb{E}_{\boldsymbol{\tau} \sim p(\boldsymbol{\tau} ; \boldsymbol{\theta}^{0})}$ is represented by trajectory sampling so the gradient term $\nabla_{\boldsymbol{\theta}}\mathbb{E}_{\boldsymbol{\tau} \sim p(\boldsymbol{\tau} ; \boldsymbol{\theta}^{0})}$ is 0 using automatic differentiation. We need to add additional terms to further derive the gradient $\nabla_{\boldsymbol{\theta}}\mathbb{E}_{\boldsymbol{\tau} \sim p(\boldsymbol{\tau} ; \boldsymbol{\theta}^{0})}$ brought by sampling dependency.

\section{Limitations on Assumptions}
\label{assumption_limitations}

Assumption \ref{assumption_1}-\ref{assumption_3} are standard assumptions used in various theoretical MAML-RL papers \cite{fallah2020convergence,fallah2021convergence,ji2020multi}. 
The Lipschitz continuity assumptions in Assumption \ref{assumption_1} make sure we can work with nonconvex inner and outer objectives. 
The unbiased first-order gradient estimators assumptions in Assumption \ref{assumption_2} can highlight our findings on two source of biases, which is also a plausible assumption in GMRL settings. As typically adopted in the analysis for stochastic optimization, we make the bounded-variance assumption in Assumption \ref{assumption_3}.
Assumption \ref{assumption_1}-\ref{assumption_3} can be conveniently verified for e.g., inner-loop RL optimization in tabular MDP settings (finite state space and action space) with soft-max parameterisation of the policy, where
$\pi_{\boldsymbol{\theta}}(\boldsymbol{a} \mid \boldsymbol{s})\propto \exp (\boldsymbol{\theta}(\boldsymbol{s}, \boldsymbol{a}))$ with parameter $\boldsymbol{\theta}={\boldsymbol{\theta}(\boldsymbol{s}, \boldsymbol{a})}$. But in large-scale RL settings like atari games, Assumption \ref{assumption_1}-\ref{assumption_3} will not hold anymore.

% \revision{\section{More discussions on Theoretical GMRL}
% \label{theoretical_work}}

% \revision{Due to the highly complex objective landscape of GMRL, most theoretical analysis focuses on convergence to stationary points. Fallah et al. \cite{fallah2020convergence} established generic convergence guarantees for gradient-based meta-learning algorithms for supervised learning with one
% inner-loop update. Recently, Ji et al.\cite{ji2020multi} extended the analysis to multi-step inner loop updates. For meta-RL, Fallah et al. \cite{fallah2021convergence} established convergence for the EMAML objective. In our paper, we consider a different expected inner-loop update in general Meta-RL problem. And we are the first to theoretically and empirically investigate the compositional bias and related upper bound in RL setting, and the effect of multi-step Hessian bias in upper bound of meta-gradient bias out of the scope of MAML-RL.}

\section{Proof of Proposition in Section~\ref{framework}}\label{propositions}
In this section, we provide the proof for Proposition~\ref{exact_meta_gradient} in Section~\ref{framework}.
\subsection{Proof of Proposition~\ref{exact_meta_gradient}}
\label{GMRL_exact_gradient_pf}

\exactmetagradient*
\begin{proof}
According to post-update inner parameters $\boldsymbol{\theta}^{K}=
\boldsymbol{\theta}^{0} + \alpha\sum_{i=0}^{K-1} \nabla_{\boldsymbol{\theta}^{i}} J^{\text{In}}(\boldsymbol{\phi},\boldsymbol{\theta}^{i})$ and the fact that $\nabla_{\boldsymbol{\theta}^{i}} J^{\text{In}}(\boldsymbol{\phi},\boldsymbol{\theta}^{i})$ is differentiable w.r.t. $\boldsymbol{\phi}$, we can treat $\boldsymbol{\theta}^{K}$ as a differentiable function w.r.t. $\boldsymbol{\phi}$. Based on the chain rule, we can get

\begin{equation}\label{chain_rule}
\begin{aligned}
    \nabla_{\boldsymbol{\phi}} J^{K}(\boldsymbol{\phi}) 
    &= \nabla_{\boldsymbol{\phi}} J^{\text {Out}}(\boldsymbol{\phi}, \boldsymbol{\theta}^{K}) + \nabla_{\boldsymbol{\phi}} \boldsymbol{\theta}^{K} \nabla_{\boldsymbol{\theta}^{K}} J^{\text {Out }}(\boldsymbol{\phi}, \boldsymbol{\theta}^{K})\\
\end{aligned}
\end{equation}

Based on the iterative updates that $\boldsymbol{\theta}^{i+1}=\boldsymbol{\theta}^{i}+\alpha \nabla_{\boldsymbol{\theta}^{i}} J^{In}(\boldsymbol{\phi}, \boldsymbol{\theta}^{i}) $, for 
$i=0, \ldots, K-1$ and similarly treat $\boldsymbol{\theta}^{i}$ as a differentiable function w.r.t. $\boldsymbol{\phi}$, we have

\begin{equation}
\begin{aligned}
\nabla_{\boldsymbol{\phi}} \boldsymbol{\theta}^{i+1}
=&\nabla_{\boldsymbol{\phi}} \boldsymbol{\theta}^{i} + \alpha \nabla_{\boldsymbol{\phi}} \nabla_{\boldsymbol{\theta}^{i}}J^{\text{In}}(\boldsymbol{\phi}, \boldsymbol{\theta}^{i}) +\alpha \nabla_{\boldsymbol{\phi}} \boldsymbol{\theta}^{i} \nabla^{2}_{\boldsymbol{\theta}^{i}}J^{\text{In}}(\boldsymbol{\phi}, \boldsymbol{\theta}^{i})  \\
=& \nabla_{\boldsymbol{\phi}} \boldsymbol{\theta}^{i} \left(I+\alpha \nabla^{2}_{\boldsymbol{\theta}^{i}}J^{\text{In}}(\boldsymbol{\phi}, \boldsymbol{\theta}^{i})\right) + \alpha \nabla_{\boldsymbol{\phi}} \nabla_{\boldsymbol{\theta}^{i}} J^{\text{In}}(\boldsymbol{\phi}, \boldsymbol{\theta}^{i})\\
\end{aligned}
\end{equation}
Telescoping the above equality over $i$ from 0 to $K-1$, we can get

\begin{equation}\label{best_response}
\begin{aligned}
\nabla_{\boldsymbol{\phi}} \boldsymbol{\theta}^{K}=&
\nabla_{\boldsymbol{\phi}} \boldsymbol{\theta}^{0} \prod_{i=0}^{K-1}\left(I+\alpha \nabla^{2}_{\boldsymbol{\theta}^{i}}J^{\text{In}}(\boldsymbol{\phi}, \boldsymbol{\theta}^{i})\right)+\alpha \sum_{i=0}^{K-1} \nabla_{\boldsymbol{\phi}} \nabla_{\boldsymbol{\theta}^{i}} J^{\text{In}}(\boldsymbol{\phi}, \boldsymbol{\theta}^{i}) \prod_{j=i+1}^{K-1}\left(I+\alpha \nabla^{2}_{\boldsymbol{\theta}^{j}}J^{\text{In}}(\boldsymbol{\phi}, \boldsymbol{\theta}^{j})\right) \\
=&\alpha \sum_{i=0}^{K-1} \nabla_{\boldsymbol{\phi}} \nabla_{\boldsymbol{\theta}^{i}} J^{\text{In}}(\boldsymbol{\phi}, \boldsymbol{\theta}^{i}) \prod_{j=i+1}^{K-1}\left(I+\alpha \nabla^{2}_{\boldsymbol{\theta}^{j}}J^{\text{In}}(\boldsymbol{\phi}, \boldsymbol{\theta}^{j})\right)
\end{aligned}
\end{equation}

Combining Eq.~\eqref{chain_rule} and Eq.~\eqref{best_response} finishes the proof of Proposition~\ref{exact_meta_gradient}.
\end{proof}
\section{Proof of Lemma in Section~\ref{analysis}}\label{lemmas}

\subsection{Proof of Lemma \ref{multi_step_Adaption_Error}} \label{multi_step_Adaption_Error_pf}

\compositionalbias*
\begin{proof}
In expected policy gradient inner-loop update setting, the iterative updates takes the form
\begin{equation}\label{lemma:eq_1}
    \boldsymbol{\theta}^{i+1} = \boldsymbol{\theta}^{i} + \alpha \nabla_{\boldsymbol{\theta}^{i}} J^{\text {In}}(\boldsymbol{\phi}, \boldsymbol{\theta}^{i}) ,\text{ }i=0, \ldots, K-1
\end{equation}
In Eq.~\eqref{GMRL_estimate_gradient}, $\boldsymbol{\theta}^{i+1}$ are estimated using samples $\boldsymbol{\tau}^{0:i}_{0}$, then we have
\begin{equation}\label{lemma:eq_2}
    \hat{\boldsymbol{\theta}}^{i+1} = \hat{\boldsymbol{\theta}}^{i} + \alpha \nabla_{\hat{\boldsymbol{\theta}}^{i}} \hat{J}^{\text {In}}(\boldsymbol{\phi}, \hat{\boldsymbol{\theta}}^{i},\boldsymbol{\tau}^{i}_{0}) ,\text{ } \hat{\boldsymbol{\theta}}^{0} =\boldsymbol{\theta}^{0},\text{ } i=0, \ldots, K-1
\end{equation}
According to the assumption that non-linear compositional vector-valued $f(\cdot)$ is Lipschitz continuous with constant $C_0$, we can get 
\begin{equation}\label{lemma:eq_0}
    \mathbb{E}_{\boldsymbol{\tau}^{0:K-1}_0} [\| f(\hat{\boldsymbol{\theta}}^{K})  -  f(\boldsymbol{\theta}^{K})\| ] \leq
    C_0 \mathbb{E}_{\boldsymbol{\tau}^{0:K-1}_0} \left[\left\|\hat{\boldsymbol{\theta}}^{K}  - \boldsymbol{\theta}^{K}\right\| \right]
\end{equation}
Based on Eq.~\eqref{lemma:eq_1} and Eq.~\eqref{lemma:eq_2}, we can get
\begin{equation}
\begin{aligned}
&\mathbb{E}_{\boldsymbol{\tau}^{0:K-1}_0} \left[\left\|\hat{\boldsymbol{\theta}}^{K}  - \boldsymbol{\theta}^{K}\right\| \right] \\
=& \mathbb{E}_{\boldsymbol{\tau}^{0:K-1}_0} \left[\bigg\|\hat{\boldsymbol{\theta}}^{K-1}  - \boldsymbol{\theta}^{K-1} + \alpha \nabla_{\boldsymbol{\theta}^{K-1}}J^{\text{In}}(\boldsymbol{\phi},\boldsymbol{\theta}^{K-1}) - \alpha \nabla_{\hat{\boldsymbol{\theta}}^{K-1}}\hat{J}^{\text{In}}(\boldsymbol{\phi},\hat{\boldsymbol{\theta}}^{K-1},\boldsymbol{\tau}^{K-1}_0)\bigg\|  \right]\\
\overset{(i)}\leq& \mathbb{E}_{\boldsymbol{\tau}^{0:K-1}_0}\left[\bigg\| \hat{\boldsymbol{\theta}}^{K-1}  - \boldsymbol{\theta}^{K-1} \bigg\|\right] + \alpha \mathbb{E}_{\boldsymbol{\tau}^{0:K-1}_0}\left[\bigg\| \nabla_{
\boldsymbol{\theta}^{K-1}}J^{\text{In}}(\boldsymbol{\phi},\boldsymbol{\theta}^{K-1})
-
\mathbb{E}_{\boldsymbol{\tau}^{K-1}_0}[\nabla_{\hat{\boldsymbol{\theta}}^{K-1}}\hat{J}^{\text{In}}(\boldsymbol{\phi},\hat{\boldsymbol{\theta}}^{K-1},\boldsymbol{\tau}^{K-1}_0)] \bigg\|\right] +\\
&\alpha \mathbb{E}_{\boldsymbol{\tau}^{0:K-1}_0}\left[\bigg\| \mathbb{E}_{\boldsymbol{\tau}^{K-1}_0}[\nabla_{
\hat{\boldsymbol{\theta}}^{K-1}}\hat{J}^{\text{In}}(\boldsymbol{\phi},\hat{\boldsymbol{\theta}}^{K-1},\boldsymbol{\tau}^{K-1}_0)] - \nabla_{
\hat{\boldsymbol{\theta}}^{K-1}}\hat{J}^{\text{In}}(\boldsymbol{\phi},\hat{\boldsymbol{\theta}}^{K-1},\boldsymbol{\tau}^{K-1}_0) \bigg\|\right]\\
\leq& \mathbb{E}_{\boldsymbol{\tau}^{0:K-1}_0}\left[\bigg\| \hat{\boldsymbol{\theta}}^{K-1}  - \boldsymbol{\theta}^{K-1} \bigg\|\right] + \alpha c_2 \mathbb{E}_{\boldsymbol{\tau}^{0:K-1}_0}\left[\bigg\| \hat{\boldsymbol{\theta}}^{K-1}
-
 \boldsymbol{\theta}^{K-1} \bigg\|\right] +\\
&\alpha \mathbb{E}_{\boldsymbol{\tau}^{0:K-1}_0}\left[\bigg\| \mathbb{E}_{\boldsymbol{\tau}^{K-1}_0}[\nabla_{
\hat{\boldsymbol{\theta}}^{K-1}}\hat{J}^{\text{In}}(\boldsymbol{\phi},\hat{\boldsymbol{\theta}}^{K-1},\boldsymbol{\tau}^{K-1}_0)] - \nabla_{\hat{\boldsymbol{\theta}}^{K-1}}\hat{J}^{\text{In}}(\boldsymbol{\phi},\hat{\boldsymbol{\theta}}^{K-1},\boldsymbol{\tau}^{K-1}_0) \bigg\|\right]\\
\leq& (1+\alpha c_{2})\mathbb{E}_{\boldsymbol{\tau}^{0:K-1}_0}\left[\bigg\| \hat{\boldsymbol{\theta}}^{K-1}  - \boldsymbol{\theta}^{K-1} \bigg\|\right] + \\
&\alpha \mathbb{E}_{\boldsymbol{\tau}^{0:K-2}_0}\Bigg[\mathbb{E}_{\boldsymbol{\tau}^{K-1}_0}\left[\bigg\| \mathbb{E}_{\boldsymbol{\tau}^{K-1}_0}[\nabla_{
\hat{\boldsymbol{\theta}}^{K-1}}\hat{J}^{\text{In}}(\boldsymbol{\phi},\hat{\boldsymbol{\theta}}^{K-1},\boldsymbol{\tau}^{K-1}_0)] - \nabla_{\hat{\boldsymbol{\theta}}^{K-1}}\hat{J}^{\text{In}}(\boldsymbol{\phi},\hat{\boldsymbol{\theta}}^{K-1},\boldsymbol{\tau}^{K-1}_0) \bigg\|\mid\boldsymbol{\tau}^{0:K-2}_0\right]\Bigg]\\
\leq&(1+\alpha c_{2})\mathbb{E}_{\boldsymbol{\tau}^{0:K-1}_0} \left[\left\|\hat{\boldsymbol{\theta}}^{K-1}  - \boldsymbol{\theta}^{K-1}\right\|  \right] + 
\alpha \mathbb{E}_{\boldsymbol{\tau}^{0:K-2}_0}\left[
\sqrt{\frac{\mathbb{V}\left[\nabla_{\hat{\boldsymbol{\theta}}^{K-1}}\hat{J}^{\text{In}} \left(\boldsymbol{\phi},\hat{\boldsymbol{\theta}}^{K-1},\boldsymbol{\tau}^{K-1}_0\right)\mid\boldsymbol{\tau}^{0:K-2}_0\right]}{|\boldsymbol{\tau}^{K-1}_0|}}
\right]
\end{aligned}
\end{equation}

where $(i)$ follows from Lemma~\ref{aux_lemma_2} and Assumption~\ref{assumption_2}.

Let $\hat{\sigma}_{\text{In}} = \max_{i} \sqrt{\mathbb{V}[\nabla_{\hat{\boldsymbol{\theta}}^{i}}\hat{J}^{\text{In}}(\boldsymbol{\phi}, \hat{\boldsymbol{\theta}}^{i}, \tau_0^{i})]}$, $|\boldsymbol{\tau}| = |\tau_{0}^{i}|$, $i\in \{0,\ldots,K-1\}$

\begin{equation}
\mathbb{E}_{\boldsymbol{\tau}^{0:K-1}_0} \left[\left\|\hat{\boldsymbol{\theta}}^{K}  - \boldsymbol{\theta}^{K}\right\| \right] \\
\leq(1+\alpha c_{2})\mathbb{E}_{\boldsymbol{\tau}^{0:K-1}_0} \left[\left\|\hat{\boldsymbol{\theta}}^{K-1}  - \boldsymbol{\theta}^{K-1}\right\|  \right] + 
\alpha 
\frac{\hat{\sigma}_{\text{In}}}{\sqrt{|\boldsymbol{\tau}|}}
\end{equation}

Iteratively, we can get

\begin{equation}
\begin{aligned}
\mathbb{E}_{\boldsymbol{\tau}^{0:K-1}_0} \left[\left\|\hat{\boldsymbol{\theta}}^{K}  - \boldsymbol{\theta}^{K}\right\| \right] 
&\leq\left(1+\ldots+(1+\alpha c_{2})^{K-1} \right) 
\alpha \frac{\hat{\sigma}_{\text{In}}}{\sqrt{|\boldsymbol{\tau}|}}\\
&= \left((1+\alpha c_2)^{K} -1\right) \frac{\hat{\sigma}_{\text{In}}}{c_2\sqrt{|\boldsymbol{\tau}|}}
\end{aligned}
\end{equation}
which concludes the proof of Lemma \ref{multi_step_Adaption_Error}.
\end{proof}

\section{Proof of Theorem in Section \ref{analysis}}\label{theorems}

\subsection{Proof of Theorem \ref{theorem_1}}\label{theorem_1_pf}
\biasvariance*

\begin{proof}
According to Proposition~\ref{exact_meta_gradient}, exact meta-gradient $\nabla_{\boldsymbol{\phi}} J^{\text{K}}(\boldsymbol{\phi})$ takes the form
\begin{equation}
\nabla_{\boldsymbol{\phi}} J^{\text {Out }}(\boldsymbol{\phi}, \boldsymbol{\theta}^{K})+\alpha \sum_{i=0}^{K-1}  \nabla_{\boldsymbol{\phi}} \nabla_{\boldsymbol{\theta}^{i}} J^{\text{In}} (\boldsymbol{\phi}, \boldsymbol{\theta}^{i}) \prod_{j=i+1}^{K-1}\left(I+\alpha \nabla^{2}_{\boldsymbol{\theta}^{j}}J^{\text{In}}(\boldsymbol{\phi}, \boldsymbol{\theta}^{j})\right) \nabla_{\boldsymbol{\theta}^{K}} J^{\text {Out }}(\boldsymbol{\phi}, \boldsymbol{\theta}^{K})
\end{equation}

where
\begin{equation}
    \boldsymbol{\theta}^{i+1} = \boldsymbol{\theta}^{i} + \alpha \nabla_{\boldsymbol{\theta}^{i}} J^{\text {In}}(\boldsymbol{\phi}, \boldsymbol{\theta}^{i}) ,\text{ }i=0, \ldots, K-1
\end{equation}

Acoordingly, in Eq.~\eqref{GMRL_estimate_gradient}, $K$-step meta-gradient estimator $\nabla_{\boldsymbol{\phi}} \hat{J}^{\text{K}}(\boldsymbol{\phi})$ takes the form

\begin{equation}
    \nabla_{\boldsymbol{\phi}} \hat{J}^{\text {Out }}(\boldsymbol{\phi}, \hat{\boldsymbol{\theta}}^{K},\boldsymbol{\tau}_3)
    +\alpha  \sum_{i=0}^{K-1}
    \nabla_{\boldsymbol{\phi}} \nabla_{\hat{\boldsymbol{\theta}}^{i}} \hat{J}^{\text{In}}(\boldsymbol{\phi},\hat{\boldsymbol{\theta}}^{i},\boldsymbol{\tau}^{i}_1) 
    \prod_{j=i+1}^{K-1}
    \left(I+\alpha  \nabla^{2}_{\hat{\boldsymbol{\theta}}^{j}}\hat{J}^{\text{In}} (\boldsymbol{\phi}, \hat{\boldsymbol{\theta}}^{j},\boldsymbol{\tau}^{j}_{2})\right) 
    \nabla_{\hat{\boldsymbol{\theta}}^{K}} \hat{J}^{\text {Out }}(\boldsymbol{\phi}, \hat{\boldsymbol{\theta}}^{K},\tau_{3})
\end{equation}

where
\begin{equation}
    \hat{\boldsymbol{\theta}}^{i+1} = \hat{\boldsymbol{\theta}}^{i} + \alpha \nabla_{\hat{\boldsymbol{\theta}}^{i}} \hat{J}^{\text {In}}(\boldsymbol{\phi}, \hat{\boldsymbol{\theta}}^{i},\boldsymbol{\tau}^{i}_{0}) ,\text{ } \hat{\boldsymbol{\theta}}^{0} =\boldsymbol{\theta}^{0},\text{ } i=0, \ldots, K-1
\end{equation}

Hence the expectation of meta-gradient estimator takes the form

\begin{equation}
\begin{aligned}
    &\mathbb{E}_{\boldsymbol{\tau}^{0:K-1}_0, \boldsymbol{\tau}^{0:K-1}_1, \boldsymbol{\tau}^{1:K-1}_2, \boldsymbol{\tau}_3}[\nabla_{\boldsymbol{\phi}} \hat{J}^{\text{K}}(\boldsymbol{\phi})] \\
    =& \mathbb{E}_{\boldsymbol{\tau}^{0:K-1}_0} \Bigg[ \mathbb{E}_{\boldsymbol{\tau}_3 } [\nabla_{\boldsymbol{\phi}} \hat{J}^{\text {Out }}(\boldsymbol{\phi}, \hat{\boldsymbol{\theta}}^{K}, \boldsymbol{\tau}_3)\mid \boldsymbol{\tau}^{0:K-1}_0 ]
    +\alpha  \sum_{i=0}^{K-1} \mathbb{E}_{\boldsymbol{\tau}^{i}_1} [\nabla_{\boldsymbol{\phi}} \nabla_{\hat{\boldsymbol{\theta}}^{i}} \hat{J}^{\text{In}}(\boldsymbol{\phi}, \hat{\boldsymbol{\theta}}^{i},\boldsymbol{\tau}^{i}_1) \mid \boldsymbol{\tau}^{0:i-1}_0] \times\\
    &\prod_{j=i+1}^{K-1} \mathbb{E}_{\boldsymbol{\tau}^{j}_2} [I+\alpha \nabla^{2}_{\hat{\boldsymbol{\theta}}^{j}}\hat{J}^{\text{In}}(\boldsymbol{\phi}, \hat{\boldsymbol{\theta}}^{j} ,\boldsymbol{\tau}^{j}_{2}) \mid \boldsymbol{\tau}^{0:j-1}_0]  \times \mathbb{E}_{\boldsymbol{\tau}_3 } [ \nabla_{\hat{\boldsymbol{\theta}}^{K}} \hat{J}^{\text {Out }}(\boldsymbol{\phi}, \hat{\boldsymbol{\theta}}^{K},\tau_{3}) \mid \boldsymbol{\tau}^{0:K-1}_0 ] \Bigg]
\end{aligned}
\end{equation}

we can then derive meta-gradient bias in $K$-step expected policy gradient setting,

\begin{equation}
\begin{aligned}
&\left\|\mathbb{E}_{\boldsymbol{\tau}^{0:K-1}_0, \boldsymbol{\tau}^{0:K-1}_1, \boldsymbol{\tau}^{0:K-1}_2, \boldsymbol{\tau}_3}[\nabla_{\boldsymbol{\phi}} \hat{J}^{\text{K}}(\boldsymbol{\phi})] - \nabla_{\boldsymbol{\phi}} J^{\text{K}}(\boldsymbol{\phi})\right\|\\
\leq &\mathbb{E}_{\boldsymbol{\tau}^{0:K-1}_0} \Bigg[ \bigg\| \mathbb{E}_{\boldsymbol{\tau}_3 } [\nabla_{\boldsymbol{\phi}} \hat{J}^{\text {Out }}(\boldsymbol{\phi}, \hat{\boldsymbol{\theta}}^{K}, \boldsymbol{\tau}_3)\mid \boldsymbol{\tau}^{0:K-1}_0 ] +\alpha  \sum_{i=0}^{K-1} \mathbb{E}_{\boldsymbol{\tau}^{i}_1} [\nabla_{\boldsymbol{\phi}} \nabla_{\hat{\boldsymbol{\theta}}^{i}} \hat{J}^{\text{In}}(\boldsymbol{\phi}, \hat{\boldsymbol{\theta}}^{i},\boldsymbol{\tau}^{i}_1) \mid \boldsymbol{\tau}^{0:i-1}_0] \times\\
&\prod_{j=i+1}^{K-1} \mathbb{E}_{\boldsymbol{\tau}^{j}_2} [I+\alpha \nabla^{2}_{\hat{\boldsymbol{\theta}}^{j}}\hat{J}^{\text{In}}(\boldsymbol{\phi}, \hat{\boldsymbol{\theta}}^{j} ,\boldsymbol{\tau}^{j}_{2}) \mid \boldsymbol{\tau}^{0:i-1}_0]  \times\mathbb{E}_{\boldsymbol{\tau}_3 } [ \nabla_{\hat{\boldsymbol{\theta}}^{K}} \hat{J}^{\text {Out }}(\boldsymbol{\phi}, \hat{\boldsymbol{\theta}}^{K},\tau_{3}) \mid \boldsymbol{\tau}^{0:K-1}_0 ] - \\
&\nabla_{\boldsymbol{\phi}} J^{\text {Out }}(\boldsymbol{\phi}, \boldsymbol{\theta}^{K})-\alpha \sum_{i=0}^{K-1} \nabla_{\boldsymbol{\phi}} \nabla_{\boldsymbol{\theta}^{i}} J^{\text{In}} (\boldsymbol{\phi}, \boldsymbol{\theta}^{i}) \prod_{j=i+1}^{K-1}\left(I+\alpha \nabla^{2}_{\boldsymbol{\theta}^{j}}J^{\text{In}}(\boldsymbol{\phi}, \boldsymbol{\theta}^{j})\right) \nabla_{\boldsymbol{\theta}^{K}} J^{\text {Out }}(\boldsymbol{\phi}, \boldsymbol{\theta}^{K})\bigg\|\Bigg] \\
\end{aligned}
\end{equation}

\vspace{-10pt}
\begin{equation}
\begin{aligned}
\leq&\mathbb{E}_{\boldsymbol{\tau}^{0:K-1}_0} \Bigg[ \bigg\| \mathbb{E}_{\boldsymbol{\tau}_3 } [\nabla_{\boldsymbol{\phi}} \hat{J}^{\text {Out }}(\boldsymbol{\phi}, \hat{\boldsymbol{\theta}}^{K}, \boldsymbol{\tau}_3)\mid \boldsymbol{\tau}^{0:K-1}_0 ] - \nabla_{\boldsymbol{\phi}} J^{\text {Out }}(\boldsymbol{\phi}, \boldsymbol{\theta}^{K}) \bigg\|\Bigg] \\
&+\mathbb{E}_{\boldsymbol{\tau}^{0:K-1}_0} \Bigg[ \bigg\| \alpha  \sum_{i=0}^{K-1} \mathbb{E}_{\boldsymbol{\tau}^{i}_1} [\nabla_{\boldsymbol{\phi}} \nabla_{\hat{\boldsymbol{\theta}}^{i}} \hat{J}^{\text{In}}(\boldsymbol{\phi}, \hat{\boldsymbol{\theta}}^{i},\boldsymbol{\tau}^{i}_1) \mid \boldsymbol{\tau}^{0:i-1}_0] \times\prod_{j=i+1}^{K-1} \mathbb{E}_{\boldsymbol{\tau}^{j}_2} [I+\alpha \nabla^{2}_{\hat{\boldsymbol{\theta}}^{j}}\hat{J}^{\text{In}}(\boldsymbol{\phi}, \hat{\boldsymbol{\theta}}^{j} ,\boldsymbol{\tau}^{j}_{2}) \mid \boldsymbol{\tau}^{0:i-1}_0]  \\
&\times \mathbb{E}_{\boldsymbol{\tau}_3 } [ \nabla_{\hat{\boldsymbol{\theta}}^{K}} \hat{J}^{\text {Out }}(\boldsymbol{\phi}, \hat{\boldsymbol{\theta}}^{K},\tau_{3}) \mid \boldsymbol{\tau}^{0:K-1}_0 ]  \\
&-\alpha \sum_{i=0}^{K-1} \nabla_{\boldsymbol{\phi}} \nabla_{\boldsymbol{\theta}^{i}} J^{\text{In}} (\boldsymbol{\phi}, \boldsymbol{\theta}^{i}) \prod_{j=i+1}^{K-1}\left(I+\alpha \nabla^{2}_{\boldsymbol{\theta}^{j}}J^{\text{In}}\left(\boldsymbol{\phi}, \boldsymbol{\theta}^{j}\right)\right) \nabla_{\boldsymbol{\theta}^{K}} J^{\text {Out}}(\boldsymbol{\phi}, \boldsymbol{\theta}^{K})\bigg\|\Bigg] \\
\overset{(i)}\leq&\mathbb{E}_{\boldsymbol{\tau}^{0:K-1}_0} \Bigg[\bigg\|   \nabla_{\boldsymbol{\phi}} J^{\text {Out}}(\boldsymbol{\phi}, \hat{\boldsymbol{\theta}}^{K})  - \nabla_{\boldsymbol{\phi}} J^{\text {Out}}(\boldsymbol{\phi}, \boldsymbol{\theta}^{K})  \bigg\|\mid \boldsymbol{\tau}^{0:K-1}_0 \Bigg] \\
&+\mathbb{E}_{\boldsymbol{\tau}^{0:K-1}_0} \Bigg[ \bigg\| \alpha  \sum_{i=0}^{K-1} \mathbb{E}_{\boldsymbol{\tau}^{i}_1} [\nabla_{\boldsymbol{\phi}} \nabla_{\hat{\boldsymbol{\theta}}^{i}} \hat{J}^{\text{In}}(\boldsymbol{\phi}, \hat{\boldsymbol{\theta}}^{i},\boldsymbol{\tau}^{i}_1) \mid \boldsymbol{\tau}^{0:i-1}_0] \times\prod_{j=i+1}^{K-1} \mathbb{E}_{\boldsymbol{\tau}^{j}_2} [I+\alpha \nabla^{2}_{\hat{\boldsymbol{\theta}}^{j}}\hat{J}^{\text{In}}(\boldsymbol{\phi}, \hat{\boldsymbol{\theta}}^{j} ,\boldsymbol{\tau}^{j}_{2}) \mid \boldsymbol{\tau}^{0:i-1}_0]  \\
&\times \mathbb{E}_{\boldsymbol{\tau}_3 } [ \nabla_{\hat{\boldsymbol{\theta}}^{K}} \hat{J}^{\text {Out }}(\boldsymbol{\phi}, \hat{\boldsymbol{\theta}}^{K},\tau_{3}) \mid \boldsymbol{\tau}^{0:K-1}_0 ] - \\
&\alpha \sum_{i=0}^{K-1} \nabla_{\boldsymbol{\phi}} \nabla_{\boldsymbol{\theta}^{i}} J^{\text{In}}(\boldsymbol{\phi}, \boldsymbol{\theta}^{i}) \prod_{j=i+1}^{K-1}\left(I+\alpha \nabla^{2}_{\boldsymbol{\theta}^{j}}J^{\text{In}}\left(\boldsymbol{\phi}, \boldsymbol{\theta}^{j}\right)\right) \nabla_{\boldsymbol{\theta}^{K}} J^{\text {Out }}(\boldsymbol{\phi}, \boldsymbol{\theta}^{K})\bigg\|\Bigg] \\
\end{aligned}
\end{equation}
where $(i)$ follows from Assumption~\ref{assumption_2}.
\begin{equation}
\begin{aligned}
\overset{(ii)}\leq& \mu_1\mathbb{E}_{\boldsymbol{\tau}^{0:K-1}_0} \Bigg[\bigg\|   \hat{\boldsymbol{\theta}}^{K}  - \boldsymbol{\theta}^{K}  \bigg\|\mid \boldsymbol{\tau}^{0:K-1}_0 \Bigg] \\
&+\mathbb{E}_{\boldsymbol{\tau}^{0:K-1}_0} \Bigg[ \bigg\| \alpha  \sum_{i=0}^{K-1} \mathbb{E}_{\boldsymbol{\tau}^{i}_1} [\nabla_{\boldsymbol{\phi}} \nabla_{\hat{\boldsymbol{\theta}}^{i}} \hat{J}^{\text{In}}(\boldsymbol{\phi}, \hat{\boldsymbol{\theta}}^{i},\boldsymbol{\tau}^{i}_1) \mid \boldsymbol{\tau}^{0:i-1}_0] \times\prod_{j=i+1}^{K-1} \mathbb{E}_{\boldsymbol{\tau}^{j}_2} [I+\alpha \nabla^{2}_{\hat{\boldsymbol{\theta}}^{j}}\hat{J}^{\text{In}}(\boldsymbol{\phi}, \hat{\boldsymbol{\theta}}^{j} ,\boldsymbol{\tau}^{j}_{2}) \mid \boldsymbol{\tau}^{0:i-1}_0]  
\\
&\times \mathbb{E}_{\boldsymbol{\tau}_3 } [ \nabla_{\hat{\boldsymbol{\theta}}^{K}} \hat{J}^{\text {Out }}(\boldsymbol{\phi}, \hat{\boldsymbol{\theta}}^{K},\tau_{3}) \mid \boldsymbol{\tau}^{0:K-1}_0 ] -\alpha \sum_{i=0}^{K-1} \nabla_{\boldsymbol{\phi}} \nabla_{\boldsymbol{\theta}^{i}} J^{\text{In}}(\boldsymbol{\phi}, \boldsymbol{\theta}^{i}) \prod_{j=i+1}^{K-1}\left(I+\alpha \nabla^{2}_{\boldsymbol{\theta}^{j}}J^{\text{In}}\left(\boldsymbol{\phi}, \boldsymbol{\theta}^{j}\right)\right) \nabla_{\boldsymbol{\theta}^{K}} J^{\text {Out }}(\boldsymbol{\phi}, \boldsymbol{\theta}^{K})\bigg\|\Bigg] \\
\end{aligned}
\end{equation}
where $(ii)$ follows from Assumption~\ref{assumption_1} on Lipschitz Continuity of $\nabla_{\boldsymbol{\phi}} J^{\text {Out}}$ 
%and Assumption~\ref{assumption_2} concerning the outer-loop unbiased estimator.

\begin{equation}
\begin{aligned}
\overset{(iii)}\leq& \mu_1\mathbb{E}_{\boldsymbol{\tau}^{0:K-1}_0} \Bigg[\bigg\|   \hat{\boldsymbol{\theta}}^{K}  - \boldsymbol{\theta}^{K}  \bigg\|\mid \boldsymbol{\tau}^{0:K-1}_0 \Bigg] \\
&+ \alpha \sum_{i=0}^{K-1} \mathbb{E}_{\boldsymbol{\tau}^{0:K-1}_0} \Bigg[ \bigg\|    \mathbb{E}_{\boldsymbol{\tau}^{i}_1} [\nabla_{\boldsymbol{\phi}} \nabla_{\hat{\boldsymbol{\theta}}^{i}} \hat{J}^{\text{In}}(\boldsymbol{\phi}, \hat{\boldsymbol{\theta}}^{i},\boldsymbol{\tau}^{i}_1) \mid \boldsymbol{\tau}^{0:i-1}_0] \times\\
&\prod_{j=i+1}^{K-1} \mathbb{E}_{\boldsymbol{\tau}^{j}_2} [I+\alpha \nabla^{2}_{\hat{\boldsymbol{\theta}}^{j}}\hat{J}^{\text{In}}(\boldsymbol{\phi}, \hat{\boldsymbol{\theta}}^{j} ,\boldsymbol{\tau}^{j}_{2}) \mid \boldsymbol{\tau}^{0:i-1}_0]  \times \mathbb{E}_{\boldsymbol{\tau}_3 } [ \nabla_{\hat{\boldsymbol{\theta}}^{K}} \hat{J}^{\text {Out }}(\boldsymbol{\phi}, \hat{\boldsymbol{\theta}}^{K},\tau_{3}) \mid \boldsymbol{\tau}^{0:K-1}_0 ] - \\
& \nabla_{\boldsymbol{\phi}} \nabla_{\boldsymbol{\theta}^{i}} J^{\text{In}}(\boldsymbol{\phi}, \boldsymbol{\theta}^{i}) \prod_{j=i+1}^{K-1}\left(I+\alpha \nabla^{2}_{\boldsymbol{\theta}^{j}}J^{\text{In}}\left(\boldsymbol{\phi}, \boldsymbol{\theta}^{j}\right)\right) \nabla_{\boldsymbol{\theta}^{K}} J^{\text {Out }}(\boldsymbol{\phi}, \boldsymbol{\theta}^{K})\bigg\|\Bigg] \\
\end{aligned}
\end{equation}
where $(iii)$ follows from Lemma~\ref{aux_lemma_2}. Using the similar add-minus trick in the proof of Lemma \ref{multi_step_Adaption_Error}, we can have
\begin{equation}
\begin{aligned}
\overset{(iv)}\leq& \mu_1\mathbb{E}_{\boldsymbol{\tau}^{0:K-1}_0} \Bigg[\bigg\|   \hat{\boldsymbol{\theta}}^{K}  - \boldsymbol{\theta}^{K}  \bigg\|\mid \boldsymbol{\tau}^{0:K-1}_0 \Bigg] \\
&+ \alpha \sum_{i=0}^{K-1} \mathbb{E}_{\boldsymbol{\tau}^{0:K-1}_0} \Bigg[ \bigg\| \prod_{j=i+1}^{K-1} \mathbb{E}_{\boldsymbol{\tau}^{j}_2} [I+\alpha \nabla^{2}_{\hat{\boldsymbol{\theta}}^{j}}\hat{J}^{\text{In}}(\boldsymbol{\phi}, \hat{\boldsymbol{\theta}}^{j} ,\boldsymbol{\tau}^{j}_{2}) \mid \boldsymbol{\tau}^{0:i-1}_0]  \times \mathbb{E}_{\boldsymbol{\tau}_3 } [ \nabla_{\hat{\boldsymbol{\theta}}^{K}} \hat{J}^{\text {Out }}(\boldsymbol{\phi}, \hat{\boldsymbol{\theta}}^{K},\tau_{3}) \mid \boldsymbol{\tau}^{0:K-1}_0 ] - \\
&\prod_{j=i+1}^{K-1}\left(I+\alpha \nabla^{2}_{\boldsymbol{\theta}^{j}}J^{\text{In}}\left(\boldsymbol{\phi}, \boldsymbol{\theta}^{j}\right)\right) \nabla_{\boldsymbol{\theta}^{K}} J^{\text {Out }}(\boldsymbol{\phi}, \boldsymbol{\theta}^{K})\bigg\|\Bigg] \times \|\nabla_{\boldsymbol{\phi}} \nabla_{\boldsymbol{\theta}^{i}} J^{\text{In}}(\boldsymbol{\phi}, \boldsymbol{\theta}^{i})\| + \\
&\mathbb{E}_{\boldsymbol{\tau}^{0:K-1}_0} \Bigg[ \bigg\|\mathbb{E}_{\boldsymbol{\tau}^{i}_1} [\nabla_{\boldsymbol{\phi}} \nabla_{\hat{\boldsymbol{\theta}}^{i}} \hat{J}^{\text{In}}(\boldsymbol{\phi}, \hat{\boldsymbol{\theta}}^{i},\boldsymbol{\tau}^{i}_1) \mid \boldsymbol{\tau}^{0:i-1}_0] - \nabla_{\boldsymbol{\phi}} \nabla_{\boldsymbol{\theta}^{i}} J^{\text{In}}(\boldsymbol{\phi}, \boldsymbol{\theta}^{i})\bigg\|\Bigg] \times \\
&\Bigg\| \mathbb{E}_{\boldsymbol{\tau}^{0:K-1}_0}\bigg[ \prod_{j=i+1}^{K-1} \mathbb{E}_{\boldsymbol{\tau}^{j}_2} [I-\nabla^{2}_{\hat{\boldsymbol{\theta}}^{j}}\hat{J}^{\text{In}}(\boldsymbol{\phi}, \hat{\boldsymbol{\theta}}^{j} ,\boldsymbol{\tau}^{j}_{2}) \mid \boldsymbol{\tau}^{0:i-1}_0]  \times \mathbb{E}_{\boldsymbol{\tau}_3 } [ \nabla_{\hat{\boldsymbol{\theta}}^{K}} \hat{J}^{\text {Out }}(\boldsymbol{\phi}, \hat{\boldsymbol{\theta}}^{K},\tau_{3}) \mid \boldsymbol{\tau}^{0:K-1}_0 ]\bigg]\Bigg\|\\
\end{aligned}
\end{equation}
Based on Assumption~\ref{assumption_1} and Assumption~\ref{assumption_2}, we can change the expectation of unbiased first-order stochastic estimator to respective first-order gradient function, then we can replace it with Lipschitz constants.
\begin{equation}
\begin{aligned}
\leq& \mu_1\mathbb{E}_{\boldsymbol{\tau}^{0:K-1}_0} \Bigg[\bigg\|   \hat{\boldsymbol{\theta}}^{K}  - \boldsymbol{\theta}^{K}  \bigg\|\mid \boldsymbol{\tau}^{0:K-1}_0 \Bigg] \\
&+ \alpha \sum_{i=0}^{K-1}\Bigg[ \mathbb{E}_{\boldsymbol{\tau}^{0:K-1}_0} \bigg[ \bigg\| \prod_{j=i+1}^{K-1} \mathbb{E}_{\boldsymbol{\tau}^{j}_2} [I+\alpha \nabla^{2}_{\hat{\boldsymbol{\theta}}^{j}}\hat{J}^{\text{In}}(\boldsymbol{\phi}, \hat{\boldsymbol{\theta}}^{j} ,\boldsymbol{\tau}^{j}_{2}) \mid \boldsymbol{\tau}^{0:i-1}_0] - \prod_{j=i+1}^{K-1}\left(I+\alpha \nabla^{2}_{\boldsymbol{\theta}^{j}}J^{\text{In}}\left(\boldsymbol{\phi}, \boldsymbol{\theta}^{j}\right)\right)\bigg\|\bigg] \times\\
&\bigg\| \nabla_{\boldsymbol{\theta}^{K}} J^{\text{Out}}(\boldsymbol{\phi}, \boldsymbol{\theta}^{K})\bigg\| + \mathbb{E}_{\boldsymbol{\tau}^{0:K-1}_0} \Bigg[ \bigg\|\mathbb{E}_{\boldsymbol{\tau}_3 } [ \nabla_{\hat{\boldsymbol{\theta}}^{K}} \hat{J}^{\text {Out }}(\boldsymbol{\phi}, \hat{\boldsymbol{\theta}}^{K},\tau_{3}) \mid \boldsymbol{\tau}^{0:K-1}_0 ] - \nabla_{\boldsymbol{\theta}^{K}} J^{\text{Out}}(\boldsymbol{\phi}, \boldsymbol{\theta}^{K})\bigg\|\Bigg] \times\\
&\bigg\|\mathbb{E}_{\boldsymbol{\tau}^{0:K-1}_0} \bigg[  \prod_{j=i+1}^{K-1} \mathbb{E}_{\boldsymbol{\tau}^{j}_2} [I+\alpha \nabla^{2}_{\hat{\boldsymbol{\theta}}^{j}}\hat{J}^{\text{In}}(\boldsymbol{\phi}, \hat{\boldsymbol{\theta}}^{j} ,\boldsymbol{\tau}^{j}_{2}) \mid \boldsymbol{\tau}^{0:i-1}_0]\bigg] \bigg\|\Bigg] \times \|\nabla_{\boldsymbol{\phi}} \nabla_{\boldsymbol{\theta}^{i}} J^{\text{In}}(\boldsymbol{\phi}, \boldsymbol{\theta}^{i})\| + \\
&\mathbb{E}_{\boldsymbol{\tau}^{0:K-1}_0} \Bigg[ \bigg\|\mathbb{E}_{\boldsymbol{\tau}^{i}_1} [\nabla_{\boldsymbol{\phi}} \nabla_{\hat{\boldsymbol{\theta}}^{i}} \hat{J}^{\text{In}}(\boldsymbol{\phi}, \hat{\boldsymbol{\theta}}^{i},\boldsymbol{\tau}^{i}_1) \mid \boldsymbol{\tau}^{0:i-1}_0] - \nabla_{\boldsymbol{\phi}} \nabla_{\boldsymbol{\theta}^{i}} J^{\text{In}}(\boldsymbol{\phi}, \boldsymbol{\theta}^{i})\bigg\|\Bigg] \times \\
& \mathbb{E}_{\boldsymbol{\tau}^{0:K-1}_0}\bigg[\Bigg\| \prod_{j=i+1}^{K-1} \mathbb{E}_{\boldsymbol{\tau}^{j}_2} [I+\alpha \nabla^{2}_{\hat{\boldsymbol{\theta}}^{j}}\hat{J}^{\text{In}}(\boldsymbol{\phi}, \hat{\boldsymbol{\theta}}^{j} ,\boldsymbol{\tau}^{j}_{2}) \mid \boldsymbol{\tau}^{0:i-1}_0] \Bigg\| \times \Bigg\|\mathbb{E}_{\boldsymbol{\tau}_3 } [ \nabla_{\hat{\boldsymbol{\theta}}^{K}} \hat{J}^{\text {Out }}(\boldsymbol{\phi}, \hat{\boldsymbol{\theta}}^{K},\tau_{3}) \mid \boldsymbol{\tau}^{0:K-1}_0 ]\Bigg\|\bigg]\\ 
\end{aligned}
\end{equation}

\begin{equation}
\begin{aligned}
\leq& \mu_1\mathbb{E}_{\boldsymbol{\tau}^{0:K-1}_0} \Bigg[\bigg\|   \hat{\boldsymbol{\theta}}^{K}  - \boldsymbol{\theta}^{K}  \bigg\| \Bigg] \\
&+ \alpha \sum_{i=0}^{K-1} c_1 \Bigg[ m_2 \mathbb{E}_{\boldsymbol{\tau}^{0:K-1}_0} \bigg[ \bigg\| \prod_{j=i+1}^{K-1} \mathbb{E}_{\boldsymbol{\tau}^{j}_2} [I+\alpha \nabla^{2}_{\hat{\boldsymbol{\theta}}^{j}}\hat{J}^{\text{In}}(\boldsymbol{\phi}, \hat{\boldsymbol{\theta}}^{j} ,\boldsymbol{\tau}^{j}_{2}) \mid \boldsymbol{\tau}^{0:i-1}_0] - \prod_{j=i+1}^{K-1}\left(I+\alpha \nabla^{2}_{\boldsymbol{\theta}^{j}}J^{\text{In}}\left(\boldsymbol{\phi}, \boldsymbol{\theta}^{j}\right)\right)\bigg\|\bigg] \\
&+ \mu_2\mathbb{E}_{\boldsymbol{\tau}^{0:K-1}_0} \Bigg[\bigg\|   \hat{\boldsymbol{\theta}}^{K}  - \boldsymbol{\theta}^{K}  \bigg\| \Bigg] \times\bigg\|\mathbb{E}_{\boldsymbol{\tau}^{0:K-1}_0} \bigg[  \prod_{j=i+1}^{K-1} \mathbb{E}_{\boldsymbol{\tau}^{j}_2} [I+\alpha \nabla^{2}_{\hat{\boldsymbol{\theta}}^{j}}\hat{J}^{\text{In}}(\boldsymbol{\phi}, \hat{\boldsymbol{\theta}}^{j} ,\boldsymbol{\tau}^{j}_{2}) \mid \boldsymbol{\tau}^{0:i-1}_0]\bigg] \bigg\|\Bigg]  + \\
&\mathbb{E}_{\boldsymbol{\tau}^{0:K-1}_0} \Bigg[ \bigg\|\mathbb{E}_{\boldsymbol{\tau}^{i}_1} [\nabla_{\boldsymbol{\phi}} \nabla_{\hat{\boldsymbol{\theta}}^{i}} \hat{J}^{\text{In}}(\boldsymbol{\phi}, \hat{\boldsymbol{\theta}}^{i},\boldsymbol{\tau}^{i}_1) \mid \boldsymbol{\tau}^{0:i-1}_0] - \nabla_{\boldsymbol{\phi}} \nabla_{\boldsymbol{\theta}^{i}} J^{\text{In}}(\boldsymbol{\phi}, \boldsymbol{\theta}^{i})\bigg\|\Bigg] \times \\
&m_2\Bigg\| \mathbb{E}_{\boldsymbol{\tau}^{0:K-1}_0}\bigg[ \prod_{j=i+1}^{K-1} \mathbb{E}_{\boldsymbol{\tau}^{j}_2} [I+\alpha \nabla^{2}_{\hat{\boldsymbol{\theta}}^{j}}\hat{J}^{\text{In}}(\boldsymbol{\phi}, \hat{\boldsymbol{\theta}}^{j} ,\boldsymbol{\tau}^{j}_{2}) \mid \boldsymbol{\tau}^{0:i-1}_0] \bigg] \Bigg\|\\ 
\end{aligned}
\end{equation}

\begin{equation}\label{K-step-bias}
\begin{aligned}
\leq& \mu_1\mathbb{E}_{\boldsymbol{\tau}^{0:K-1}_0} \Bigg[\bigg\|   \hat{\boldsymbol{\theta}}^{K}  - \boldsymbol{\theta}^{K}  \bigg\| \Bigg] \\
&+ \alpha \sum_{i=0}^{K-1} c_1  m_2 \underbrace{\mathbb{E}_{\boldsymbol{\tau}^{0:K-1}_0} \bigg[ \bigg\| \prod_{j=i+1}^{K-1} \mathbb{E}_{\boldsymbol{\tau}^{j}_2} [I+\alpha \nabla^{2}_{\hat{\boldsymbol{\theta}}^{j}}\hat{J}^{\text{In}}(\boldsymbol{\phi}, \hat{\boldsymbol{\theta}}^{j} ,\boldsymbol{\tau}^{j}_{2}) \mid \boldsymbol{\tau}^{0:i-1}_0] - \prod_{j=i+1}^{K-1}\left(I+\alpha \nabla^{2}_{\boldsymbol{\theta}^{j}}J^{\text{In}}\left(\boldsymbol{\phi}, \boldsymbol{\theta}^{j}\right)\right)\bigg\|\bigg]}_{\textbf{Term (i)}} \\
&+ \Bigg[c_1 \mu_2\mathbb{E}_{\boldsymbol{\tau}^{0:K-1}_0} \Bigg[\bigg\|   \hat{\boldsymbol{\theta}}^{K}  - \boldsymbol{\theta}^{K}  \bigg\| \Bigg]  + m_2\underbrace{\mathbb{E}_{\boldsymbol{\tau}^{0:K-1}_0} \Bigg[ \bigg\|\mathbb{E}_{\boldsymbol{\tau}^{i}_1} [\nabla_{\boldsymbol{\phi}} \nabla_{\hat{\boldsymbol{\theta}}^{i}} \hat{J}^{\text{In}}(\boldsymbol{\phi}, \hat{\boldsymbol{\theta}}^{i},\boldsymbol{\tau}^{i}_1) \mid \boldsymbol{\tau}^{0:i-1}_0] - \nabla_{\boldsymbol{\phi}} \nabla_{\boldsymbol{\theta}^{i}} J^{\text{In}}(\boldsymbol{\phi}, \boldsymbol{\theta}^{i})\bigg\|\Bigg]}_{\textbf{Term (ii)}}
\Bigg]\times \\
&\underbrace{\Bigg\| \mathbb{E}_{\boldsymbol{\tau}^{0:K-1}_0}\bigg[ \prod_{j=i+1}^{K-1} \mathbb{E}_{\boldsymbol{\tau}^{j}_2} [I+\alpha \nabla^{2}_{\hat{\boldsymbol{\theta}}^{j}}\hat{J}^{\text{In}}(\boldsymbol{\phi}, \hat{\boldsymbol{\theta}}^{j} ,\boldsymbol{\tau}^{j}_{2}) \mid \boldsymbol{\tau}^{0:i-1}_0] \bigg] \Bigg\|}_{\textbf{Term (iii)}} \\ 
\end{aligned}
\end{equation}
Let $J_{\boldsymbol{\phi}, \boldsymbol{\theta}}$ denote $\nabla_{\boldsymbol{\phi}} \nabla_{\boldsymbol{\theta}} J^{\text {In}}$, $H_{\boldsymbol{\theta}, \boldsymbol{\theta}}$ denote $ \nabla^{2}_{\boldsymbol{\theta}} J^{\text {In}}$ ,$\hat{\Delta}_{J}=\max \|\mathbb{E}[\hat{J}_{\boldsymbol{\phi},\boldsymbol{\theta}}]-J_{\boldsymbol{\phi}, \boldsymbol{\theta}}\|$. $\hat{\Delta}_{H}=\max\|\mathbb{E}[\hat{H}_{\boldsymbol{\theta}, \boldsymbol{\theta}}]  -  H_{\boldsymbol{\theta}, \boldsymbol{\theta}}\|$. $(\hat{\sigma}_{J})^{2}= \frac{\max\mathbb{V} [\hat{J}_{\boldsymbol{\phi}, \boldsymbol{\theta}}]}{|\tau|}$. $(\hat{\sigma}_{H})^{2}= \frac{\max\mathbb{V} [\hat{H}_{\boldsymbol{\theta}, \boldsymbol{\theta}}]}{|\tau|}$. We upper bound terms (i)-(ii) in Eq. \eqref{K-step-bias} respectively, that is, 

\vspace{4pt}
\noindent
\textbf{Term (i).} According to 

\begin{equation}
\begin{aligned}
&\mathbb{E}_{\boldsymbol{\tau}^{0:K-1}_0} \bigg[ \bigg\| \prod_{j=i+1}^{K-1} \mathbb{E}_{\boldsymbol{\tau}^{j}_2} [I+\alpha \nabla^{2}_{\hat{\boldsymbol{\theta}}^{j}}\hat{J}^{\text{In}}(\boldsymbol{\phi}, \hat{\boldsymbol{\theta}}^{j} ,\boldsymbol{\tau}^{j}_{2}) \mid \boldsymbol{\tau}^{0:i-1}_0] - \prod_{j=i+1}^{K-1}\left(I+\alpha \nabla^{2}_{\boldsymbol{\theta}^{j}}J^{\text{In}}\left(\boldsymbol{\phi}, \boldsymbol{\theta}^{j}\right)\right)\bigg\|\bigg]\\
\leq & \mathbb{E}_{\boldsymbol{\tau}^{0:K-1}_0} \bigg[ \prod_{j=i+1}^{K-2} (I+\alpha  c_2 +\alpha \hat{\Delta}_{H} )\bigg] \bigg(\alpha\hat{\Delta}_{H}
+ \alpha \rho_2 \left((1+\alpha c_2)^{K-1} -1\right) \frac{\hat{\sigma}_{\text{In}}}{c_2\sqrt{|\boldsymbol{\tau}|}}
\bigg)\\
&+ (1+ \alpha c_2) \times\\
&\mathbb{E}_{\boldsymbol{\tau}^{0:K-1}_0} \bigg[ \bigg\| \prod_{j=i+1}^{K-2} \mathbb{E}_{\boldsymbol{\tau}^{j}_2} [I+\alpha \nabla_{\hat{\boldsymbol{\theta}}^{j}}^{2} \hat{J}^{In}(\boldsymbol{\phi}, \hat{\boldsymbol{\theta}}^{j} ,\boldsymbol{\tau}^{j}_{2}) \mid \boldsymbol{\tau}^{0:i-1}_0] - \prod_{j=i+1}^{K-2}\left(I+\alpha \nabla_{\boldsymbol{\theta}^{j}}^{2} J^{In}\left(\boldsymbol{\phi}, \boldsymbol{\theta}^{j}\right)\right)\bigg\|\bigg]\\
\leq& \alpha \bigg[ (1+ \alpha c_2 + \alpha \hat{\Delta}_{H})^{K-i-1} - (1+ \alpha c_2 )^{K-i-1} \bigg] \\
& + \frac{\rho_2}{c_2} \left((1+\alpha c_2+\alpha\hat{\Delta}_{H})^{K-i-1} -1\right)(1+\alpha c_2)^{K-1} \frac{\hat{\sigma}_{\text{In}}}{c_2\sqrt{|\boldsymbol{\tau}|}}
\end{aligned}
\end{equation}

\vspace{4pt}
\noindent
\textbf{Term (ii).}

\begin{equation}
\begin{aligned}
&\mathbb{E}_{\boldsymbol{\tau}^{0:K-1}_0} \Bigg[ \bigg\|\mathbb{E}_{\boldsymbol{\tau}^{i}_1} [\nabla_{\boldsymbol{\phi}} \nabla_{\hat{\boldsymbol{\theta}}^{i}} \hat{J}^{\text{In}}(\boldsymbol{\phi}, \hat{\boldsymbol{\theta}}^{i},\boldsymbol{\tau}^{i}_1) \mid \boldsymbol{\tau}^{0:i-1}_0] - \nabla_{\boldsymbol{\phi}} \nabla_{\boldsymbol{\theta}^{i}} J^{\text{In}}(\boldsymbol{\phi}, \boldsymbol{\theta}^{i})\bigg\|\Bigg]\\
\leq& \mathbb{E}_{\boldsymbol{\tau}^{0:K-1}_0} \Bigg[ \bigg\|\mathbb{E}_{\boldsymbol{\tau}^{i}_1} [\nabla_{\boldsymbol{\phi}} \nabla_{\hat{\boldsymbol{\theta}}^{i}} \hat{J}^{\text{In}}(\boldsymbol{\phi}, \hat{\boldsymbol{\theta}}^{i},\boldsymbol{\tau}^{i}_1)\mid \boldsymbol{\tau}^{0:i-1}_0]  -  \nabla_{\boldsymbol{\phi}} \nabla_{\hat{\boldsymbol{\theta}}^{i}} J^{\text{In}}(\boldsymbol{\phi}, \hat{\boldsymbol{\theta}}^{i})\bigg\|\Bigg] +\\
& \mathbb{E}_{\boldsymbol{\tau}^{0:K-1}_0} \Bigg[ \bigg\|
\nabla_{\boldsymbol{\phi}} \nabla_{\hat{\boldsymbol{\theta}}^{i}} J^{\text{In}}(\boldsymbol{\phi}, \hat{\boldsymbol{\theta}}^{i}) - \nabla_{\boldsymbol{\phi}} \nabla_{\boldsymbol{\theta}^{i}} J^{\text{In}}(\boldsymbol{\phi}, \boldsymbol{\theta}^{i})\mid \boldsymbol{\tau}^{0:i-1}_0]\bigg\|\Bigg]\\
\end{aligned}
\end{equation}

%Let $\hat{\Delta}_{J}=\left\|\mathbb{E}_{\boldsymbol{\tau}^{i}_1} [\nabla_{\boldsymbol{\phi}} \nabla_{\boldsymbol{\theta}} \hat{J}^{\text{In}}(\boldsymbol{\phi}, \hat{\boldsymbol{\theta}}^{j},\boldsymbol{\tau}^{i}_1)]  -  \nabla_{\boldsymbol{\phi}} \nabla_{\boldsymbol{\theta}} J^{\text{In}}(\boldsymbol{\phi}, \hat{\boldsymbol{\theta}}^{j})\right\|$,
%$\hat{\Delta}_{J}=\max_{j} \hat{\Delta}_{J}$,$j\in \{0,\ldots,K-1\}$

\begin{equation}
\begin{aligned}
&\mathbb{E}_{\boldsymbol{\tau}^{0:K-1}_0} \Bigg[ \bigg\|\mathbb{E}_{\boldsymbol{\tau}^{i}_1} [\nabla_{\boldsymbol{\phi}} \nabla_{\hat{\boldsymbol{\theta}}^{i}} \hat{J}^{\text{In}}(\boldsymbol{\phi}, \hat{\boldsymbol{\theta}}^{i},\boldsymbol{\tau}^{i}_1) \mid \boldsymbol{\tau}^{0:i-1}_0] - \nabla_{\boldsymbol{\phi}} \nabla_{\boldsymbol{\theta}^{i}} J^{\text{In}}(\boldsymbol{\phi}, \boldsymbol{\theta}^{i})\bigg\|\Bigg]\\
\leq& \mathbb{E}_{\boldsymbol{\tau}^{0:K-1}_0} \bigg[\hat{\Delta}_{J} \bigg] + \lambda_2 \mathbb{E}_{\boldsymbol{\tau}^{0:K-1}_0} \Bigg[\bigg\|   \hat{\boldsymbol{\theta}}^{i}  - \boldsymbol{\theta}^{i}  \bigg\| \Bigg]\\
\leq& \hat{\Delta}_{J}  + \lambda_2 \mathbb{E}_{\boldsymbol{\tau}^{0:K-1}_0} \Bigg[\bigg\|   \hat{\boldsymbol{\theta}}^{i}  - \boldsymbol{\theta}^{i}  \bigg\| \Bigg]\\
\end{aligned}
\end{equation}

\vspace{4pt}
\noindent
\textbf{Term (iii).}

\begin{equation}
\begin{aligned}
&\Bigg\| \mathbb{E}_{\boldsymbol{\tau}^{0:K-1}_0}\bigg[ \prod_{j=i+1}^{K-1} \mathbb{E}_{\boldsymbol{\tau}^{j}_2} [I+\alpha \nabla^{2}_{\hat{\boldsymbol{\theta}}^{j}}\hat{J}^{\text{In}}(\boldsymbol{\phi}, \hat{\boldsymbol{\theta}}^{j} ,\boldsymbol{\tau}^{j}_{2}) \mid \boldsymbol{\tau}^{0:i-1}_0] \bigg] \Bigg\|\\
\leq & \mathbb{E}_{\boldsymbol{\tau}^{0:K-1}_0} \bigg[ \prod_{j=i+1}^{K-1} \bigg(I+\alpha\bigg\|  \mathbb{E}_{\boldsymbol{\tau}^{j}_2}  [\nabla^{2}_{\hat{\boldsymbol{\theta}}^{j}}\hat{J}^{\text{In}}(\boldsymbol{\phi}, \hat{\boldsymbol{\theta}}^{j} ,\boldsymbol{\tau}^{j}_{2})\mid \boldsymbol{\tau}^{0:i-1}_0]  -  \nabla^{2}_{\hat{\boldsymbol{\theta}}^{j}}\hat{J}^{\text{In}}(\boldsymbol{\phi}, \hat{\boldsymbol{\theta}}^{j} ) \bigg\| +\alpha \bigg\|   \nabla^{2}_{\hat{\boldsymbol{\theta}}^{j}}\hat{J}^{\text{In}}(\boldsymbol{\phi}, \hat{\boldsymbol{\theta}}^{j} )  \bigg\| \bigg)\bigg] \\
\end{aligned}
\end{equation}

%Let $\hat{\Delta}_{H}=\left\|\mathbb{E}_{\boldsymbol{\tau}^{i}_2} [\nabla^{2}_{\hat{\boldsymbol{\theta}}^{j}}\hat{J}^{\text{In}}(\boldsymbol{\phi}, \hat{\boldsymbol{\theta}}^{j},\boldsymbol{\tau}^{i}_2)]  -  \nabla_{\boldsymbol{\phi}} \nabla_{\boldsymbol{\theta}} J^{\text{In}}(\boldsymbol{\phi}, \hat{\boldsymbol{\theta}}^{j})\right\|$,
%$\hat{\Delta}_{H}=\max_{j} \hat{\Delta}_{H}$,$j\in \{0,\ldots,K-1\}$

\begin{equation}
\begin{aligned}
&\Bigg\| \mathbb{E}_{\boldsymbol{\tau}^{0:K-1}_0}\bigg[ \prod_{j=i+1}^{K-1} \mathbb{E}_{\boldsymbol{\tau}^{j}_2} [I+\alpha \nabla^{2}_{\hat{\boldsymbol{\theta}}^{j}}\hat{J}^{\text{In}}(\boldsymbol{\phi}, \hat{\boldsymbol{\theta}}^{j} ,\boldsymbol{\tau}^{j}_{2}) \mid \boldsymbol{\tau}^{0:i-1}_0] \bigg] \Bigg\|\\
\leq& \mathbb{E}_{\boldsymbol{\tau}^{0:K-1}_0} \bigg[ \prod_{j=i+1}^{K-1} (1+ \alpha c_2 + \alpha \hat{\Delta}_{H})
\mid \boldsymbol{\tau}^{0:i-1}_0]\bigg]\\
\leq&  (1+ \alpha c_2 + \alpha \hat{\Delta}_{H})^{K-i-1}\\
\end{aligned}
\end{equation}

Then combine terms (i)-(iii)  together, that is

\begin{equation}
\begin{aligned}
&\left\|\mathbb{E}_{\boldsymbol{\tau}^{0:K-1}_0, \boldsymbol{\tau}^{0:K-1}_1, \boldsymbol{\tau}^{0:K-1}_2, \boldsymbol{\tau}_3}[\nabla_{\boldsymbol{\phi}} \hat{J}^{\text{K}}(\boldsymbol{\phi})] - \nabla_{\boldsymbol{\phi}} J^{\text{K}}(\boldsymbol{\phi})\right\|\\
\leq& \mu_1\left((1+\alpha c_2)^{K} -1\right) \frac{\hat{\sigma}_{\text{In}}}{c_2\sqrt{|\boldsymbol{\tau}|}}
 \\
&+ \alpha \sum_{i=0}^{K-1} \alpha c_1  m_2  \bigg[ (1+ \alpha c_2 + \alpha \hat{\Delta}_{H})^{K-i-1} - (1+ \alpha c_2 )^{K-i-1} \bigg] \\
& + c_1  m_2 \frac{\rho_2}{c_2} \left((1+\alpha c_2+\hat{\Delta}_{H})^{K-i-1} -1\right)(1+\alpha c_2)^{K-1} \frac{\hat{\sigma}_{\text{In}}}{c_2\sqrt{|\boldsymbol{\tau}|}} \\
&+ (1+ \alpha c_2 + \hat{\Delta}_{H})^{K-i-1} \Bigg[c_1 \mu_2\left((1+\alpha c_2)^{K} -1\right) \frac{\hat{\sigma}_{\text{In}}}{c_2\sqrt{|\boldsymbol{\tau}|}}  + m_2\hat{\Delta}_{J}  + m_2\lambda_2 \left((1+\alpha c_2)^{i} -1\right) \frac{\hat{\sigma}_{\text{In}}}{c_2\sqrt{|\boldsymbol{\tau}|}}
\Bigg]\\
\leq & 
    (\mu_1 + \alpha (c_1  m_2 \frac{\rho_2}{c_2} +c_1 \mu_2 +m_2 \lambda_2))  
    \left(((1+\alpha c_2 )^{K} -1\right)
    \left((1+\alpha c_2 + \alpha \hat{\Delta}_{H})^{K-1} -1\right)
    \frac{\hat{\sigma}_{In}}{c_2\sqrt{|\boldsymbol{\tau}|}}
    \\
    &+   
    (\alpha m_2) \left((1+\alpha c_2 + \alpha \hat{\Delta}_{H})^{K-1} -1\right) \hat{\Delta}_{J}\\
    &+(\alpha^{2} c_1 m_2)  \left((1+\alpha c_2 
    + \alpha \hat{\Delta}_{H})^{K-1} -
    (1+\alpha c_2 )^{K-1} \right)
\end{aligned}
\end{equation}

\begin{equation}
\begin{aligned}
&\left\|\mathbb{E}_{\boldsymbol{\tau}^{0:K-1}_0, \boldsymbol{\tau}^{0:K-1}_1, \boldsymbol{\tau}^{0:K-1}_2, \boldsymbol{\tau}_3}[\nabla_{\boldsymbol{\phi}} \hat{J}^{\text{K}}(\boldsymbol{\phi})] - \nabla_{\boldsymbol{\phi}} J^{\text{K}}(\boldsymbol{\phi})\right\| \\
\leq & \mathcal{O}\bigg(
(1+\alpha c_2 + \alpha \hat{\Delta}_{H})^{K-1}
\left(\mathbb{E} [\|\hat{\boldsymbol{\theta}}^{K}  - \boldsymbol{\theta}^{K}\|]+
\hat{\Delta}_{J} +(K-1)
\right) \bigg)
\end{aligned}
\end{equation}

which concludes the proof of upper bound of meta-gradient bias.

According to Lemma~\ref{lm:totvariance},

\begin{equation}\label{total_variance_MS}
\mathbb{V} \left[\nabla_{\boldsymbol{\phi}}\hat{J}^{\text{K}}(\boldsymbol{\phi})\right] =
\underbrace{\mathbb{V}\left[\mathbb{E}_{\boldsymbol{\tau}^{0:K-1}_1, \boldsymbol{\tau}^{1:K-1}_2, \boldsymbol{\tau}_3}\left[\nabla_{\boldsymbol{\phi}}\hat{J}^{\text{K}}(\boldsymbol{\phi}) \mid \boldsymbol{\tau}^{0:K-1}_0 \right]\right]}_{\textbf{Term (i)}} +
\underbrace{
\mathbb{E}_{\boldsymbol{\tau}^{0:K-1}_0}\left[\mathbb{V}\left[\nabla_{\boldsymbol{\phi}}\hat{J}^{\text{K}}(\boldsymbol{\phi}) \mid \boldsymbol{\tau}^{0:K-1}_0\right]\right] }_{\textbf{Term (ii)}} 
\end{equation}

We upper bound terms (i)-(ii) in Eq.~\eqref{total_variance_MS} respectively, that is, 

\vspace{4pt}
\noindent
\textbf{Term (i).} 

\begin{equation}
\begin{aligned}
&\mathbb{V}\left[\mathbb{E}_{\boldsymbol{\tau}^{0:K-1}_1, \boldsymbol{\tau}^{1:K-1}_2, \boldsymbol{\tau}_3}\left[\nabla_{\boldsymbol{\phi}}\hat{J}^{\text{K}}(\boldsymbol{\phi}) \mid \boldsymbol{\tau}^{0:K-1}_0 \right]\right]\\
=& \mathbb{E}_{\boldsymbol{\tau}^{0:K-1}_0}\left[\left\|
\mathbb{E}_{\boldsymbol{\tau}^{0:K-1}_1, \boldsymbol{\tau}^{1:K-1}_2, \boldsymbol{\tau}_3}\left[\nabla_{\boldsymbol{\phi}}\hat{J}^{\text{K}}(\boldsymbol{\phi})\mid \boldsymbol{\tau}^{0:K-1}_0 \right]-
\mathbb{E}_{ 
\boldsymbol{\tau}^{0:K-1}_0,\boldsymbol{\tau}^{0:K-1}_1, \boldsymbol{\tau}^{1:K-1}_2, \boldsymbol{\tau}_3}
\left[\nabla_{\boldsymbol{\phi}}\hat{J}^{\text{K}}(\boldsymbol{\phi})\right]\right\|^{2}
 \right]\\
\leq& \mathbb{E}_{\boldsymbol{\tau}^{0:K-1}_0}\left[\left\|
\mathbb{E}_{\boldsymbol{\tau}^{0:K-1}_1, \boldsymbol{\tau}^{1:K-1}_2, \boldsymbol{\tau}_3}\left[\nabla_{\boldsymbol{\phi}}\hat{J}^{\text{K}}(\boldsymbol{\phi})\mid \boldsymbol{\tau}^{0:K-1}_0 \right]-
\nabla_{\boldsymbol{\phi}} J^{\text{K}}(\boldsymbol{\phi})
\right\|^{2}
 \right]\\
\end{aligned}
\end{equation}

According to proof of upper bound of bias term, together with Lemma~\ref{aux_lemma_3} \ref{lem:varprop_ii}.

\begin{equation}
\begin{aligned}
&\left\|\mathbb{E}_{\boldsymbol{\tau}^{0:K-1}_0, \boldsymbol{\tau}^{0:K-1}_1, \boldsymbol{\tau}^{0:K-1}_2, \boldsymbol{\tau}_3}[\nabla_{\boldsymbol{\phi}} \hat{J}^{\text{K}}(\boldsymbol{\phi})] - \nabla_{\boldsymbol{\phi}} J^{\text{K}}(\boldsymbol{\phi})\right\|\\
\leq& \mu_1\left((1+\alpha c_2)^{K} -1\right) \frac{\hat{\sigma}_{\text{In}}}{c_2\sqrt{|\boldsymbol{\tau}|}}
 \\
&+ \alpha \sum_{i=0}^{K-1} \alpha c_1  m_2  \bigg[ (1+ \alpha c_2 + \hat{\Delta}_{H})^{K-i-1} - (1+ \alpha c_2 )^{K-i-1} \bigg] \\
& + c_1  m_2 \frac{\rho_2}{c_2} \left((1+\alpha c_2+\hat{\Delta}_{H})^{K-i-1} -1\right)(1+\alpha c_2)^{K-1} \frac{\hat{\sigma}_{\text{In}}}{c_2\sqrt{|\boldsymbol{\tau}|}} \\
&+ (1+ \alpha c_2 + \hat{\Delta}_{H})^{K-i-1} \Bigg[c_1 \mu_2\left((1+\alpha c_2)^{K} -1\right) \frac{\hat{\sigma}_{\text{In}}}{c_2\sqrt{|\boldsymbol{\tau}|}}  + m_2\hat{\Delta}_{J}  + m_2\lambda_2 \left((1+\alpha c_2)^{i} -1\right) \frac{\hat{\sigma}_{\text{In}}}{c_2\sqrt{|\boldsymbol{\tau}|}}
\Bigg]\\
\end{aligned}
\end{equation}

\begin{equation}
\begin{aligned}
    &\mathbb{V}\left[\mathbb{E}_{\boldsymbol{\tau}^{0:K-1}_1, \boldsymbol{\tau}^{1:K-1}_2, \boldsymbol{\tau}_3}\left[\nabla_{\boldsymbol{\phi}}\hat{J}^{\text{K}}(\boldsymbol{\phi}) \mid \boldsymbol{\tau}^{0:K-1}_0 \right]\right]\\
    \leq & 
    4^{K}\left(2\mu^{2}_1 + 2K \alpha^{2} (c_1  m_2 \frac{\rho_2}{c_2} +c_1 \mu_2 +m_2 \lambda_2)^{2}\right)  \times\\
    &\bigg(\left((1+\alpha c_2)^{2} +\alpha^{2} c^{2}_2\right)^{K} -1\bigg)
    \bigg(\left((1+\alpha c_2)^{2} + \alpha^{2} (\hat{\Delta}_{H})^{2}\right)^{K-1} -1\bigg)
    \frac{(\hat{\sigma}_{\text{In}})^{2}}{c^{2}_2|\boldsymbol{\tau}|}
    \\
    &+   
    (2K \alpha^{2} m^{2}_2) \bigg(\left((1+\alpha c_2)^{2} + \alpha^{2} (\hat{\Delta}_{H})^{2}\right)^{K-1} -1\bigg) (\hat{\Delta}_{J})^{2}\\
    &+(2K\alpha^{4} c^{2}_1 m^{2}_2) \bigg( \left((1+\alpha c_2)^{2} + \alpha^{2} (\hat{\Delta}_{H})^{2}\right)^{K-1}  - 
    \left((1+\alpha c_2)^{2} +\alpha^{2} c^{2}_2\right)^{K-1}\bigg)
\end{aligned}
\end{equation}

\vspace{4pt}
\noindent
\textbf{Term (ii).}

\begin{equation}
\begin{aligned}
&\mathbb{E}_{\boldsymbol{\tau}^{0:K-1}_0}\left[\mathbb{V}\left[\nabla_{\boldsymbol{\phi}}\hat{J}^{\text{K}}(\boldsymbol{\phi}) \mid \boldsymbol{\tau}^{0:K-1}_0\right]\right]
\\
=& \mathbb{E}_{\boldsymbol{\tau}^{0:K-1}_0}\left[\mathbb{E}_{\boldsymbol{\tau}^{0:K-1}_1, \boldsymbol{\tau}^{1:K-1}_2, \boldsymbol{\tau}_3}
\left[
\left\|\nabla_{\boldsymbol{\phi}}\hat{J}^{\text{K}}(\boldsymbol{\phi})-
\mathbb{E}_{\boldsymbol{\tau}^{0:K-1}_1, \boldsymbol{\tau}^{1:K-1}_2, \boldsymbol{\tau}_3}
\left[\nabla_{\boldsymbol{\phi}}\hat{J}^{\text{K}}(\boldsymbol{\phi})\right]\right\|^{2}\mid \boldsymbol{\tau}^{0:K-1}_0
\right] \right]
\\
\leq&\mathbb{E}_{\boldsymbol{\tau}^{0:K-1}_0} \Bigg[2\mathbb{E}_{ \boldsymbol{\tau}^{0:K-1}_1, \boldsymbol{\tau}^{1:K-1}_2, \boldsymbol{\tau}_3}\bigg[\bigg\|\nabla_{\boldsymbol{\phi}} \hat{J}^{\text {Out }}(\boldsymbol{\phi}, \hat{\boldsymbol{\theta}}^{K}, \boldsymbol{\tau}_3)
-\mathbb{E}_{\boldsymbol{\tau}_3 } [\nabla_{\boldsymbol{\phi}} \hat{J}^{\text {Out }}(\boldsymbol{\phi}, \hat{\boldsymbol{\theta}}^{K}, \boldsymbol{\tau}_3) ]\bigg\|^{2}\bigg]+\\
&2 \alpha^{2}\mathbb{E}_{ \boldsymbol{\tau}^{0:K-1}_1, \boldsymbol{\tau}^{1:K-1}_2, \boldsymbol{\tau}_3}\bigg[\bigg\|\sum_{i=0}^{K-1} \nabla_{\boldsymbol{\phi}} \nabla_{\hat{\boldsymbol{\theta}}^{i}} \hat{J}^{\text{In}}(\boldsymbol{\phi}, \hat{\boldsymbol{\theta}}^{i},\boldsymbol{\tau}^{i}_1) \prod_{j=i+1}^{K-1}\left(I+\alpha \nabla^{2}_{\hat{\boldsymbol{\theta}}^{j}}\hat{J}^{\text{In}}(\boldsymbol{\phi}, \hat{\boldsymbol{\theta}}^{j},\boldsymbol{\tau}^{j}_{2})\right)  
\nabla_{\hat{\boldsymbol{\theta}}^{K}} \hat{J}^{\text {Out }}(\boldsymbol{\phi}, \hat{\boldsymbol{\theta}}^{K},\tau_{3})\\
&-\sum_{i=0}^{K-1}\mathbb{E}_{\boldsymbol{\tau}^{i}_1} [\nabla_{\boldsymbol{\phi}} \nabla_{\hat{\boldsymbol{\theta}}^{i}} \hat{J}^{\text{In}}(\boldsymbol{\phi}, \hat{\boldsymbol{\theta}}^{i},\boldsymbol{\tau}^{i}_1) ] \prod_{j=i+1}^{K-1} \mathbb{E}_{\boldsymbol{\tau}^{j}_2} [I+\alpha \nabla^{2}_{\hat{\boldsymbol{\theta}}^{j}}\hat{J}^{\text{In}}(\boldsymbol{\phi}, \hat{\boldsymbol{\theta}}^{j} ,\boldsymbol{\tau}^{j}_{2}) ]\mathbb{E}_{\boldsymbol{\tau}_3 } [ \nabla_{\hat{\boldsymbol{\theta}}^{K}} \hat{J}^{\text {Out }}(\boldsymbol{\phi}, \hat{\boldsymbol{\theta}}^{K},\tau_{3})]\bigg\|^{2}\bigg]
\mid \boldsymbol{\tau}^{0:K-1}_0\Bigg]\\
\leq&\mathbb{E}_{\boldsymbol{\tau}^{0:K-1}_0} \Bigg[2(\sigma_1)^{2}+2K \alpha^{2}  \sum_{i=0}^{K-1}\bigg\| \nabla_{\boldsymbol{\phi}} \nabla_{\hat{\boldsymbol{\theta}}^{i}} \hat{J}^{\text{In}}
(\boldsymbol{\phi}, \hat{\boldsymbol{\theta}}^{i},\boldsymbol{\tau}^{i}_1)\bigg\|^{2} \times\\ &\mathbb{E}_{ \boldsymbol{\tau}^{0:K-1}_1, \boldsymbol{\tau}^{1:K-1}_2, \boldsymbol{\tau}_3}\bigg[\bigg\|  \prod_{j=i+1}^{K-1}\left(I-
\alpha \nabla^{2}_{\hat{\boldsymbol{\theta}}^{j}}\hat{J}^{\text{In}}(\boldsymbol{\phi}, \hat{\boldsymbol{\theta}}^{j},\boldsymbol{\tau}^{j}_{2})\right)  \nabla_{\hat{\boldsymbol{\theta}}^{K}} \hat{J}^{\text {Out }}(\boldsymbol{\phi}, \hat{\boldsymbol{\theta}}^{K},\tau_{3})-\\
&\prod_{j=i+1}^{K-1} \mathbb{E}_{\boldsymbol{\tau}^{j}_2} [I+\alpha \nabla^{2}_{\hat{\boldsymbol{\theta}}^{j}}\hat{J}^{\text{In}}(\boldsymbol{\phi}, \hat{\boldsymbol{\theta}}^{j} ,\boldsymbol{\tau}^{j}_{2}) ]\mathbb{E}_{\boldsymbol{\tau}_3 } [ \nabla_{\hat{\boldsymbol{\theta}}^{K}} \hat{J}^{\text {Out }}(\boldsymbol{\phi}, \hat{\boldsymbol{\theta}}^{K},\tau_{3})]\bigg\|^{2}\bigg]
+\\
&\bigg\|\prod_{j=i+1}^{K-1} \mathbb{E}_{\boldsymbol{\tau}^{j}_2} [I+\alpha \nabla^{2}_{\hat{\boldsymbol{\theta}}^{j}}\hat{J}^{\text{In}}(\boldsymbol{\phi}, \hat{\boldsymbol{\theta}}^{j} ,\boldsymbol{\tau}^{j}_{2}) ]\mathbb{E}_{\boldsymbol{\tau}_3 } [ \nabla_{\hat{\boldsymbol{\theta}}^{K}} \hat{J}^{\text {Out }}(\boldsymbol{\phi}, \hat{\boldsymbol{\theta}}^{K},\tau_{3})] \bigg\|^{2}\\
&\mathbb{E}_{ \boldsymbol{\tau}^{0:K-1}_1, \boldsymbol{\tau}^{1:K-1}_2, \boldsymbol{\tau}_3}\bigg[\bigg\|  \nabla_{\boldsymbol{\phi}} \nabla_{\hat{\boldsymbol{\theta}}^{i}} \hat{J}^{\text{In}}(\boldsymbol{\phi}, \hat{\boldsymbol{\theta}}^{i},\boldsymbol{\tau}^{i}_1)-\mathbb{E}_{\boldsymbol{\tau}^{i}_1} [\nabla_{\boldsymbol{\phi}} \nabla_{\hat{\boldsymbol{\theta}}^{i}} \hat{J}^{\text{In}}(\boldsymbol{\phi}, \hat{\boldsymbol{\theta}}^{i},\boldsymbol{\tau}^{i}_1) ] \bigg\|^{2}\bigg]\mid \boldsymbol{\tau}^{0:K-1}_0\Bigg]\\
\end{aligned}
\end{equation}

\begin{equation}
\begin{aligned}
\leq&\mathbb{E}_{\boldsymbol{\tau}^{0:K-1}_0} \Bigg[2(\sigma_1)^{2}+ 2K \alpha^{2}  \sum_{i=0}^{K-1} \Bigg[\bigg\| \nabla_{\hat{\boldsymbol{\theta}}^{K}} \hat{J}^{\text {Out }}(\boldsymbol{\phi}, \hat{\boldsymbol{\theta}}^{K},\tau_{3})\bigg\|^{2} \bigg\| \nabla_{\boldsymbol{\phi}} \nabla_{\hat{\boldsymbol{\theta}}^{i}} \hat{J}^{\text{In}}(\boldsymbol{\phi}, \hat{\boldsymbol{\theta}}^{i},\boldsymbol{\tau}^{i}_1)\bigg\|^{2}\times\\
&\mathbb{E}_{ \boldsymbol{\tau}^{0:K-1}_1, \boldsymbol{\tau}^{1:K-1}_2, \boldsymbol{\tau}_3}\bigg[\bigg\|  \prod_{j=i+1}^{K-1}\left(I+\alpha \nabla^{2}_{\hat{\boldsymbol{\theta}}^{j}}\hat{J}^{\text{In}}(\boldsymbol{\phi}, \hat{\boldsymbol{\theta}}^{j},\boldsymbol{\tau}^{j}_{2})\right)  -\prod_{j=i+1}^{K-1} \mathbb{E}_{\boldsymbol{\tau}^{j}_2} [I+\alpha \nabla^{2}_{\hat{\boldsymbol{\theta}}^{j}}\hat{J}^{\text{In}}(\boldsymbol{\phi}, \hat{\boldsymbol{\theta}}^{j} ,\boldsymbol{\tau}^{j}_{2}) ]\bigg\|^{2}\bigg]
+\\
&(\sigma_2)^{2}\bigg\|\prod_{j=i+1}^{K-1} \mathbb{E}_{\boldsymbol{\tau}^{j}_2} [I+\alpha \nabla^{2}_{\hat{\boldsymbol{\theta}}^{j}}\hat{J}^{\text{In}}(\boldsymbol{\phi}, \hat{\boldsymbol{\theta}}^{j} ,\boldsymbol{\tau}^{j}_{2}) ] \bigg\|^{2} \bigg\| 
\nabla_{\boldsymbol{\phi}} \nabla_{\hat{\boldsymbol{\theta}}^{i}} \hat{J}^{\text{In}}
(\boldsymbol{\phi}, \hat{\boldsymbol{\theta}}^{i},\boldsymbol{\tau}^{i}_1)\bigg\|^{2}\Bigg]
+\\
&m^{2}_2\bigg\|\prod_{j=i+1}^{K-1} \mathbb{E}_{\boldsymbol{\tau}^{j}_2} [I+\alpha \nabla^{2}_{\hat{\boldsymbol{\theta}}^{j}}\hat{J}^{\text{In}}(\boldsymbol{\phi}, \hat{\boldsymbol{\theta}}^{j} ,\boldsymbol{\tau}^{j}_{2}) ]\bigg\|^{2} \times\\
&\mathbb{E}_{ \boldsymbol{\tau}^{0:K-1}_1, \boldsymbol{\tau}^{1:K-1}_2, \boldsymbol{\tau}_3}\bigg[\bigg\|
\nabla_{\boldsymbol{\phi}} \nabla_{\hat{\boldsymbol{\theta}}^{i}} \hat{J}^{\text{In}}(\boldsymbol{\phi}, \hat{\boldsymbol{\theta}}^{i},\boldsymbol{\tau}^{i}_1)-\mathbb{E}_{\boldsymbol{\tau}^{i}_1} [\nabla_{\boldsymbol{\phi}} \nabla_{\hat{\boldsymbol{\theta}}^{i}} \hat{J}^{\text{In}}(\boldsymbol{\phi}, \hat{\boldsymbol{\theta}}^{i},\boldsymbol{\tau}^{i}_1) ] \bigg\|^{2}\bigg]
\mid \boldsymbol{\tau}^{0:K-1}_0\Bigg]\\
\end{aligned}
\end{equation}
\begin{equation}
\begin{aligned}
\leq&\mathbb{E}_{\boldsymbol{\tau}^{0:K-1}_0} \Bigg[2(\sigma_1)^{2} + 
2K \alpha^{2}  \sum_{i=0}^{K-1} \Bigg[ \bigg(2m^{2}_2 + 2(\sigma_2)^{2}
\bigg)\times \bigg\| \nabla_{\boldsymbol{\phi}} \nabla_{\hat{\boldsymbol{\theta}}^{i}} \hat{J}^{\text{In}}(\boldsymbol{\phi}, \hat{\boldsymbol{\theta}}^{i},\boldsymbol{\tau}^{i}_1)\bigg\|^{2}\times\\
&\underbrace{\mathbb{E}_{ \boldsymbol{\tau}^{0:K-1}_1, \boldsymbol{\tau}^{1:K-1}_2, \boldsymbol{\tau}_3}\bigg[\bigg\|  \prod_{j=i+1}^{K-1}\left(I+\alpha \nabla^{2}_{\hat{\boldsymbol{\theta}}^{j}}\hat{J}^{\text{In}}(\boldsymbol{\phi}, \hat{\boldsymbol{\theta}}^{j},\boldsymbol{\tau}^{j}_{2})\right)  -\prod_{j=i+1}^{K-1} \mathbb{E}_{\boldsymbol{\tau}^{j}_2} [I+\alpha\nabla^{2}_{\hat{\boldsymbol{\theta}}^{j}}\hat{J}^{\text{In}} (\boldsymbol{\phi}, \hat{\boldsymbol{\theta}}^{j} ,\boldsymbol{\tau}^{j}_{2}) ]\bigg\|^{2}\bigg]}_{\text{Part (I)}}
+\\
&(\sigma_2)^{2}\bigg\|\prod_{j=i+1}^{K-1} \mathbb{E}_{\boldsymbol{\tau}^{j}_2} [I+\alpha \nabla^{2}_{\hat{\boldsymbol{\theta}}^{j}}\hat{J}^{\text{In}}(\boldsymbol{\phi}, \hat{\boldsymbol{\theta}}^{j} ,\boldsymbol{\tau}^{j}_{2}) ] \bigg\|^{2} \times
\underbrace{\bigg\| \nabla_{\boldsymbol{\phi}} \nabla_{\hat{\boldsymbol{\theta}}^{i}} \hat{J}^{\text{In}}(\boldsymbol{\phi}, \hat{\boldsymbol{\theta}}^{i},\boldsymbol{\tau}^{i}_1)\bigg\|^{2}}_{\text{Part (II)}}
+\\
&m^{2}_2\underbrace{\bigg\|\prod_{j=i+1}^{K-1} \mathbb{E}_{\boldsymbol{\tau}^{j}_2} [I+\alpha \nabla^{2}_{\hat{\boldsymbol{\theta}}^{j}}\hat{J}^{\text{In}}(\boldsymbol{\phi}, \hat{\boldsymbol{\theta}}^{j} ,\boldsymbol{\tau}^{j}_{2}) ]\bigg\|^{2}}_{\text{Part (III)}} \times\\
&\mathbb{E}_{ \boldsymbol{\tau}^{0:K-1}_1, \boldsymbol{\tau}^{1:K-1}_2, \boldsymbol{\tau}_3}\bigg[\bigg\|  \nabla_{\boldsymbol{\phi}} \nabla_{\hat{\boldsymbol{\theta}}^{i}} \hat{J}^{\text{In}}(\boldsymbol{\phi}, \hat{\boldsymbol{\theta}}^{i},\boldsymbol{\tau}^{i}_1)-\mathbb{E}_{\boldsymbol{\tau}^{i}_1} [\nabla_{\boldsymbol{\phi}} \nabla_{\hat{\boldsymbol{\theta}}^{i}} \hat{J}^{\text{In}}(\boldsymbol{\phi}, \hat{\boldsymbol{\theta}}^{i},\boldsymbol{\tau}^{i}_1) ] \bigg\|^{2}\bigg]
\mid \boldsymbol{\tau}^{0:K-1}_0\Bigg]\\
\end{aligned}
\end{equation}

\vspace{4pt}
\noindent
\textbf{Part (I)} According to

\begin{equation}
\begin{aligned}
&\mathbb{E}_{ \boldsymbol{\tau}^{0:K-1}_1, \boldsymbol{\tau}^{1:K-1}_2, \boldsymbol{\tau}_3}\bigg[\bigg\|  \prod_{j=i+1}^{K-1}\left(I+\alpha \nabla^{2}_{\hat{\boldsymbol{\theta}}^{j}}\hat{J}^{\text{In}}(\boldsymbol{\phi}, \hat{\boldsymbol{\theta}}^{j},\boldsymbol{\tau}^{j}_{2})\right)  -\prod_{j=i+1}^{K-1} \mathbb{E}_{\boldsymbol{\tau}^{j}_2} [I+\alpha\nabla^{2}_{\hat{\boldsymbol{\theta}}^{j}}\hat{J}^{\text{In}} (\boldsymbol{\phi}, \hat{\boldsymbol{\theta}}^{j} ,\boldsymbol{\tau}^{j}_{2}) ]\bigg\|^{2}\bigg]\\
\leq&2 \prod_{j=i+1}^{K-2} \Bigg(3(1+\alpha c_2)^{2}
+3\alpha^{2} (\hat{\Delta}_{H})^{2}
+3\alpha^{2} (\hat{\sigma}_{H})^{2}
\Bigg)\times \alpha^{2}(\hat{\sigma}_{H})^{2} +
\\
&4\Bigg((1+\alpha c_2)^{2}
+\alpha^{2} (\hat{\Delta}_{H})^{2}
\Bigg)\times\\
& \mathbb{E}_{ \boldsymbol{\tau}^{1:K-1}_2}\bigg[\bigg\| \Bigg(\prod_{j=i+1}^{K-2}\left(I+\alpha \nabla^{2}_{\hat{\boldsymbol{\theta}}^{j}}\hat{J}^{\text{In}}(\boldsymbol{\phi}, \hat{\boldsymbol{\theta}}^{j},\boldsymbol{\tau}^{j}_{2})\right)  -
\prod_{j=i+1}^{K-2} \mathbb{E}_{\boldsymbol{\tau}^{j}_2} [I+\alpha \nabla^{2}_{\boldsymbol{\theta}}\hat{J}^{\text{In}}(\boldsymbol{\phi}, \hat{\boldsymbol{\theta}}^{j} ,\boldsymbol{\tau}^{j}_{2})]\Bigg)\bigg\|^{2}\bigg]\\
\leq& 6^{K-i-1} \bigg[ \left((1+\alpha c_2)^{2}
+\alpha^{2} (\hat{\Delta}_{H})^{2}
+\alpha^{2} (\hat{\sigma}_{H})^{2}\right)^{K-i-1} - 
\left((1+\alpha c_2)^{2}
+\alpha^{2} (\hat{\Delta}_{H})^{2}
\right)^{K-i-1} \bigg] \\
\end{aligned}
\end{equation}

\vspace{4pt}
\noindent
\textbf{Part (II)} According to the supporting Lemma~

\begin{equation}
\begin{aligned}
    & \left\|\nabla_{\boldsymbol{\phi}} \nabla_{\boldsymbol{\theta}^{0}} \hat{J}^{\text{In}}(\boldsymbol{\phi}, \boldsymbol{\theta}^{0}, \tau_1)\right\|^{2} \\
    \leq&\bigg\| \nabla_{\boldsymbol{\phi}} \nabla_{\hat{\boldsymbol{\theta}}^{i}} J^{\text{In}}(\boldsymbol{\phi}, \hat{\boldsymbol{\theta}}^{i})- 
    \nabla_{\boldsymbol{\phi}} \nabla_{\hat{\boldsymbol{\theta}}^{i}} J^{\text{In}}(\boldsymbol{\phi}, \hat{\boldsymbol{\theta}}^{i})
    +\mathbb{E}_{\boldsymbol{\tau}^{i}_1} [\nabla_{\boldsymbol{\phi}} \nabla_{\hat{\boldsymbol{\theta}}^{i}} \hat{J}^{\text{In}}\left(\boldsymbol{\phi}, \hat{\boldsymbol{\theta}}^{i},\boldsymbol{\tau}^{i}_1\right) ]\\
    &-\mathbb{E}_{\boldsymbol{\tau}^{i}_1} [\nabla_{\boldsymbol{\phi}} \nabla_{\hat{\boldsymbol{\theta}}^{i}} \hat{J}^{\text{In}}(\boldsymbol{\phi}, \hat{\boldsymbol{\theta}}^{i},\boldsymbol{\tau}^{i}_1) ]+\nabla_{\boldsymbol{\phi}} \nabla_{\hat{\boldsymbol{\theta}}^{i}} \hat{J}^{\text{In}}(\boldsymbol{\phi}, \hat{\boldsymbol{\theta}}^{i},\boldsymbol{\tau}^{i}_1)\bigg\|^{2}\\
    \leq& 
    3c^{2}_1+
    3\left(
    (\hat{\Delta}_{J})^{2}  +
    (\hat{\sigma}_{J})^{2}
    \right)\\
\end{aligned}
\end{equation}
\begin{equation}
    \left\|\nabla_{\boldsymbol{\phi}} \nabla_{\boldsymbol{\theta}^{0}} \hat{J}^{\text{In}}(\boldsymbol{\phi}, \boldsymbol{\theta}^{0}, \boldsymbol{\tau}^{i}_1)\right\|^{2} 
    \leq 
    3c^{2}_1+
    3\left(
    (\hat{\Delta}_{J})^{2}  +
    (\hat{\sigma}_{J})^{2}
    \right)
\end{equation}

\vspace{4pt}
\noindent
\textbf{Part (III)} 

\begin{equation}
\begin{aligned}
&\bigg\|\prod_{j=i+1}^{K-1} \mathbb{E}_{\boldsymbol{\tau}^{j}_2} [I+\alpha \nabla^{2}_{\hat{\boldsymbol{\theta}}^{j}}\hat{J}^{\text{In}}(\boldsymbol{\phi}, \hat{\boldsymbol{\theta}}^{j} ,\boldsymbol{\tau}^{j}_{2}) ]\bigg\|^{2}\\
\leq&\prod_{j=i+1}^{K-1} \bigg\|  I
-\alpha \nabla^{2}_{\hat{\boldsymbol{\theta}}^{j}}J^{\text{In}}(\boldsymbol{\phi}, \hat{\boldsymbol{\theta}}^{j})
+\alpha \nabla^{2}_{\hat{\boldsymbol{\theta}}^{j}}J^{\text{In}}(\boldsymbol{\phi}, \hat{\boldsymbol{\theta}}^{j})-\alpha\mathbb{E}_{\boldsymbol{\tau}^{j}_2} [ \nabla^{2}_{\hat{\boldsymbol{\theta}}^{j}}\hat{J}^{\text{In}}(\boldsymbol{\phi}, \hat{\boldsymbol{\theta}}^{j} ,\boldsymbol{\tau}^{j}_{2})] \bigg\|^{2}\\
\leq& \left((1+\alpha c_2)^{2}
+\alpha^{2} (\hat{\Delta}_{H})^{2}\right)^{K-i-1}
\end{aligned}
\end{equation}

\begin{equation}
\begin{aligned}
&\mathbb{E}_{\boldsymbol{\tau}^{0:K-1}_0}\left[\mathbb{V}\left[\nabla_{\boldsymbol{\phi}}\hat{J}^{\text{K}}(\boldsymbol{\phi}) \mid \boldsymbol{\tau}^{0:K-1}_0\right]\right]\\
\leq&\mathbb{E}_{\boldsymbol{\tau}^{0:K-1}_0} \Bigg[2(\sigma_1)^{2} + 
2K \alpha^{2}  \sum_{i=0}^{K-1} \Bigg[ \bigg(2m^{2}_2 + 2\sigma_2^{2}
\bigg)\times \bigg(    3c^{2}_1+
    3\left(
    (\hat{\Delta}_{J})^{2}  +
    (\hat{\sigma}_{J})^{2}
    \right)\bigg)\times\\
&6^{K-i-1} \bigg[ \left((1+\alpha c_2)^{2}
+\alpha^{2} (\hat{\Delta}_{H})^{2}
+\alpha^{2} (\hat{\sigma}_{H})^{2}\right)^{K-i-1} - 
\left((1+\alpha c_2)^{2}
+\alpha^{2} (\hat{\Delta}_{H})^{2}
\right)^{K-i-1} \bigg]
+\\
&(\sigma_2)^{2}\left((1+\alpha c_2)^{2}
+\alpha^{2} (\hat{\Delta}_{H})^{2}\right)^{K-i-1} \times
\bigg(    3c^{2}_1+
    3\left(
    (\hat{\Delta}_{J})^{2}  +
    (\hat{\sigma}_{J})^{2}
    \right)\bigg)
+\\
&m^{2}_2\left((1+\alpha c_2)^{2}
+\alpha^{2} (\hat{\Delta}_{H})^{2}\right)^{K-i-1} \times (\hat{\sigma}_{J})^{2}
\mid \boldsymbol{\tau}^{0:K-1}_0\Bigg]\\
\leq& 2(\sigma_1)^{2} \\  
    &+(6K \alpha^{2} \sigma^{2}_2) 
    \bigg(\left((1+\alpha c_2)^{2} + \alpha^{2} (\hat{\Delta}_{H})^{2}\right)^{K-1} -1\bigg)
    \left( c^{2}_1 + (\hat{\Delta}_{J})^{2} + (\hat{\sigma}_{J})^{2} \right)\\
    &+(2K \alpha^{2} m^{2}_2) 
    \bigg(\left((1+\alpha c_2)^{2} + \alpha^{2} (\hat{\Delta}_{H})^{2}\right)^{K-1} -1\bigg)
    (\hat{\sigma}_{J})^{2} \\
    &  +2\alpha^{2} (m^{2}_{1} +3 \sigma_2^{2})     \left( c^{2}_1 + (\hat{\Delta}_{J})^{2} + (\hat{\sigma}_{J})^{2} \right) \\
    &+(6^{K}K \alpha^{2} (12m^{2}_2 + 12\sigma_2^{2}))     \left( c^{2}_1 + (\hat{\Delta}_{J})^{2} + (\hat{\sigma}_{J})^{2} \right) \times\\ 
    &\bigg( \left((1+\alpha c_2)^{2}
+\alpha^{2} (\hat{\Delta}_{H})^{2}
+\alpha^{2} (\hat{\sigma}_{H})^{2}\right)^{K-1}- 
    \left((1+\alpha c_2)^{2} + \alpha^{2} (\hat{\Delta}_{H})^{2}\right)^{K-1}\bigg)
\end{aligned}
\end{equation}

Then combine terms (i)-(ii)  together, that is
\begin{equation}
\begin{aligned}
&\mathbb{V} \left[\nabla_{\boldsymbol{\phi}}\hat{J}^{\text{K}}(\boldsymbol{\phi})\right]\\
\leq &\mathcal{O}\Bigg(
    (V_1+\hat{\Delta}^{2}_{H})^{K-1}  
    \left(\mathbb{E} [\|\hat{\boldsymbol{\theta}}^{K}  - \boldsymbol{\theta}^{K}\|^{2}] +
    (K-1)\right)
    \\
    &+
    \left(
    V_2+(V_1+\hat{\Delta}^{2}_{H} + \hat{\sigma}^{2}_{H})^{K-1}-
    (V_1+\hat{\Delta}^{2}_{H})^{K-1}\right)
    (\hat{\Delta}^{2}_{J} + \hat{\sigma}^{2}_{J})
    \Bigg)
\end{aligned}
\end{equation}

which concludes the proof of Theorem~\ref{theorem_1}.
\end{proof}

%% file: core/9_aux_lemma.tex
\section{Supporting Lemmas}\label{supporting_lemmas}
In this section, we present the supporting lemmas.
\begin{definition}
Let $X$ be a random vector in $\mathbb{R}^{d}$. Then the norm of $X$ is
\begin{equation}\label{norm_def}
   \|X\| := \sqrt{\sum_i X_i^2}
\end{equation}
\end{definition}

\begin{lemma}\label{aux_lemma_1}
Let $X$ be a random vector in $\mathbb{R}^d$ with finite second moment, where $\mathbb{E}[\|X\|^{2}] \leq +\infty$. Then $\|\mathbb{E}[X]\| \leq \mathbb{E}[\|X\|]$, $\|\mathbb{E}[X]\|^{2} \leq \mathbb{E}[\|X\|^{2}]$.
\end{lemma}

\begin{proof}
Due to the convexity of norm operator, we can have $\|\mathbb{E}[X]\| \leq \mathbb{E}[\|X\|]$ using Jensen's inequality. Further we can get
$\|\mathbb{E}[X]\|^{2} \leq (\mathbb{E}[\|X\|^{2}])^2 \leq 
\mathbb{E}[\|X\|^{2}]$ and the statement follows. 
\end{proof}

\begin{lemma}\label{aux_lemma_2}
Let $X$ and $Y$ be two random variables in $\mathbb{R}^d$ with finite second moment. Then $\mathbb{E}[\|X+Y\|] \leq \mathbb{E}[\|X\|] + \mathbb{E}[\|Y\|]$.
\end{lemma}

\begin{proof}
According to Minkowski's inequality that $(\mathbb{E}[\|X+Y\|^{p}])^{1/p} \leq (\mathbb{E}[\|X\|^{p}])^{1/p} + (\mathbb{E}[\|Y\|^{p}])^{1/p}$, set $p=1$ and the statement follows.
\end{proof}

\begin{definition}
Let $X$ be a random vector with values in $\mathbb{R}^{d}$. Then the variance of $X$ is
\begin{equation}\label{var_def}
   \mathbb{V}[X] := \mathbb{E}[\|X-\mathbb{E}[X]\|^{2}]
\end{equation}
\end{definition}

\begin{lemma}[Properties of the variance]\label{aux_lemma_3}
Let $X$ and $Y$ be two independent random variables in $\mathbb{R}^d$. We also assume that $X,Y$, have finite second moment. Then the following hold.

\begin{enumerate}[label={\rm (\roman*)}]
\item\label{lem:varprop_i} 
$\mathbb{V}[X] = \mathbb{E}[\|X\|^2] - \|\mathbb{E}[X]\|^2$,
\item\label{lem:varprop_ii} 
For every $x \in \mathbb{R}^d$, $\mathbb{E}[\|X- x\|^2] = \mathbb{V}[X] + \|\mathbb{E}[X] - x\|^2$. Hence,
$\mathbb{V}[X] = \min_{x \in \mathbb{R}^d} \mathbb{E}[\|X- x\|^2]$,
\item\label{lem:varprop_iii} 
$ \mathbb{V}[X + Y] = \mathbb{V}[X] + \mathbb{V}[Y]$.
\end{enumerate}
\end{lemma}

\begin{proof}
\ref{lem:varprop_i}-\ref{lem:varprop_ii}:
Let $x \in \mathbb{R}^d$. Then, $\|X - x\|^2 = \|X- \mathbb{E}[X]\|^2 + \|\mathbb{E}[X] - x\|^2
+ 2(X- \mathbb{E}[X])^\top(\mathbb{E}[X]- x)$. Hence, taking the expectation
we get $\mathbb{E}[\|X- x\|^2] = \mathbb{V}[X] + \|\mathbb{E}[X] - x\|^2$.
Therefore, $\mathbb{E}[\|X- x\|^2] \geq \mathbb{V}[X]$ and for $x=\mathbb{E}[X]$ we get $\mathbb{E}[\|X- x\|^2] = \mathbb{V}[X]$. Finally, for $x=0$ we get \ref{lem:varprop_i}.

\ref{lem:varprop_iii}:
Let $\bar X := \mathbb{E}[X]$ and $\bar Y := \mathbb{E}[Y]$, we have
\begin{align*}
    \mathbb{V}[X + Y] &= \mathbb{E}[\|X - \bar X + Y - \bar Y\|^2] \\
    &= \mathbb{E}[\|X- \bar X\|^2] + \mathbb{E}[\|Y- \bar Y\|^2] + 2 \mathbb{E}[X - \bar X]^\top\mathbb{E}[Y - \bar Y] \\
    &= \mathbb{E}[\|X- \bar X\|^2] + \mathbb{E}[\|Y- \bar Y\|^2]
\end{align*}
Recalling the definition of $\mathbb{V}[X]$ the statement follows.
\end{proof}

\begin{definition}\textit{(Conditional Variance)}.
Let $X$ be a random variable with values in $\mathbb{R}^d$ 
and $Y$ be a random variable with values in a measurable space $\mathcal{Y}$. We call \emph{conditional variance} of $X$ given $Y$ the quantity
\begin{equation*}
    \mathbb{V}[X \mid Y] := \mathbb{E}[\|X -\mathbb{E}[X \mid Y]\|^2 \mid Y].
\end{equation*}{}
\end{definition}

\begin{lemma}{(Law of total variance)}\label{lm:totvariance}
Let $X$ and $Y$ be two random variables, we can prove that 
\begin{equation}
    \mathbb{V}[X] = \mathbb{E}[\mathbb{V}[X \mid Y]] + \mathbb{V}[\mathbb{E}[X \mid Y]]
\end{equation}
\end{lemma}

\begin{proof}
\begin{align*}
    \mathbb{V}[X] &= \mathbb{E}[\|X - \mathbb{E}[X]\|^2] \\
     \quad 
    &= \mathbb{E}[\|X\|^2] - \|\mathbb{E}[X]\|^2 \\
     \quad
    &= \mathbb{E}[\mathbb{E}[\|X\|^2 \mid Y]] - \|\mathbb{E}[\mathbb{E}[X \mid Y]]\|^2 \\
     \quad
     &= \mathbb{E}[\mathbb{V}[X \mid Y] + \|\mathbb{E}[X \mid Y]\|^2] - \|\mathbb{E}[\mathbb{E}[X \mid Y]]\|^2 \\
    &= \mathbb{E}[\mathbb{V}[X \mid Y]] + \left(\mathbb{E}[\|\mathbb{E}[X \mid Y]\|^2] - \|\mathbb{E}[\mathbb{E}[X \mid Y]]\|^2\right)
\end{align*}
recognizing that the term inside the parenthesis is the conditional variance of $\mathbb{E}[X | Y]$ gives the result.
\end{proof}

\begin{lemma}
\label{aux_lemma_4}
Let $\zeta$ and $\eta$ be two independent random variables with values in $\mathcal{Z}$ and $\mathcal{Y}$ respectively.
Let $\psi \colon \mathcal{Y} \to \mathbb{R}^{m\times n}, \phi\colon \mathcal{Z} \to \mathbb{R}^{n\times p}$, and
$\varphi\colon \mathcal{Y} \to \mathbb{R}^{p\times q}$ matrix-valued measurable functions. Then
\begin{equation}
\mathbb{E}[\psi(\eta) (\phi(\zeta) - \mathbb{E}[\phi(\zeta)] )\varphi(\eta)] = 0
\end{equation}
\end{lemma}
\begin{proof}
Since, for every $y \in \mathcal{Y}$, $B \mapsto \psi(y) B \varphi(y)$ is linear and $\zeta$ and $\eta$ 
are independent, we have
\begin{equation*}
\mathbb{E}[\psi(\eta) (\psi(\zeta) - \mathbb{E}[\psi(\zeta)]) \varphi(\eta) \,\vert \eta] = \psi(\eta) \mathbb{E}\big[\phi(\zeta) - \mathbb{E}[\phi(\zeta)]\big] \varphi(\eta) =0.
\end{equation*}
Taking the expectation the statement follows.
\end{proof}

\section{Experiment}
\label{apx: experiment}

\paragraph{Computational resources.} For compute resources, We used one internal compute servers which consists consisting of $2x$ Tesla A100 cards and $256$ CPUs, however each model is trained on at most 1 card.

\subsection{Tabular MDP}
\subsubsection{Experimental Settings}
\label{apx:tabular_setting}
We adopt the tabular random MDP setting presented in \cite{tang2021unifying}. The dimension is 20 for state space and 5 for action space, so we have the reward matrix $R \in \mathbb{R}^{20\times5}$. The transition probability matrix is generated from independent Dirichlet distributions. The policy is a  matrix $\theta^{0}\in \mathbb{R}^{20\times5}$. The final policy $\pi_{\theta}$ is obtained by adopting Softmax activation on this policy matrix: $\pi_{\theta}(a \mid s)=\exp (\theta(s, a)) / \sum_{b} \exp (\theta(s, b))$. We set the initial policy as the uniform policy (by setting $\theta^{0}$ as zero matrix) in MAML and LIRPG experiment. We conduct the inner-loop update stating from the same point for several times when estimating the meta-gradient correlation and variance. For accuracy measurement between estimation $x\in \mathbb{R}^{L}$ and ground truth $y\in \mathbb{R}^{L}$, we use the following equation:
\begin{equation}
\operatorname{Acc}(x, y):=\frac{x^{T} y}{\sqrt{x^{T} x} \sqrt{y^{T} y}} .
\end{equation}
\subsubsection{Implementation for decomposing Gradient estimation}
\label{apx:implement}
To decompose the gradient estimation effects brought by different sources, such as outer estimation variance and inner estimation bias (compositional bias, hessian estimation error), we utilise the following implementation trick: Using estimator I to estimate $\theta_{c}^{\prime} = \theta + \alpha \nabla J(\theta)$, estimator II to estimate $\theta_{h}^{\prime} = \theta + \alpha \nabla J(\theta)$ and finally combine them with: $\theta^{\prime} = \perp\theta_{c}^{\prime} + \theta_{h}^{\prime} - \perp \theta_{h}^{\prime}$, where $\perp$ is the "stop gradient" operator. By this implementation trick, we can have the following property: $\theta^{\prime} \rightarrow \theta_{c}^{\prime}$ and $\nabla_{\theta}\theta^{\prime}=\nabla_{\theta}\theta_{h}^{\prime}$, where $\rightarrow$ is the "evaluates to" operator. "Evaluates to" operator $\rightarrow$ is in contrast with =, which also brings the equality of gradients. By "Evaluates to" operator, the "stop gradient" operator means that $\perp \left(f_{\theta}(x)\right) \rightarrow f_{\theta}(x)$ but $\nabla_{\theta} \perp\left(f_{\theta}(x)\right) \rightarrow 0$. This property guarantee that the compositional bias is only influenced by estimator I while hessian estimation error is controlled by estimator II. Besides estimator I and estimator II, an extra estimator III is used for outer-loop policy gradient $\nabla_{\theta^{k}}J^{\text{out}}(\pi^{k})$ estimation, which helps us understand the effect of outer-loop policy gradient.

\subsubsection{Additional Experimental Results on Tabular MAML-RL}
\label{apx:tabular_result}
We offer additional experimental results on more estimators (DiCE/ Loaded-DiCE)/settings (All 7 permutations)/metrics (variance of Meta-gradient estimation).

\textbf{Ablation study on sample size and estimator.} Additional experimental results are shown in Fig. \ref{fig_apx:tabular_all}. The comparison between $SSS, SES, ESS$ and $SSE, SEE, ESE$ reveals the importance of the outer-loop gradient estimation. Accurate outer-loop policy gradient estimation brings more significant improvement over the correlation compared with the correction of Hessian error or compositional bias. In addition, with estimated outer-loop policy gradient, the correction of these two terms also helps ($EES>SES>ESS>SSS$). 

Next we discuss the comparison between different estimators. The DiCE estimator have real high variance on first-order and second-order, and its first-order gradient corresponds to the REINFORCE algorithm \cite{williams1992simple} while the rest 3 estimators' first-order gradient corresponds to the Actor-critic algorithm. That is why DiCE performs the worst in all cases. With stochastic outer-loop estimation, the LVC and Loaded-DiCE estimator have comparable correlation while the variance of LVC is smaller than Loaded-DiCE. The AD estimator performs worse than LVC and Loaded-DiCE when the Hessian is estimated (SSE, ESE, SSS, ESS). This corresponds to the conclusion in \cite{rothfuss2018promp} that the LVC estimator introduces low-bias and low-variance Hessian estimation while AD estimator has large-bias and low-varaince Hessian estimation. With exact outer-loop estimation, the LVC has relatively great Hessian estimation so the correction of compositional bias has the same effect with Hessian correction ($ESE=SEE>SSE$), while the Hessian correction is still important in Loaded-DiCE ($SEE>ESE>SSE$).

\textbf{Ablation study on inner learning rate, step and estimator.} Additional ablation study on inner learning rate and number of steps are shown in Fig. \ref{fig_apx:tabular_lr}, \ref{fig_apx:tabular_step}. The results show that: With more steps and larger learning rates, the inner-loop estimation can become more important than outer-loop policy gradient (the correlation decreases a lot in $SSE$ in all estimators). Also in multi-step and large learning rate setting, the importance of Hessian estimation and compositional bias become comparable in LVC and Loaded-DiCE ($SEE\approx ESE, SES\approx ESS$). 

\textbf{Meta-gradient variance}
In all three plots Fig. \ref{fig_apx:tabular_all}, \ref{fig_apx:tabular_lr}, \ref{fig_apx:tabular_step}, we report additional metric on variance of the meta-gradient estimation. We observe that the correction of compositional bias increases the variance especially when outer-loop policy gradient estimator is poor (estimator III uses stochastic samples) or Hessian variance is large (in DiCE and Loaded-DiCE). Only with low Hessian variance (LVC/AD) and great outer-loop policy gradient (estimator III uses analytical solution), the correction of compositional bias can decrease the variance.
\begin{figure*}[tbhp]
    \centering
    \includegraphics[width=0.9\linewidth]{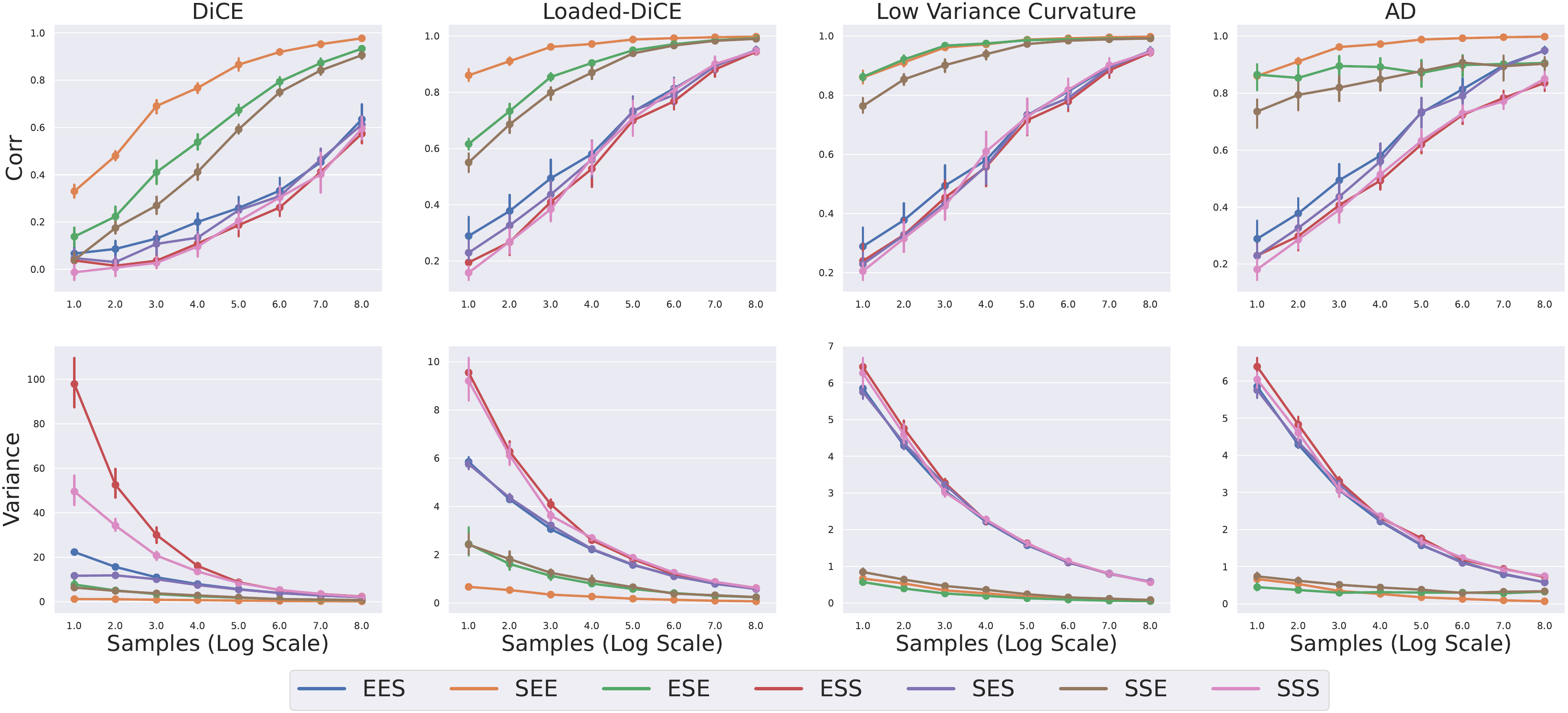}
    \caption{Ablation study on sample size and estimator in 1-step inner-loop setting. (1) Outer-loop policy gradient is important for estimation (2) Compositional bias correction helps increase the correlation (3) The LVC and Loaded-DiCE can achieve higher correlation compared with AD when the Hessian matrix is estimated.}
    \label{fig_apx:tabular_all}
\end{figure*}

\begin{figure*}[tbhp]
    \centering
    \includegraphics[width=0.8\linewidth]{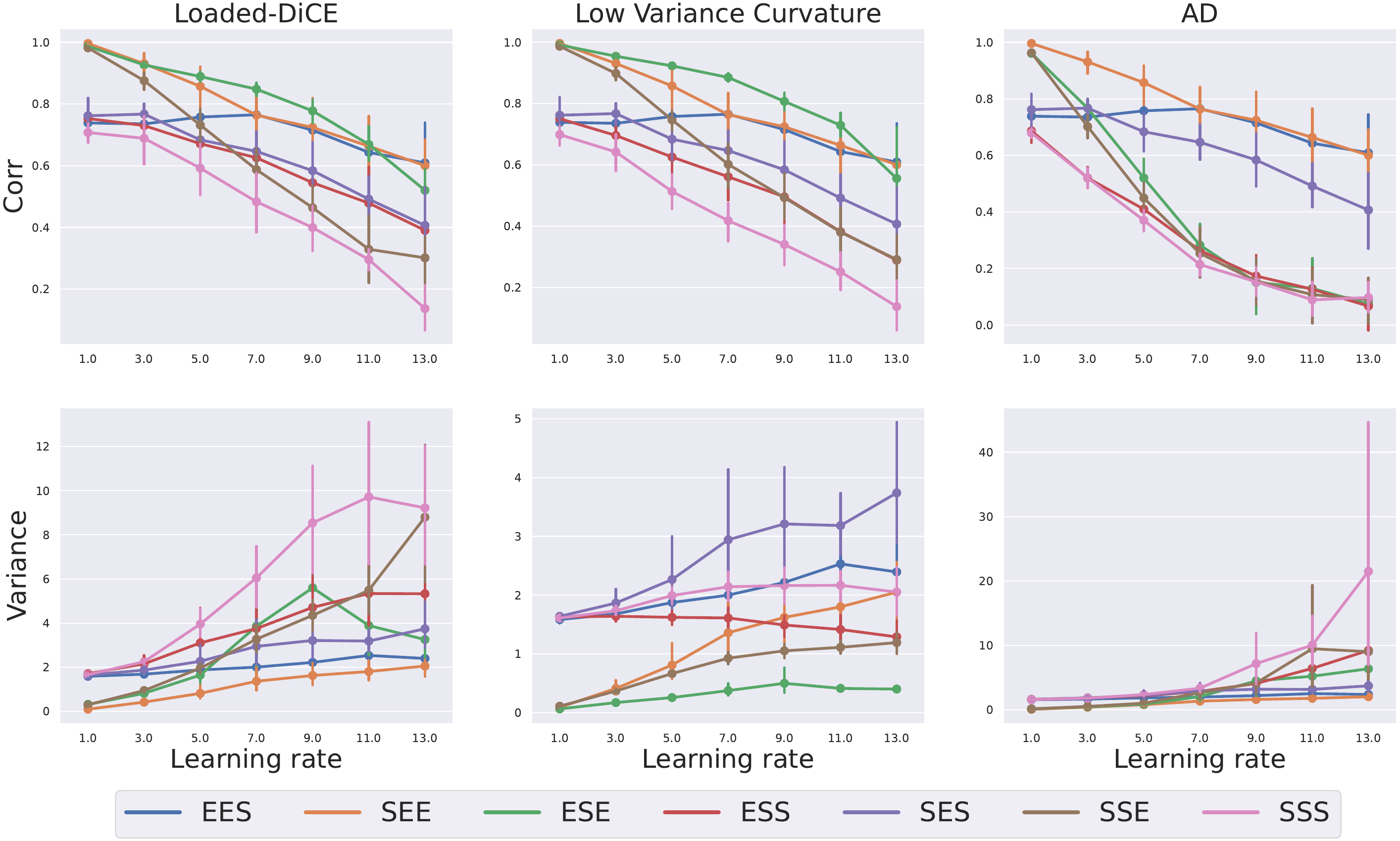}
    \caption{Ablation study on inner learning rate and estimator. (1) In Loaded-DiCE and LVC, With larger learning rate, the compositional bias basically shares the same importance with Hessian estimation error. (2) With larger learning rate, the Hessian estimation problem in AD largely decreases the correlation.}
    \label{fig_apx:tabular_lr}
\end{figure*}

\begin{figure*}[tbhp]
    \centering
    \includegraphics[width=0.8\linewidth]{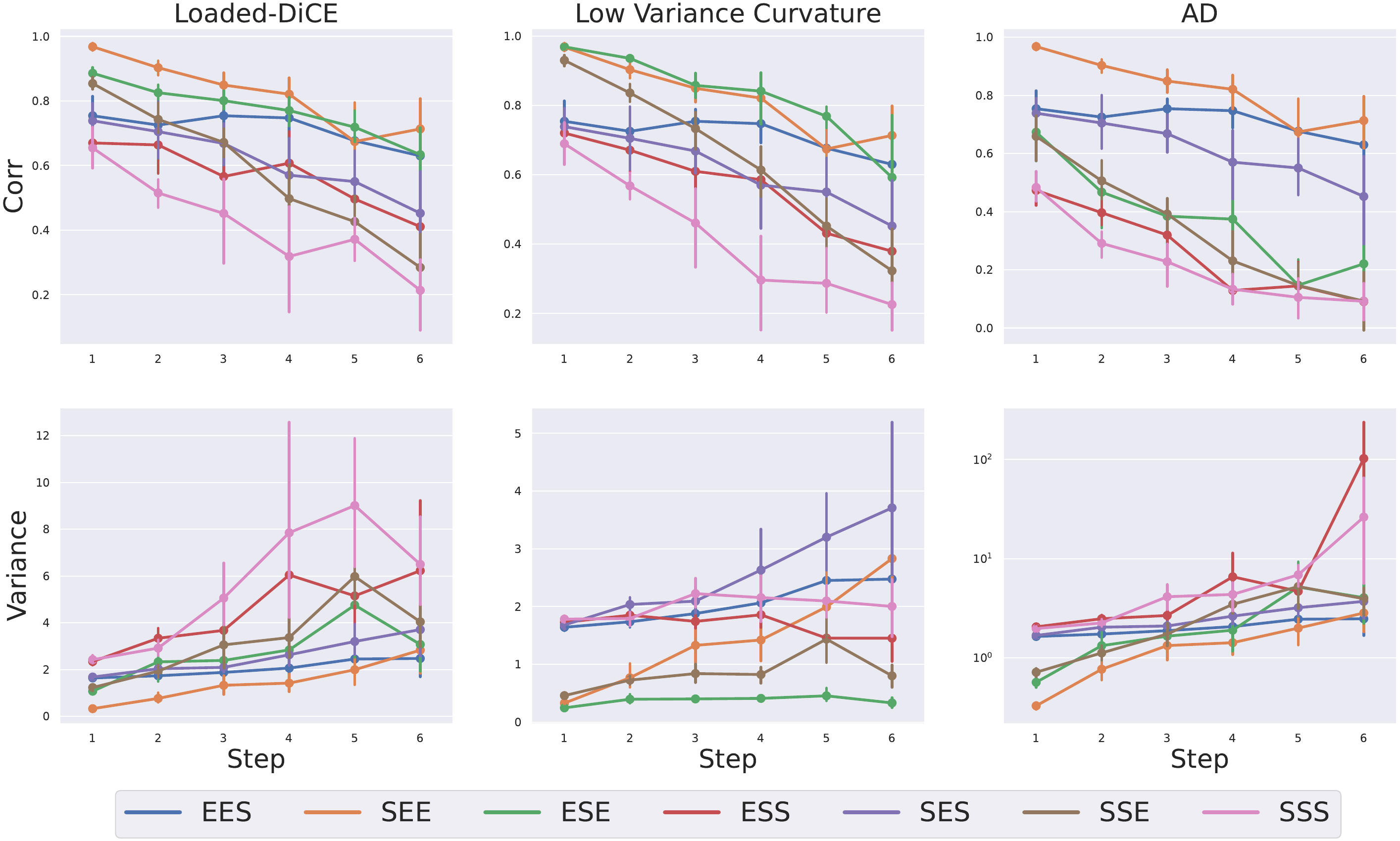}
    \caption{Ablation study on inner step and estimator. Results of larger steps show similar phenomenon with larger inner-loop learning rate.}
    \label{fig_apx:tabular_step}
\end{figure*}

\subsubsection{Additional Experimental Results on Tabular LIRPG}
In Fig. \ref{fig_apx:step_all} we offer additional experimental Results on estimation variance. Basically the AD based estimation in LIRPG setting tend to have higher variance. 
\begin{figure*}[tbhp]
    \centering
    \includegraphics[width=0.6\linewidth]{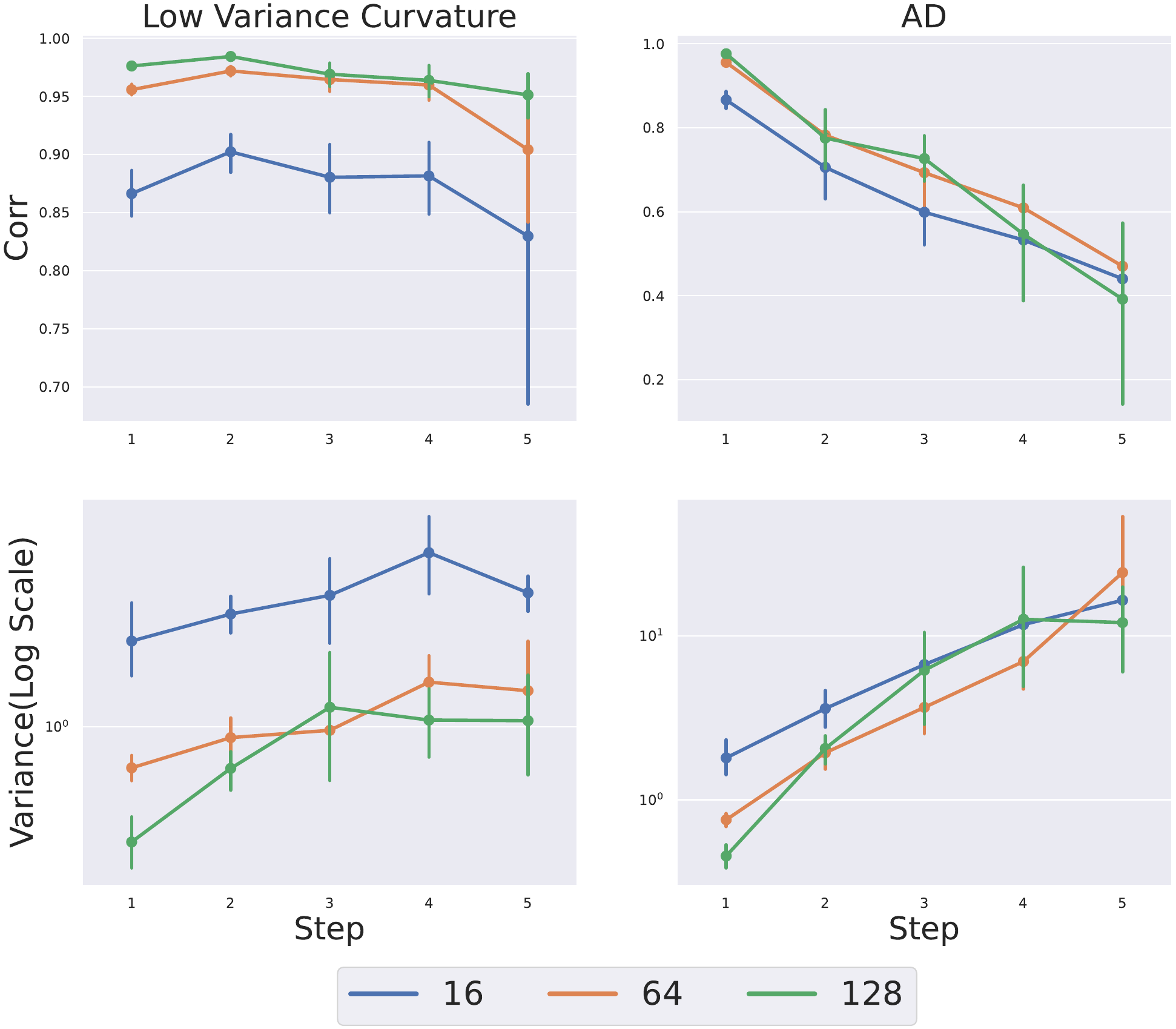}
    \caption{Additional experiment results on. Different color refers to different trajectory sample size.}
    \label{fig_apx:step_all}
\end{figure*}
\subsubsection{Hyperparameters}
We offer the hyperparameter settings for our Tabular MDP experiment in Table \ref{tb:hyp_tmdp}.
\begin{table}[H]
\caption{Hyper-parameter settings for Tabular MDP.}
\label{tb:hyp_tmdp}
\centering
\begin{sc}
\resizebox{1. \textwidth}{!}{
\begin{tabular}{lcl}
\toprule
\textbf{Settings }& \textbf{Value} & \textbf{Description}  \\  \hline \hline
Trajectory length & 20 & RL trajecotry length  \\
Discount factor & 0.8 & Learning rate for meta-solver updates  \\
Inner learning rate & 10 & Learning rate for inner-loop update  \\
Inner step & 1 & step number of inner-update  \\
Independent trials & 10 & Number of independent trials on environments  \\
Same trials & 20 & number of independent trials on the same point \\
Dimension of state & 20 & dimension of state \\
Dimension of action & 5 & Dimension of action\\
Noise coefficient & 1.0 & Noise factor for simulating estimated value function \\
Density & 0.001 & Parameters of Dirichilet Distribution  \\
\hline \hline
\bottomrule
\end{tabular}
}
\end{sc}
\end{table}

\subsection{LOLA-DiCE on Iterated Prisoner Dilemma (IPD)}
\subsubsection{Experimental Settings}
\label{apx_lola_dice}
In Iterated Prisoner Dilemma, the Prisoner Dilemma game is played repeatedly by the same players. The payoffs of Prisoner Dilemma for players are shown as follows. 
$$
\mathbf{R}^1 = \left[\begin{matrix}
-2 & 0 \\
-3 & -1
\end{matrix}\right] 
\quad 
\mathbf{R}^2 = \left[\begin{matrix}
-2 & -3 \\
0 & -1
\end{matrix}\right],
$$
where the action 0 (correpsonds to column/row 0) as "cooperation" (don't confess) and the action 1 (correpsonds to column/row 1) as "defection" (confess). Agent in Iterated Prisoner Dilemma aims at maximising the cumulative Discounted reward. By LOLA-DiCE algorithm, it is possible for both agent to reach social welfare: (-1, -1). Refer to Appendix \ref{apx:opponent_shaping} for how the algorithm is formulated.

We conduct our experiment by adapting code from the official codebase\footnote{\url{https://github.com/alexis-jacq/LOLA_DiCE}}. The official code only conducts the experiment using one fixed seed and the performance is highly sensitive to different random seeds using default hyperparameters. To evaluate the performance reliably, we conduct all the experiments for 10 random seeds and report the average result.
\subsubsection{Additional Experimental Results}
\label{apx_lola_dice_add}
\begin{figure*}[tbhp]
    \centering
    \includegraphics[width=0.9\linewidth]{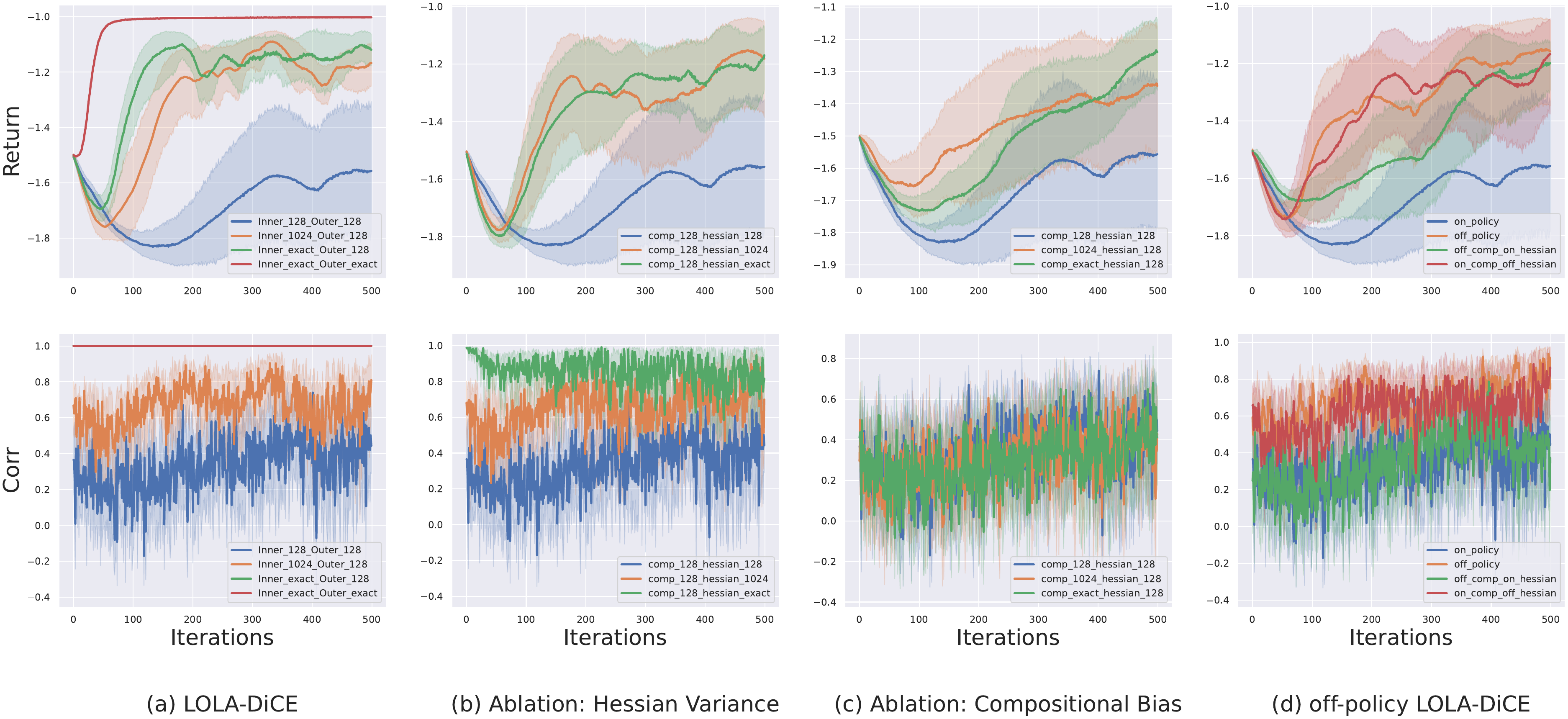}
    \caption{Experiment result of LOLA-DiCE. (a) Poor inner-loop estimation can fail the LOLA-DiCE algorithm. (b) Hessian estimation variance is the main problem in LOLA-DiCE. (c) The correction of compositional bias also helps increase the average return. (d) The off-policy correction can both decrease the compositional bias and Hessian estimation variance, which largely increases the final return.}
    \label{apx_fig:lola}
\end{figure*}
\textbf{Ablation on LOLA-DiCE inner/outer estimation.} We report the correlation result of conducting ablation study for different inner/outer-loop estimation of LOLA-DiCE in the Fig.~\ref{apx_fig:lola}(a).  Higher correlation does not guarantee higher return. The bonus brought by setting inner-loop as exact solution have a really large improvement over correlation (from 0.7 to 1.0) but have limited improvement on return. We believe it is because the outer-loop gradient estimation becomes the main issue when inner-loop estimation is really well. 

\textbf{Ablation on LOLA-DiCE Hessian variance and compositional bias.} We show additional experimental result in Fig. \ref{apx_fig:lola}(c).  An interesting thing is that we find out the gradient correlation of these three settings are comparable. An possible explanation is that the main issue here is the hessian variance and this is why the performance gain by lowering hessian variance is larger that lowering compositional bias. Though by correcting compositional bias LOLA can have better estimation with performance gain, the gain is not obvious in the aspect of gradient correlation because the hessian variance is still large.

\textbf{Off-policy DiCE and ablation study}
The correlation gain for off-policy comp$\&$on-policy hessian is still limited like that in Fig. \ref{apx_fig:lola}(c). But the performance gain verifies the bonus brought by correcting compositional bias.
\subsubsection{Hyperparameters}
We offer the hyperparameter settings for our LOLA experiment in Table \ref{tb:hyp_lola}.
\begin{table}[H]
\caption{Hyper-parameter settings for LOLA-DiCE.}
\label{tb:hyp_lola}
\centering
\begin{sc}
\resizebox{1. \textwidth}{!}{
\begin{tabular}{lcl}
\toprule
\textbf{Settings }& \textbf{Value} & \textbf{Description}  \\  \hline \hline
Outer Learning rate & 0.1 & Outer Learning rate  \\
Inner Learning rate & 0.3 & Inner Learning rate \\
Discount factor & 0.96 & Discount factor  \\
Update & 500 & step number of meta-update  \\
Rollout Length & 100 & Length of IPD rollout  \\
Inner step & 1 & number of virtual inner-step look-ahead \\
Value function learning rate & 0.1 & Value function learning rate \\
Off-policy buffer size & 1024 & Buffer size\\
Sample batch size & 128 & Comp/Hessian/Outer sample batch size \\
\hline \hline
\bottomrule
\end{tabular}
}
\end{sc}
\end{table}

\subsection{MGRL on Atari games}
\subsubsection{Experimental setting}
\label{apx_mgrl_setting}
We reimplement the MGRL algorithm based on A2C baseline. In this case, Meta-parameters $\phi$ involves 4 hyperparameters: Discount factor, value loss coefficient, entropy loss coefficient and GAE ratio. The procedures of 'discard' strategy we use is summarized as follows: Starting from the inner-policy parameters $\theta^{0}$, we utilise take 3 A2C updates and get the 3-step updated policy $\theta^{3}$. Then we can calculate the meta-gradient by backpropogating from $R(\theta^{3})$ to the meta parameters. Finally we reset the inner-loop policy parameters back to $\theta^{1}$ so the rest 2 updates are in fact virtual update. It is only used for the meta-gradient estimation.
\subsubsection{Discussion on the 'Discard' strategy}
\label{apx: mgrl_dis}
In the MGRL experiment, we follow previous work \cite{bonnet2021one} for conducting multi-step MGRL. So the inner-loop policy will take multi-step virtual updates for meta-parameters update. As mentioned in Section 4.2 in their paper, this strategy can only keep the RL update times unchanged among different algorithms and is not particularly sample efficient because they need to take virtual look-ahead for the update of meta-parameters. However, one benefit of adopting such strategy is that we can keep the amount of meta-update large enough to verify the effect brought by the LVC correction. We also take some experiments on another setting where we take meta-update after each 3-step inner-loop update. Note that they are no longer virtual inner-loop updates. However, we find out that in many environment this setting largely decrease the meta-update times and make the comparison of different meta-gradient estimation less meaningful.
\subsubsection{Additional Experimental results}
\label{apx: mgrl_all}
We offer the full experiments results on all 8 Atari games: Asteroids, Qbert, Tennis, BeamRider, Alien, Assault, DoubleDunk, Seaquest. The reward performance is shown in Fig. \ref{fig:mgrl_all}. We also offer trajectories for all 4 meta-parameters on these experiments in Fig. \ref{fig:mgrl_hyper8}. From Fig. \ref{fig:mgrl_all} it can show that MGRL with LVC correction can achieve comparable or better performance in almost all 8 environments. Note that we need to clarify that in some RL experiments the MGRL cannot achieve better performance compared with A2C baseline. This also corresponds to the experimental results in original MGRL paper \cite{xu2018meta}. However, since our main comparison only happens between MGRL and MGRL with LVC correction, it is still a fair comparison to verify the effectiveness of LVC hessian correction. Fig. \ref{fig:mgrl_hyper8} reveals that even we have only 4 meta-parameters, different meta-gradient estimation can still results in large gap between the meta optimisation trajectory and final GMRL performance.
\begin{figure*}[t]
    \centering
    \includegraphics[width=0.9\linewidth]{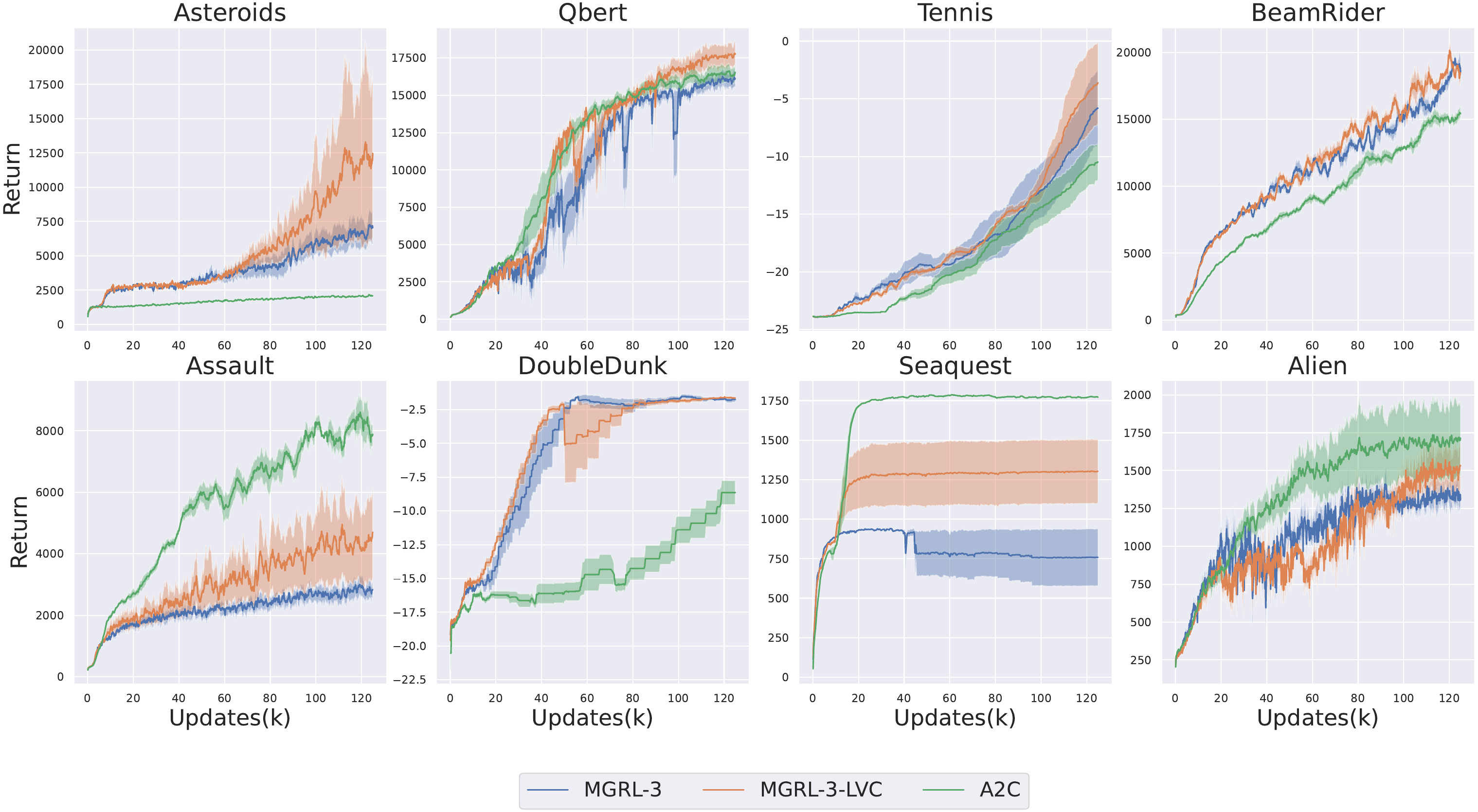}
    \caption{Experimental results on Atari game over 5 seeds. 3-step MGRL with LVC correction can  achieve at least the same performance compared with 3-step MGRL in basically all environments.}
    \label{fig:mgrl_all}
\end{figure*}
\begin{figure*}[t]
    \centering
    \includegraphics[width=0.9\linewidth]{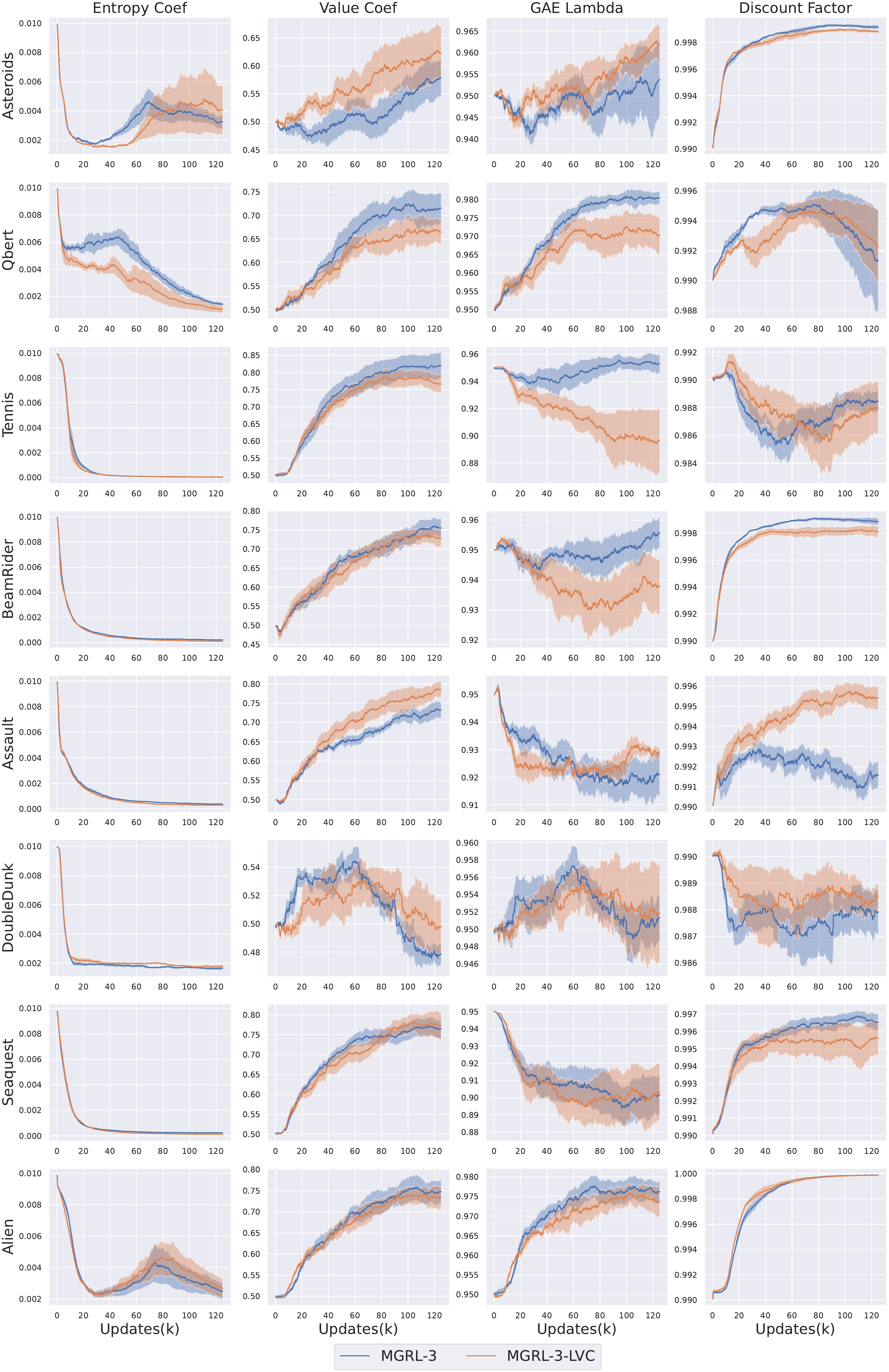}
    \caption{4 Meta parameters trajectories on Atari game for 3-step MGRL and 3-step MGRL with LVC correction.}
    \label{fig:mgrl_hyper8}
\end{figure*}

\subsubsection{Implementations and hyperparameters}

We adopt the codebase of A2C from \cite{pytorchrl} and differentiable optimization library \cite{TorchOpt} to implement MGRL algorithms. We use a shared CNN network (3 Conv layers and one fully connected (FC) layer) for the policy network and critic network. The (out-channel, filters, stride) for each Conv layer is (32, $8\times8$, 4), (64, $4\times4$, 2) and (64, $3\times3$, 1) respectively while the hidden size is 512 for the FC layer. For the training loss, we adopt additional entropy regularisation for policy loss and Mean Square Error (MSE) for the value loss. We adopt the Generalized Advantage estimation (GAE) for advantage estimation. We offer the hyperparameter settings for our experiment in Table \ref{tb:hyp_mgrl}. We tun our algorithm for 125k inner updates, which corresponds to 40M environment steps for baseline A2C.
\begin{table}[H]
\caption{Hyper-parameter settings for MGRL.}
\label{tb:hyp_mgrl}
\centering
\begin{sc}
\resizebox{1. \textwidth}{!}{
\begin{tabular}{lcl}
\toprule
\textbf{Settings }& \textbf{Value} & \textbf{Description}  \\  \hline \hline
Inner Learning rate & 7e-4 & Inner Learning rate \\
Learning rate Scheduling & Linear decay& linearly decrease to 0  \\
Discount factor & 0.99 & Discount factor  \\
GAE LAMBDA & 0.95 & ratio of generalized advantage estimation  \\
Value coef & 0.5 & coefficient of value loss \\
entropy coef & 0.01 & coefficient of entropy loss \\
Update & 125k & number of inner update \\
Number of process & 64 & number of multi process\\
number of step per update & 5 & number of step per update\\
Meta update & 3 & number of inner-update for conducting meta-update\\
Meta learning rate & 0.001 & meta learning rate  \\
Inner optimizer & Adam & Inner-loop optimizer \\
Outer optimizer & Adam & Outer-loop optimizer \\
\hline \hline
\bottomrule
\end{tabular}
}
\end{sc}
\end{table}

\section{Author contribution}
\label{apx: author_contrib}
We summarise the main contributions from each of the authors as follows:

\textbf{Bo Liu}: Algorithm design, main theoretical proof, some code implementation and experiments running (on tabular MDP and MGRL), and paper writing.

\textbf{Xidong Feng}: Idea proposing, algorithm design, part of theoretical proof, main code implementation and experiments running (on tabular MDP, LOLA and MGRL), and paper writing.

\textbf{Jie Ren}: Code implementation and experiments running for MGRL.

\textbf{Luo Mai}: Project discussion and paper writing.

\textbf{Rui Zhu}: Project discussion and paper writing.  

\textbf{Haifeng Zhang}: Computational resource sponsor.

\textbf{Jun Wang}: Project discussion and overall project supervision.

\textbf{Yaodong Yang}: Project lead, idea proposing, experiment supervision, and whole manuscript writing.